\documentclass{article}

\PassOptionsToPackage{numbers, compress}{natbib}
\usepackage[preprint]{neurips_2026}

\usepackage[T1]{fontenc}
\usepackage[utf8]{inputenc}
\usepackage{microtype}    

\usepackage{graphicx}
\usepackage{subcaption}

\usepackage{tabularray}
\UseTblrLibrary{booktabs}

\usepackage{pgffor}
\usepackage{pgfplots}
\usepackage{schemabloc,tikz-cd}
\usepackage{tikz}
\pgfplotsset{compat=1.18}
\usepgfplotslibrary{fillbetween}
\usepgfplotslibrary{groupplots}

\usepackage{amsfonts}
\usepackage{amsmath}
\usepackage{amssymb}
\usepackage{amsthm}
\usepackage{bm}
\usepackage{mathtools}
\usepackage{nicefrac}
\DeclareMathOperator*{\argmin}{arg\,min}
\newcommand{\expnumber}[2]{{#1}\mathrm{e}{#2}}

\usepackage[algoruled,noline,linesnumbered]{algorithm2e}
\SetKw{Continue}{continue}
\SetKw{Or}{or}

\usepackage[acronym]{glossaries}
\newacronym[longplural=Linear Matrix Inequalities]{lmi}{LMI}{Linear Matrix Inequality}
\newacronym{bfgs}{BFGS}{Broyden-Fletcher-Goldfarb-Shanno}
\newacronym{ekf}{EKF}{Extended Kalman Filter}
\newacronym{gbn}{GBN}{Generalized Binary Noise}
\newacronym{hippo}{HiPPO}{High-order Polynomial Projection Operators}
\newacronym{hpc}{HPC}{High-Performance Computer}
\newacronym{iqr}{IQR}{Inter-Quantile Range}
\newacronym{jit}{JIT}{Just-In-Time}
\newacronym{llm}{LLM}{Large Language Model}
\newacronym{lti}{LTI}{Linear Time-Invariant}
\newacronym{mamba}{\textit{Mamba}}{Simplified Structured Selective State-Space Sequence}
\newacronym{ml}{ML}{Machine Learning}
\newacronym{mlp}{MLP}{Multi-Layer Perceptron}
\newacronym{mpc}{MPC}{Model Predictive Control}
\newacronym{mse}{MSE}{Mean Squared Error}
\newacronym{msvr}{MSVR}{Mean Squared Violation Radius}
\newacronym{nlti}{NLTI}{Non-Linear Time-Invariant}
\newacronym{nmse}{NMSE}{Normalized Mean Squared Error}
\newacronym{nn}{NN}{Neural Network}
\newacronym{nsfe}{NSFE}{Normalized Squared Frobenius Error}
\newacronym{nssr}{NSSR}{Normalized Squared Spectral Radius}
\newacronym{ode}{ODE}{Ordinary Differential Equation}
\newacronym{rtss}{RTSs}{Rauch-Tung-Striebel smoother}
\newacronym{s4}{S4}{Structured State-Space Sequence}
\newacronym{silu}{SiLU}{Sigmoid Linear Unit}
\newacronym{sim}{SIM}{Subspace Identification Method}
\newacronym{simba}{SIMBa}{System Identification Methods leveraging Backpropagation}
\newacronym{slp}{SLP}{Single-Layer Perceptron}
\newacronym{sota}{SotA}{State of the Art}
\newacronym{ss}{SS}{State-Space}
\newacronym{svd}{SVD}{Singular-Value Decomposition}
\newacronym{sysid}{SysID}{System Identification}

\usepackage{url}
\usepackage{hyperref}
\usepackage{cleveref}

\theoremstyle{plain}
\newtheorem{theorem}{Theorem}[section]

\newtheorem{lemma}[theorem]{Lemma}

\theoremstyle{definition}
\newtheorem{definition}[theorem]{Definition}

\theoremstyle{remark}

\title{A Novel Schur Decomposition-Based Weight Projection Method for Stable State-Space Neural-Network Architectures}

\author{%
  Sergio Vanegas \\
  Computational Engineering\\
  LUT University\\
  Lappeenranta, Finland \\
  \texttt{sergio.vanegas.arias@lut.fi} \\
  \And
  Lasse Lensu \\
  Computational Engineering\\
  LUT University\\
  Lappeenranta, Finland \\
  \texttt{lasse.lensu@lut.fi} \\
  \And
  Fredy Ruiz \\
  DEIB\\
  Politecnico di Milano\\
  Milan, Italy \\
  \texttt{fredy.ruiz@polimi.it} \\
}

\begin{document}

\maketitle

\begin{abstract}
    Building black-box models for dynamical systems from data is a challenging problem in machine learning, especially when asymptotic stability guarantees are required. In this paper, we introduce a novel stability-ensuring and backpropagation-compatible projection scheme based on the Schur decomposition for the state matrix of linear discrete-time state-space layers, as well as an alternative pre-factorized formulation of the methodology. The proposed methods dynamically project the quasi-triangular factor of the state matrix's real Schur decomposition onto its nearest stable peer, ensuring stable dynamics with minimal overparameterization. Experiments on synthetic linear systems demonstrate that, despite a marginal increase in computational complexity, the method achieves accuracy and convergence rates comparable to those of state-of-the-art stable-system identification techniques. Furthermore, the lower weight count facilitates convergence during training without sacrificing accuracy in stacked neural-network architectures with static nonlinearities targeting real-world datasets. These results suggest that the Schur-based projection provides a numerically robust framework for identifying complex dynamics on par with the state of the art while satisfying strict asymptotic-stability requirements. All code is available at \url{https://codeberg.org/sergiovaneg/SchurSS}.
\end{abstract}

\section{Introduction} \label{sec:introduction}

Recent developments in \gls{ss} \gls{sysid} theory have sparked renewed interest within the \gls{ml} community as an alternative sequence-modelling paradigm to Transformers~\cite{vaswani2017attention}. The \gls{ss} framework models the sampled output $\bm{y}[k] \in \mathbb{R}^{n_y}, k \in \mathbb{N}_0$ of a real \gls{lti} system with exogenous input $\bm{u}[k] \in \mathbb{R}^{n_u}$ as a difference-equation system of the form
\begin{align} \label{eq:ss_discrete}
    \begin{cases}
        \bm{x}[k+1] & = \bm{A} \bm{x}[k] + \bm{B} \bm{u}[k] ,
        \\
        \bm{y}[k]   & = \bm{C} \bm{x}[k] + \bm{D} \bm{u}[k] ,
    \end{cases}
    \quad s.t. \quad
    \begin{split}
        \bm{A} \in \mathbb{R}^{n_x \times n_x} \ , \ \bm{B} \in \mathbb{R}^{n_x \times n_u},
        \\
        \bm{C} \in \mathbb{R}^{n_y \times n_x} \ , \ \bm{D} \in \mathbb{R}^{n_y \times n_u},
    \end{split}
\end{align}
which can be used to represent the sampled dynamics of a continuous-time system by means of a discretization scheme~\cite{al2002model}. Even if incapable of modelling nonlinear dynamics on their own, the versatility and expressiveness of this approach have made \gls{ss} subsystems the building block of choice for some of the most noteworthy sequence-modelling \glspl{nn} in recent years~\cite{gu2022efficiently,smith2023simplified,gu2024mamba}.

The revitalization of \gls{ss} theory for \glspl{nn} was catalyzed by the introduction of the \gls{hippo} framework~\cite{gu2020hippo}, which optimizes the state vector $\bm{x}$ to maximize the representation of the exogenous-input history. This is achieved through a prescribed upper-triangular matrix $\bm{A}$ that maximizes the information compressed in the states with respect to a given measure function, allowing \gls{ss} models to capture long- and short-term dependencies with efficiency comparable to Transformer-based architectures. Beyond its memory capabilities, the \gls{hippo} matrix ensures inherent stability: it is Hurwitz-stable by construction in continuous time, and its discretized counterpart remains Schur-stable regardless of the step size. To overcome the limitations of fixed state dynamics, subsequent works such as \gls{s4} and \gls{mamba} have utilized this matrix (and its diagonalized variant) primarily as a robust initialization strategy, rather than a static parameter~\cite{gu2022efficiently,gu2024mamba}.

While high-capacity architectures are designed for the billion-parameter scale characteristic of \glspl{llm}, their computational overhead is often prohibitive for time-sensitive applications such as real-time control. To address this gap, backpropagation-ready matrix parameterizations~\cite{di2024simba} and weight regularization techniques~\cite{bemporad2025bfgs} targeting smaller-scale nonlinear systems have emerged. Despite these advances, methodologies for low-dimensional models remain less mature than their large-scale counterparts, requiring significant over-parameterization~\cite{di2024simba} or careful hyperparameter tuning~\cite{bemporad2025bfgs} to ensure stable convergence and performance.

The present work adds to this line of research by:
\begin{enumerate}
    \item Introducing an innovative Schur-stable state-matrix weight-projection scheme, exploiting recent results on $\Omega$-stable nearest-matrix identification~\cite{noferini2021nearest};
    \item Proposing an alternative stable state-matrix parameterization using the above weight-projection algorithm, which trades parameter count in favour of higher computational efficiency during model training;
    \item Carrying out a comprehensive benchmark of the proposed methodologies against the \gls{sota} in Schur-stable state-matrix identification, targeting both synthetic and real-world datasets.
\end{enumerate}

The paper is divided as follows. \Cref{sec:related} contains a brief historical overview of the literature on nonlinear \gls{ss} \gls{sysid}. \Cref{sec:schur_ss} introduces the proposed Schur-stable projection method, as well as the $\Omega$-stable projection theory behind it. \Cref{sec:exp} presents a series of experiments to determine the performance of the developed method compared to the \gls{sota}, comprising both \gls{lti} \gls{sysid} synthetic-data benchmarks and a performance comparison between alternative \gls{ss}-block formulations within a nonlinear model on real-world datasets. Finally, \Cref{sec:conclusions} lays out some final remarks based on the experimental results and potential extensions to the showcased methods.

All code associated to this manuscript is made available at \url{https://codeberg.org/sergiovaneg/SchurSS} and released under the BSD-3-Clause license.

\section{Related Work} \label{sec:related}

While \gls{s4}~\cite{gu2022efficiently} and \gls{mamba}~\cite{gu2024mamba} are built using the continuous-time \gls{ss} formulation, and even though methods for stable identification in continuous time can be found in the literature, the discrete formulation from \Cref{eq:ss_discrete} is more commonly studied as it directly maps to data-driven methodologies. Thus, unless otherwise stated, any mention of \textit{stability} will refer to Schur stability; i.e., $$|\lambda_i| \le 1 \quad \forall \lambda_i \in \Lambda_{\bm{A}},$$ where $\Lambda_{\bm{A}} \in \mathbb{C}^{n_x}$ denotes the spectrum of $\bm{A}$.

\subsection{Classical Identification Methods}

Although \gls{ss} \gls{sysid} theory was a well-established field already in the 1970s~\cite{aastrom1971system}, modern deterministic methods can, for the most part, be traced back to the unifying theory of \glspl{sim} proposed by \citet{van1995unifying}, under which identified models are guaranteed to be stable provided that the source system is stable. These techniques operate by organizing available input and output data into a mathematical structure known as a Hankel matrix~\cite{mu2014hankel}, which is shaped by the estimated system order and the available number of samples. As this matrix is constructed directly from recorded measurements, specific \gls{sim} approaches differ from each other based on how they recover the individual matrices~\cite{qin2006overview}.

\begin{figure}[tb!]
    \centering
    \begin{subfigure}{0.49\textwidth}
        \centering
        \resizebox{\linewidth}{!}{
            \begin{tikzpicture}
                \sbEntree{U}

                \sbStyleBloc{red}
                \sbBloc[2.5]{G}{$G(z)$}{U}
                \sbRelier[$u_t$]{U}{G}

                \sbStyleBloc{blue}
                \sbBloc[2.5]{f}{$f(x)$}{G}
                \sbRelier[$x_t$]{G}{f}

                \sbStyleBloc{red}
                \sbBloc[2.5]{S}{$S(z)$}{f}
                \sbRelier[$r_t$]{f}{S}

                \sbSortie[2.5]{Y}{S}
                \sbRelier[$y_t$]{S}{Y}
            \end{tikzpicture}
        }
        \caption{Wiener-Hammerstein}
        \label{subfig:wh}
    \end{subfigure}
    \hfill
    \begin{subfigure}{0.49\textwidth}
        \centering
        \resizebox{\linewidth}{!}{
            \begin{tikzpicture}
                \sbEntree{U}

                \sbStyleBloc{blue}
                \sbBloc[2.5]{f}{$f(u)$}{U}
                \sbRelier[$u_t$]{U}{f}

                \sbStyleBloc{red}
                \sbBloc[2.5]{G}{$G(z)$}{f}
                \sbRelier[$x_t$]{f}{G}

                \sbStyleBloc{blue}
                \sbBloc[2.5]{g}{$g(r)$}{G}
                \sbRelier[$r_t$]{G}{g}

                \sbSortie[2.5]{Y}{g}
                \sbRelier[$y_t$]{g}{Y}
            \end{tikzpicture}
        }
        \caption{Hammerstein-Wiener}
        \label{subfig:hw}
    \end{subfigure}
    \caption{Single-branch structures considered in \citet{schoukens2017identification}, obtained by chaining \gls{lti} blocks $G(z)$ and $S(z)$, and static nonlinear blocks $f(\cdot)$ and $g(\cdot)$.}
    \label{fig:single_branch}
\end{figure}
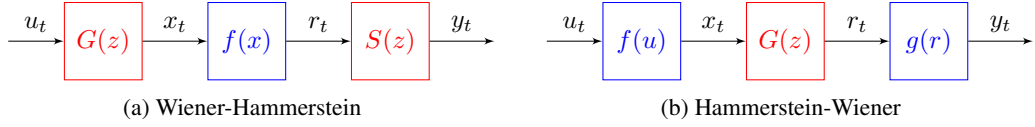

\gls{ss} blocks of the form in \Cref{eq:ss_discrete} are incapable of modelling nonlinear dynamics on their own, so they usually operate as subunits of larger model structures. This concept appears in nonlinear stacks like those of \Cref{fig:single_branch}~\cite{schoukens2017identification} and in dimensional-elevation techniques like the Koopman operator~\cite{brunton2022modern}, which have been extensively studied in the control literature, relying on \glspl{sim} to calibrate their parameters. These frameworks are favoured in the literature not only for the simplicity of their formulation, but also because they decouple the nonlinear effects from the \gls{ss} feedback loop, preserving its asymptotic-stability guarantees. However, the invertibility constraints imposed on the nonlinear blocks by conventional model-fitting algorithms (see \citet{schoukens2017identification}) limit the structure and stack dimensionality of these models.

In practice, this limitation is overcome by applying multi-level optimization methods~\cite{wills2013identification}. Under this optic, it is clear how backpropagation fits the use case: not only does it benefit from the same assumptions made by previous data-driven methods (large data availability, model fitness, etc.), but it is also indifferent to the model's invertibility. In exchange, the stability guarantees given by \glspl{sim} are lost~\cite{maciejowski1995guaranteed}, necessitating a new strategy to ensure asymptotic stability.

\subsubsection{Stability by Regularization}

Still within the realm of classical model identification theory, but lacking a closed-form solution, it is possible to find regularized-loss identification schemes in the literature~\cite{pillonetto2022regularized}. These schemes are least-squares methods penalizing the violation of asymptotic stability through an additive regularization term as
\begin{equation} \label{eq:regularized_obj}
    \min_{\bm{\Theta} \triangleq
        \begin{bsmallmatrix}
            \bm{A} & \bm{B}
            \\
            \bm{C} & \bm{D}
        \end{bsmallmatrix}
    }
    \ell(\bm{\Theta} ; \bm{y}, \bm{x}[0]) + \rho_r r(\bm{\Theta})
    \quad s.t. \quad
    \ell (\bm{\Theta}; \bm{y}, \bm{x}[0]) = \frac{1}{N} \sum_{k=0}^{N-1} \mathcal{L} \left( \bm{y}[k], \bm{\hat{y}}[k; \bm{\Theta}, \bm{x}[0]] \right),
\end{equation}
where $\mathcal{L} : \mathbb{R}^{n_y} \times \mathbb{R}^{n_y} \rightarrow \mathbb{R}$ is a suitable loss function, $\bm{y} \in \mathbb{R}^{n_y}$ denotes the measured output of the target system totaling $N \in \mathbb{N}$ training samples, $\bm{\hat{y}} \in \mathbb{R}^{n_y}$ is the output generated by the dynamics in \Cref{eq:ss_discrete} given an initial condition $\bm{x}[0] \in \mathbb{R}^{n_x}$, $r : \mathbb{R}^{(n_x + n_y) \times (n_x + n_u)} \rightarrow \mathbb{R}$ is the aforementioned regularization term, and $\rho_r \ge 0$ denotes the \textit{strength} of the regularization term.

The regularization term can assume a wide range of shapes. \citet{sjoberg1993use} proposed a simple regularization term via the vector 1-norm, which equally penalizes large coefficients across all matrices, whereas \citet{van2002identification} used a positive semi-definite matrix as a proxy for a weighted sum of the state-matrix spectrum, which penalizes large-magnitude eigenvalues. These functions, however, fail to discriminate between stable and unstable matrices, and they could misidentify marginally stable systems, which is why \citet{bemporad2025bfgs} very recently proposed a scheme capable of making such a distinction by means of the piecewise regularization term
\begin{equation} \label{eq:regularization_bemporad}
    r(\bm{\Theta}) = \max \left\{ \left\| \bm{A} \right\|_2^2 - 1 + \epsilon , 0 \right\}^2,
\end{equation}
where $\|\cdot\|_2$ denotes the spectral norm and $0 \le \epsilon \ll 1 $ is added for numerical stability. Not only does this method impose a zero penalization for asymptotically stable systems, but it is also proven not to limit the coverage of the Schur-stable space~\cite[Lemma 1]{bemporad2025bfgs}. Nevertheless, like its peers, this stabilization method is highly dependent on the careful tuning of the parameter $\rho_r$, as values that are either too small or too large can result in unstable or null state matrices, respectively.

\subsection{Modern Structured State-Space Models}

The current \gls{sota} in \gls{ss} modelling is the \gls{mamba} model~\cite{gu2024mamba}. \Gls{mamba} models nonlinear irregularly-sampled signals, such as natural language or DNA sequencing, by combining \gls{hippo} initialization, channel-wise variable-step discretization, \gls{mlp} aggregation, and diagonal states. This last property not only facilitates parallelized model execution, but it also simplifies the stabilization of its dynamics.

These design choices are fitting for discrete-output systems where the concept of ``regular sampling'' does not apply. However, these properties also make the model unnecessarily difficult to train and run for regularly-sampled, continuous-valued systems~\cite{salam2025comprehensive}. Works like \citet{patro_simba_2024} have tried to lighten this burden by using a single \gls{mamba} layer followed by a series of frequency-domain Hadamard products (time convolutions) and cross-channel matrix multiplications, at the expense of losing time-dependence locality and requiring a Fourier transform. Furthermore, diagonal real-valued state coefficients are not capable of modelling autonomous oscillations, and complex-valued state matrices require (in general) additional projections to yield a real-valued output.

While constraining the number of complex-conjugate eigenvalues could theoretically enable diagonal models to capture oscillatory dynamics, this approach requires a-priori knowledge of the state matrix's sparsity pattern. Furthermore, the reliance on complex-number arithmetic introduces significant implementation overhead and limits flexibility. Thus, a more intelligent approach to stabilizing dense state matrices is required to reliably model a broader range of dynamical systems without requiring predefined structural tuning.

\subsection{SIMBa}

The closest competitor in the literature to this paper's approach is \gls{simba}~\cite{di2024simba}. The authors leverage \glspl{lmi} to design constraint-free parameterizations guaranteeing system stability with arbitrary eigenvalue-magnitude upper bounds. Moreover, they prove that their method is capable of capturing \textit{all} Schur-stable matrices, limited only by floating-point precision.

The only concern raised by the authors of \gls{simba} is regarding its training time compared to traditional identification methods. Since the projection-based methods proposed in this paper also rely on backpropagation, this metric is expected to be in the same order of magnitude. However, it is worth pointing out that \gls{simba} requires tuning $5n_x^2$ weights to parameterize an $n_x^2$-sized matrix, which marginally increases the model's memory requirements and, more importantly, induces a gradient fan-out effect that could potentially impact the model-training's convergence rate.

\section{State-Matrix Stabilization via Schur Factorization} \label{sec:schur_ss}

\subsection{Nearest \texorpdfstring{$\Omega$}{Omega}-Stable Matrix Identification} \label{subsec:schur_ss_full}

Not unique to the realm of \gls{sysid}, but with clear applicability, is the problem of the nearest $\Omega$-stable matrix identification, formally posed in \Cref{problem:nearest_stable}. In the case of real discrete-time systems (i.e., as in \Cref{eq:ss_discrete}), $\Omega = \Omega_S \triangleq \{\lambda \in \mathbb{C} : |\lambda| \le 1 \}$.

\begin{definition}[\citet{noferini2021nearest}] \label[definition]{problem:nearest_stable}
    Let $\Omega$ be a non-empty closed subset of $\mathbb{C}$, and let
    \begin{equation}
        S(\Omega, n, \mathbb{F}) \triangleq \left\{ \bm{X} \in \mathbb{F}^{n \times n} : \Lambda_{\bm{X}} \subseteq \Omega \right\} \subseteq \mathbb{F}^{n \times n}
    \end{equation}
    be the set of $n \times n$ matrices (with either $\mathbb{F} = \mathbb{C}$ or $\mathbb{F} = \mathbb{R}$) whose eigenvalues all belong to $\Omega$.

    Given $\bm{A} \in \mathbb{F}^{n \times n}$, the nearest $\Omega$-stable matrix $\bm{\hat{A}}$ is defined as
    \begin{equation} \label{eq:nearest_obj_function}
        \bm{\hat{A}} \triangleq \argmin_{\bm{X} \in S(\Omega, n, \mathbb{F})} \left\| \bm{A} - \bm{X} \right\|_F^2,
    \end{equation}
    where $\| \bm{X} \|_F^2 \triangleq\sum_{i,j=1}^n |X_{ij}|^2 = \mathrm{tr}\left( \bm{X}^* \bm{X} \right)$ denotes the \textit{Frobenius} norm.
\end{definition}

This problem shows up repeatedly in the systems-theory literature~\cite{datta2004numerical}, but the non-convexity of the Schur-stable matrix space makes it notoriously hard to guarantee a global minimum. It was not until recent years that \citet{noferini2021nearest} provided a solution with such a guarantee by proposing a bi-level optimization over the real Schur decomposition of $\bm{\hat{A}}$, defined as in \Cref{theorem:real_schur}. At every iteration, the $2 \times 2$ block-diagonal elements of $\bm{T}$ are projected to the nearest element\footnote{For $\bm{T_{i,i}} \in \mathbb{R}^{1 \times 1}$, $\bm{\hat{T}_{i,i}} = \nicefrac{\bm{T_{i,i}}}{\max\{1,|\bm{T_{i,i}}|\}}$.} in the set $\mathcal{\hat{S}}(\bm{T_{i,i}}) \subset \mathbb{R}^{2 \times 2}$ of 15 elements at most (see \Cref{app:nearest_schur_stable}), followed by the optimization of the orthogonal factor $\bm{Z}$ over the Riemannian manifold $SO(2) \triangleq \{ \bm{X} \in O(2) : \mathrm{det}(\bm{X}) = 1 \}$.

\begin{theorem}[\citet{horn2012matrix}] \label{theorem:real_schur}
    Given $\bm{X} \in \mathbb{R}^{n \times n}$, there exists a pair of matrices $\bm{Z} \in O(n) \triangleq \{\bm{W} \in \mathbb{R}^{n \times n} : \bm{W}^\intercal \bm{W} = \bm{W} \bm{W}^\intercal = \bm{I}_n \}$ real orthogonal and $\bm{T} \in \mathbb{R}^{n \times n}$ upper quasi-triangular defining the Schur decomposition; i.e., a pair such that
    \begin{equation} \label{eq:schur_decomposition}
        \bm{Z}^\intercal \bm{X} \bm{Z} = \bm{T} = \begin{bmatrix}
            \bm{T_{1,1}} & \bm{T_{1,2}} & \cdots & \bm{T_{1,k-1}}   & \bm{T_{1,k}}
            \\
            \bm{0}       & \bm{T_{2,2}} & \cdots & \bm{T_{2,k-1}}   & \bm{T_{2,k}}
            \\
            \vdots       & \vdots       & \ddots & \vdots           & \vdots
            \\
            \bm{0}       & \bm{0}       & \cdots & \bm{T_{k-1,k-1}} & \bm{T_{k-1,k}}
            \\
            \bm{0}       & \bm{0}       & \cdots & \bm{0}           & \bm{T_{k,k}}
        \end{bmatrix} ,
    \end{equation}
    where $\bm{T_{i,i}} \in \mathbb{R}^{1 \times 1} \cup \mathbb{R}^{2 \times 2}$ and $\left\lceil \frac{n}{2} \right\rceil \le k \le n$.
\end{theorem}


This method could, in theory, be used as is to constrain the state matrix, as it provides rigorous asymptotic-stability guarantees (unlike regularization methods) without the over-parameterization overhead introduced by \gls{lmi} methods such as \gls{simba}. However, doing so would require solving a computationally-expensive, numerically-sensitive surrogate optimization problem for every update of $\bm{A}$ (i.e., once per training batch). Thus, its applicability to \gls{ml} is limited in its current form.

\subsection{The Truncated Nearest Schur-Stable Matrix Projection} \label{subsec:schur_ss_trunc}

\begin{algorithm}[tb!]
    \caption{Quasi-triangular Schur factor projection.}
    \label{alg:schur_proj}

    \KwIn{Schur factor $\bm{T} \in \mathbb{R}^{n \times n}$ (quasi-triangular)}
    \KwOut{Stable Schur factor $\bm{\hat{T}} \in \mathbb{R}^{n \times n}$}

    \Begin{
        Initialize $i \gets 1, \bm{\hat{T}} \gets \bm{T}$ \;
        \For(\tcp*[f]{$k$ from \Cref{theorem:real_schur}}){$i \gets 1$ to $k$}{
            \eIf{$\bm{T_{i,i}} \in \mathbb{R}^{1 \times 1}$}{
                $\bm{\hat{T}_{i,i}} \gets \nicefrac{\bm{T_{i,i}}}{\max\{\bm{1}, |\bm{T_{i,i}}|\}}$ \;
            }(\tcp*[f]{$2 \times 2$ block}){
                Initialize $\bm{X} \gets \bm{0_{2 \times 2}}, d_{\mathrm{min}}^2 \gets \infty$ \;
                \ForEach{$\bm{X'} \in \mathcal{\hat{S}}(\bm{T_{i,i}})$}{
                    $d^2 \gets \| \bm{T_{i,i}} - \bm{X'} \|_F^2$ \;
                    \lIf{$d^2 \ge d_{\mathrm{min}}^2$}{\Continue}
                    $\lambda_1, \lambda_2 \gets \Lambda_{\bm{X'}}$ \;
                    \lIf{$|\lambda_1| > 1$ \Or $|\lambda_2| > 1$}{\Continue}
                    $d_{\mathrm{min}}^2 \gets d^2$; $\bm{X} \gets \bm{X'}$
                }
                $\bm{\hat{T}_{i,i}} \gets \bm{X}$
            }
        }
    }
\end{algorithm}

To make the method proposed in \citet{noferini2021nearest} feasible for \gls{ss} \gls{sysid}, we restrict the search space to stable matrices sharing the same orthogonal factor. Under this constraint, $\bm{\hat{A}}$ is found by first calculating the Schur decomposition $\bm{A} = \bm{Z} \bm{T} \bm{Z}^\intercal$ and then applying the projection scheme in \Cref{alg:schur_proj} to the quasi-triangular factor $\bm{T}$, which is equivalent to performing a single iteration of the bi-level optimization problem at the end of every training batch. This approach is not only computationally efficient, but also theoretically sound, as one of the results in \citet{noferini2021nearest} is that this projection is guaranteed to yield the nearest quasi-triangular Schur-stable matrix.

To evaluate the performance of the fast approximation relative to the original algorithm, the following metrics are considered:
\begin{enumerate}
    \item The \gls{nsfe}, defined as
          \begin{equation}
              \mathrm{NSFE}(\bm{A}, \bm{X}) = \frac{\left\| \bm{A} - \bm{X} \right\|_F^2}{\left\| \bm{A} \right\|_F^2},
          \end{equation}
          normalizes the objective function in \Cref{eq:nearest_obj_function} to be matrix-size agnostic. The optimal value of this metric is not zero, as this would require the target matrix $\bm{A}$ to be stable. Instead, the \gls{nsfe} of the truncated method should be evaluated relative to the one yielded by the bi-level optimization algorithm.
    \item The \gls{nssr}, defined as the optimal cost of the integer linear program
          \begin{equation}
              \min_{\bm{C} \in \{0,1\}^{n \times n}} \frac{ \sum_{\substack{\lambda_i \in \Lambda_{\bm{X}} \\ \lambda_j \in \Lambda_{\bm{A}}}} c_{ij} |\lambda_i - \lambda_j|^2}{\sum_{\lambda_j \in \Lambda_{\bm{A}}} |\lambda_j|^2} ,
              \quad s.t. \quad
              \sum_{i=1}^n c_{ij} = n \quad \forall j , \quad
              \sum_{j=1}^n c_{ij} = n \quad \forall i ,
          \end{equation}
          measures the average spectral distance relative to the original matrix eigenvalues by matching the nearest one in the approximation. This metric quantifies the \textit{dynamics' distortion} induced by the projection method. Furthermore, with the state-observation matrix $\bm{C}$ being a free parameter, the state order of the projected matrix is not a concern. Just like the \gls{nsfe}, this metric should be evaluated relative to the full projection, not to 0.
    \item The \gls{msvr}, defined as
          \begin{equation}
              \mathrm{MSVR}(\bm{X}) = \frac{1}{|\Lambda_{\bm{X}}|}\sum_{\lambda_i \in \Lambda_{\bm{X}}}\max\{|\lambda_i| - 1, 0\}^2,
          \end{equation}
          provides a size-agnostic measure of the compliance with the Schur-stability bounds that mirrors the regularization term from \citet{bemporad2025bfgs} in \Cref{eq:regularization_bemporad}. Since both projection methods are theoretically guaranteed to be stable, this metric evaluates the numerical stability of the algorithm in practice.
    \item The execution time. It is worth noting that the truncated projection is \gls{jit}-compiled, which is not possible for the Bi-level optimization due to its dynamic stop condition.
\end{enumerate}

\begin{table}[tb!]
    \caption{Performance metrics for the matrix-approximation scheme from \citet{noferini2021nearest} and the truncated projection. Hardware specs and extended benchmarks in \Cref{app:projection}.}
    \label{tb:projection}
    \centering
    \begin{tblr}{
        colspec={*{6}{c}},
        colsep={3pt},
        cells={valign=m, font=\footnotesize},
        row{1}={font={\bfseries \boldmath \footnotesize}},
        hline{1,every[2]{2}{-1}}={},
        cell{every[2]{2}{-2}}{1}={r=2}{},
        cell{2-Z}{1,3-Z}={mode=math},
            }
        n   & Method                    & NSFE                 & NSSR                 & MSVR                  & Time [s]
        \\
        10  & Bi-level Optimization     & \expnumber{4.08}{-1} & \expnumber{5.57}{-1} & \expnumber{2.18}{-11} & \expnumber{1.02}{3}
        \\
            & Block-Diagonal Projection & \expnumber{6.67}{-1} & \expnumber{4.06}{-1} & \expnumber{2.43}{-6}  & \expnumber{2.43}{-6}
        \\
        20  & Bi-level Optimization     & \expnumber{4.55}{-1} & \expnumber{6.83}{-1} & \expnumber{9.10}{-9}  & \expnumber{1.24}{3}
        \\
            & Block-Diagonal Projection & \expnumber{3.42}{-1} & \expnumber{5.54}{-1} & \expnumber{1.36}{-6}  & \expnumber{9.49}{-5}
        \\
        50  & Bi-level Optimization     & \expnumber{4.42}{-1} & \expnumber{7.48}{-1} & \expnumber{1.36}{-3}  & \expnumber{8.20}{3}
        \\
            & Block-Diagonal Projection & \expnumber{3.24}{-1} & \expnumber{6.13}{-1} & \expnumber{5.66}{-2}  & \expnumber{4.76}{-5}
        \\
        100 & Bi-level Optimization     & \expnumber{4.85}{-1} & \expnumber{7.61}{-1} & \expnumber{2.65}{-2}  & \expnumber{7.05}{4}
        \\
            & Block-Diagonal Projection & \expnumber{3.73}{-1} & \expnumber{6.28}{-1} & \expnumber{4.17}{-1}  & \expnumber{7.03}{-5}
    \end{tblr}
\end{table}

\Cref{tb:projection} reports the above metrics for the projection of various matrices with normally-distributed random coefficients, with additional cases available in \Cref{app:projection}. Since the orthogonal factor $\bm{Z}$ is no longer a free variable, \Cref{alg:schur_proj} initially incurs a higher \gls{nsfe}; yet this advantage vanishes for larger matrices as the execution-time limit of the manifold optimization leads to an early interruption of the minimization process. Moreover, the \gls{nssr} results indicate that this approach yields a better approximation of the original matrix eigenvalues, which is a more relevant property for sequence modelling. Finally, and more importantly, the full projection is several orders of magnitude slower than \Cref{alg:schur_proj}, making the former prohibitively expensive to use as a weight constraint.

The truncated projection serves as an effective weight constraint when fitting models containing structured \gls{ss} layers. For a Hammerstein-Wiener architecture, such as the one shown in \Cref{subfig:hw}, the identification process involves solving the following optimization problem:
\begin{equation} \label{eq:projected_obj}\min_{
        \bm{\Theta} \triangleq
        \left\{
        \begin{bsmallmatrix}
            \bm{A} & \bm{B}
            \\
            \bm{C} & \bm{D}
        \end{bsmallmatrix}
        , \bm{\theta_f} , \bm{\theta_g}
        \right\}
    } \ell(\bm{\hat{\Theta}} ; \bm{y}, \bm{x}[0])
    \quad s.t. \quad
    \bm{\hat{\Theta}} = \left\{
    \begin{bmatrix}
        \bm{\hat{A}} & \bm{B}
        \\
        \bm{C}       & \bm{D}
    \end{bmatrix}
    , \bm{\theta_f} , \bm{\theta_g}
    \right\},
\end{equation}
where $\ell$ mirrors the function defined in \Cref{eq:regularized_obj}, although driven by the output of the Hammerstein-Wiener model; $\bm{\theta_f}$ and $\bm{\theta_g}$ denote the parameters of the input and output nonlinearities, respectively; and $\bm{\hat{A}} = \bm{Z} \bm{\hat{T}} \bm{Z}^\intercal$ is calculated using the output of \Cref{alg:schur_proj}. To generalize this definition, it is sufficient to apply the stabilization algorithm to the state matrix of each \gls{ss} layer in the \gls{nn}.

Alternatively, $\bm{A}$ can be parameterized directly through its Schur decomposition factors, removing the computational overhead of having to calculate the Schur factorization for every training step. In this case, $\bm{Z} \in \mathbb{R}^{n_x \times n_x}$ is constrained to be orthogonal via the \gls{svd} method from \citet{golub1996matrix} (see \Cref{app:orthogonal}), while $\bm{\hat{T}}$ is stabilized using \Cref{alg:schur_proj}, zeroing the elements below the block-diagonal. For simplicity, and without loss of generality since $\bm{\hat{Z}}$ is also a free parameter, $\bm{\hat{T}}$ is assumed to consist of $2 \times 2$ blocks, with a final $1 \times 1$ element in the bottom-right if and only if $n_x$ is odd.

Regardless of the parameterization, compared to \gls{simba}, the proposed methods trade computational complexity for weight efficiency. Indeed, while the former demands $5n_x^2$ weights to yield a Schur-stable state matrix, the dynamically decomposed and pre-factorized approaches require only $n_x^2$ and $2n_x^2$ weights, respectively. Moreover, unlike regularization methods, \Cref{alg:schur_proj} guarantees Schur stability (up to numerical error) without requiring relative-strength coefficient tuning.

\section{Experimental Results} \label{sec:exp}

In this section, the following stable-identification methods were evaluated: 1. Weight constraint over the state matrix following \Cref{alg:schur_proj}, denoted \textit{Schur (Proj.)} hereinafter; 2. Weight constraint over the state-matrix factors, using \Cref{alg:schur_proj} for the quasi-triangular matrix and the \gls{svd} method for the orthogonal factor, denoted \textit{Schur (Built)} henceforth; 3. \gls{simba} parameterization~\cite{di2024simba}; 4. Weight regularization~\cite{bemporad2025bfgs}, as in \Cref{eq:regularization_bemporad}, denoted \textit{Regularized} in tables and figures.

Large \gls{ss} architectures, such as \gls{s4} and \gls{mamba}, were excluded as baselines due to their hardware overhead and lack of stability guarantees~\cite{salam2025comprehensive}. The implementation was carried out in \textit{Keras 3} with the \textit{JAX} backend. Despite experiments being conducted on HPC CPU cores (hardware specs in \Cref{subsec:synthetic_description,subsec:real_description}), as CPU execution is faster for these dataset sizes~\cite{di2024simba}, the codebase remains GPU-compatible via a custom Schur decomposition implementation (see \Cref{app:gpu}).

\subsection{Synthetic-Data Benchmarks} \label{subsec:synthetic_exp}

The first experiment involves the replication and extension of the random-system benchmarks from \citet[Section V.A]{di2024simba}. The \textbf{original} setup consists of 50 stable discrete-time systems ($n_x=5, n_y=3, n_u=3$) with a single trajectory of 300 samples per partition (training, validation, and testing). Optimization is performed for 50,000 epochs using \textit{AdamW}~\cite{loshchilov2018decoupled} (initial learning rate $\eta_0 = \expnumber{1}{-3}$), retaining the weights yielding the best validation loss. Inputs follow a \gls{gbn} distribution ($p=0.1$), and unbiased Gaussian noise ($\sigma = 0.25$) is added to the training partition's output.

An analysis of the \gls{simba} codebase revealed that the generated systems lacked complex-conjugate eigenvalues, which is not representative of most physical systems. Furthermore, a low sample-to-parameter ratio was identified: a \gls{simba} layer of the aforementioned system size requires $5n_x^2 + n_x (n_u + n_y) + n_u n_y = 164$ parameters, which is more than half of the training sample count. Although the proposed Schur-based formulations are more efficient (64 and 89 parameters for the direct and parameterized formulations, respectively), additional configurations were introduced to ensure a comprehensive performance evaluation: \textbf{Extended} increases the number of samples and incorporates early stopping along with a checkpoint mechanism that restores the weights yielding the best validation loss; \textbf{Small} and \textbf{Smaller} reduce the total number of samples to \textbf{128} and \textbf{64}, respectively, to evaluate performance under low data availability; \textbf{Large} features increased system dimensionality to $(n_x, n_y, n_u) = (10, 6, 6)$ and splits the dataset into multiple sequences per partition to evaluate the methods' performance with larger $n_x$ values and batched data.

\begin{table}[tb!]
    \centering
    \caption{Median numerical results for the synthetic-data benchmarks (uncertainty defined as half of the \gls{iqr}). Epoch time is measured in seconds. Row-wise best is highlighted.}
    \label{tb:synthetic_benchmark}
    \begin{tblr}{
        colspec={*{6}{c}},
        colsep={3pt},
        cells={valign=m,font=\footnotesize},
        row{1}={font={\footnotesize \bfseries}},
        cell{every[4]{2}{-4}}{1}={r=4}{},
        cell{2-Z}{3-Z}={mode=math},
        cell{13}{3}={c=4}{},
        hline{1,every[4]{2}{-1}}={},
        }
        Setup    & Metric      & Schur (Proj.)                        & Schur (Built)                        & \gls{simba}                          & Regularized
        \\
        Smaller  & NMSE        & \expnumber{(5.91 \pm 5.31)}{-4}      & \expnumber{(7.14 \pm 8.32)}{-4}      & \bm{\expnumber{(3.78 \pm 3.45)}{-4}} & \expnumber{(5.25 \pm 5.01)}{-4}
        \\
                 & NSSR        & \expnumber{(1.04 \pm 5.81)}{-2}      & \expnumber{(7.02 \pm 11.9)}{-2}      & \bm{\expnumber{(1.07 \pm 5.43)}{-2}} & \expnumber{(2.03 \pm 7.62)}{-2}
        \\
                 & Epoch Time  & \expnumber{(4.61 \pm 0.30)}{-3}      & \expnumber{(5.06 \pm 0.32)}{-3}      & \bm{\expnumber{(4.37 \pm 0.27)}{-3}} & \expnumber{(4.41 \pm 0.28)}{-3}
        \\
                 & Epoch Count & \expnumber{(1.84 \pm 0.26)}{4}       & \bm{\expnumber{(1.77 \pm 0.15)}{4}}  & \expnumber{(1.79 \pm 0.20)}{4}       & \expnumber{(2.13 \pm 0.22)}{4}
        \\
        Small    & NMSE        & \bm{\expnumber{(5.00 \pm 4.18)}{-4}} & \expnumber{(5.83 \pm 4.07)}{-4}      & \expnumber{(5.30 \pm 4.18)}{-4}      & \expnumber{(5.59 \pm 4.75)}{-4}
        \\
                 & NSSR        & \expnumber{(2.23 \pm 8.92)}{-2}      & \expnumber{(5.74 \pm 11.5)}{-2}      & \bm{\expnumber{(1.69 \pm 7.27)}{-2}} & \expnumber{(5.47 \pm 13.5)}{-2}
        \\
                 & Epoch Time  & \expnumber{(4.96 \pm 0.33)}{-3}      & \expnumber{(5.33 \pm 0.34)}{-3}      & \expnumber{(4.65 \pm 0.30)}{-3}      & \bm{\expnumber{(4.60 \pm 0.30)}{-3}}
        \\
                 & Epoch Count & \expnumber{(1.86 \pm 0.28)}{4}       & \expnumber{(1.76 \pm 0.20)}{4}       & \bm{\expnumber{(1.73 \pm 0.16)}{4}}  & \expnumber{(2.04 \pm 0.21)}{4}
        \\
        Original & NMSE        & \expnumber{(4.97 \pm 7.92)}{-5}      & \expnumber{(3.87 \pm 7.00)}{-5}      & \bm{\expnumber{(3.43 \pm 6.04)}{-5}} & \expnumber{(3.79 \pm 6.99)}{-5}
        \\
                 & NSSR        & \expnumber{(3.60 \pm 10.3)}{-2}      & \bm{\expnumber{(9.27 \pm 33.0)}{-3}} & \expnumber{(1.50 \pm 6.11)}{-2}      & \expnumber{(1.72 \pm 8.46)}{-2}
        \\
                 & Epoch Time  & \expnumber{(5.20 \pm 0.39)}{-3}      & \expnumber{(5.89 \pm 0.46)}{-3}      & \expnumber{(5.23 \pm 0.40)}{-3}      & \bm{\expnumber{(5.17 \pm 0.38)}{-3}}
        \\
                 & Epoch Count & \bm{\expnumber{(5.00 \pm 0.00)}{4}}
        \\
        Extended & NMSE        & \expnumber{(5.10 \pm 2.66)}{-4}      & \bm{\expnumber{(5.02 \pm 2.93)}{-4}} & \expnumber{(6.11 \pm 2.29)}{-4}      & \expnumber{(5.29 \pm 2.96)}{-4}
        \\
                 & NSSR        & \bm{\expnumber{(1.18 \pm 5.17)}{-2}} & \expnumber{(3.80 \pm 10.5)}{-2}      & \expnumber{(1.32 \pm 6.67)}{-2}      & \expnumber{(5.17 \pm 8.19)}{-2}
        \\
                 & Epoch Time  & \expnumber{(7.47 \pm 0.56)}{-3}      & \expnumber{(8.08 \pm 0.67)}{-3}      & \expnumber{(7.51 \pm 0.62)}{-3}      & \bm{\expnumber{(7.35 \pm 0.55)}{-3}}
        \\
                 & Epoch Count & \expnumber{(1.68 \pm 0.17)}{4}       & \expnumber{(1.68 \pm 0.14)}{4}       & \bm{\expnumber{(1.62 \pm 0.12)}{4}}  & \expnumber{(1.91 \pm 0.23)}{4}
        \\
        Large    & NMSE        & \expnumber{(4.49 \pm 1.55)}{-4}      & \bm{\expnumber{(4.09 \pm 1.93)}{-4}} & \expnumber{(6.49 \pm 2.63)}{-4}      & \expnumber{(5.19 \pm 1.22)}{-4}
        \\
                 & NSSR        & \expnumber{(4.61 \pm 5.35)}{-2}      & \expnumber{(1.11 \pm .50)}{-1}       & \expnumber{(5.46 \pm 5.54)}{-2}      & \bm{\expnumber{(3.64 \pm 3.86)}{-2}}
        \\
                 & Epoch Time  & \expnumber{(2.40 \pm 0.11)}{-2}      & \expnumber{(2.47 \pm 0.90)}{-2}      & \expnumber{(2.34 \pm 0.12)}{-2}      & \bm{\expnumber{(2.31 \pm 0.10)}{-2}}
        \\
                 & Epoch Count & \expnumber{(3.48 \pm 0.17)}{4}       & \expnumber{(3.40 \pm 0.26)}{4}       & \bm{\expnumber{(3.11 \pm 0.28)}{4}}  & \expnumber{(3.45 \pm 0.16)}{4}
    \end{tblr}
\end{table}

Furthermore, the output-noise scale for the training partition of the new cases was reduced to $\sigma = 0.01$ due to the high level of uncertainty relative to the expected error, which was an issue as the accuracy gap between methods is too narrow. A detailed description of the experimental parameters for each case is available in \Cref{subsec:synthetic_description}. The numerical results from the experiments above are shown in \Cref{tb:synthetic_benchmark}, whereas supplementary figures are delegated to \Cref{subsec:synthetic_supplementary}.

Despite the drastic reduction in output-noise scale, all evaluated methods demonstrate comparable performance in terms of \gls{nmse}, managing to correctly identify the underlying dynamics, with each method's expected tracking error falling within the uncertainty range of its peers. Similar conclusions can be drawn from the \gls{nssr} and epoch time, although the latter tends to be smaller for the regularized method due to its simplicity in formulation and the low relative computational cost of the spectral norm for the considered $n_x$ values. The true discriminant is the epoch count, which is consistently larger for the regularized method as it does not strictly guarantee stable dynamics, to the detriment of the model's convergence rate.

\subsection{Real-Data Benchmarks} \label{subsec:real_exp}

The second round of experiments adapts the nonlinear \gls{sysid} benchmark from \citet[Section IV.C]{bemporad2025bfgs} to a stacked Hammerstein-Wiener structure (see \Cref{subfig:hw}), as opposed to the original single-layer discrete-time \gls{nlti} \gls{ss} model, and extends it to 4 additional datasets~\cite{schoukens2026NLBench} and model configurations (see \Cref{subsec:real_description}). To accommodate these architectural changes, the original step-wise training scheme was replaced by a single-stage backpropagation cycle: assuming $\bm{x}[0] = \bm{0_{n_x}}$, each method was tested over five random initializations to minimize the \gls{nmse} using the \textit{AdamW} optimizer ($\eta_0 = \expnumber{1}{-3}$) for up to 50,000 epochs (early stopping with patience of 10,000 epochs). Consistency with the original experiment was maintained by calculating the optimal initial state on the validation and test partitions via an \gls{ekf} + \gls{rtss} algorithm~\cite{bemporad2025bfgs}, using the former to select the best model and the latter to get the output sequences over which the error metrics are calculated.

\begin{table*}[tb!]
    \caption{Numerical results for the real-data benchmarks. Error uncertainty is defined as the standard deviation of the normalized squared differences. Median epoch time in seconds (uncertainty defined as half of the \gls{iqr}). Epoch count reported for the best-performing model. Row-wise best is highlighted.}
    \label{tb:real_benchmark}
    \centering
    \begin{tblr}{
        colspec={X[c]*{5}{c}},
        colsep={3pt},
        cells={valign=m, font=\footnotesize},
        row{1}={font={\footnotesize \bfseries}},
        cell{2-Z}{3-Z}={mode=math},
        cell{every[3]{2}{-3}}{1}={r=3}{},
        hline{1,every[3]{2}{-1}}={},
        }
        Dataset              & Metric      & Schur (proj.)                        & Schur (built)                        & SIMBa                                & Regularized
        \\
        Silverbox            & Test NMSE   & \expnumber{(1.52 \pm 7.86)}{-2}      & \expnumber{(1.51 \pm 7.87)}{-2}      & \bm{\expnumber{(1.50 \pm 7.97)}{-2}} & \expnumber{(1.53 \pm 7.49)}{-2}
        \\
                             & Epoch Time  & \expnumber{(1.59 \pm 0.02)}{-1}      & \expnumber{(1.42 \pm 0.01)}{-1}      & \bm{\expnumber{(1.39 \pm 0.06)}{-1}} & \expnumber{(1.43 \pm 0.03)}{-1}
        \\
                             & Epoch Count & \expnumber{8.36}{3}                  & \bm{\expnumber{6.89}{3}}             & \expnumber{1.40}{4}                  & \expnumber{1.05}{4}
        \\
        CED                  & Test NMSE   & \bm{\expnumber{(3.30 \pm 7.40)}{-2}} & \expnumber{(3.70 \pm 8.79)}{-2}      & \expnumber{(4.12 \pm 9.92)}{-2}      & \expnumber{(4.08 \pm 8.49)}{-2}
        \\
                             & Epoch Time  & \bm{\expnumber{(6.14 \pm 0.27)}{-3}} & \expnumber{(6.80 \pm 0.16)}{-3}      & \expnumber{(6.29 \pm 0.15)}{-3}      & \expnumber{(6.56 \pm 0.10)}{-3}
        \\
                             & Epoch Count & \expnumber{1.14}{4}                  & \bm{\expnumber{7.83}{3}}             & \expnumber{2.65}{4}                  & \expnumber{1.75}{4}
        \\
        EMPS                 & Test NMSE   & \expnumber{(3.21 \pm 3.24)}{-3}      & \bm{\expnumber{(3.07 \pm 4.51)}{-3}} & \expnumber{(4.86 \pm 5.78)}{-3}      & \expnumber{(2.50 \pm 5.18)}{1}
        \\
                             & Epoch Time  & \expnumber{(7.49 \pm 0.42)}{-3}      & \expnumber{(7.46 \pm 0.42)}{-3}      & \expnumber{(6.51 \pm 0.10)}{-3}      & \bm{\expnumber{(6.32 \pm 0.09)}{-3}}
        \\
                             & Epoch Count & \expnumber{3.16}{4}                  & \expnumber{1.33}{4}                  & \bm{\expnumber{1.31}{4}}             & \expnumber{4.78}{4}
        \\
        Industrial Robot     & Test NMSE   & \expnumber{(6.63 \pm 12.3)}{-1}      & \bm{\expnumber{(5.76 \pm 10.9)}{-1}} & \expnumber{(7.25 \pm 17.5)}{-1}      & \expnumber{(7.40 \pm 17.5)}{-1}
        \\
                             & Epoch Time  & \bm{\expnumber{(1.78 \pm 0.01)}{-1}} & \expnumber{(1.83 \pm 0.01)}{-1}      & \expnumber{(1.82 \pm 0.04)}{-1}      & \expnumber{(1.86 \pm 0.02)}{-1}
        \\
                             & Epoch Count & \expnumber{1.25}{4}                  & \bm{\expnumber{1.20}{4}}             & \expnumber{2.13}{4}                  & \expnumber{2.95}{4}
        \\
        Fine-Steering Mirror & Test NMSE   & \expnumber{(4.76 \pm 9.31)}{-2}      & \expnumber{(3.96 \pm 7.76)}{-2}      & \bm{\expnumber{(3.62 \pm 6.68)}{-2}} & \expnumber{(3.99 \pm 7.82)}{-2}
        \\
                             & Epoch Time  & \expnumber{(3.03 \pm 0.04)}{-1}      & \expnumber{(3.02 \pm 0.04)}{-1}      & \bm{\expnumber{(2.43 \pm 0.02)}{-1}} & \expnumber{(3.03 \pm 0.25)}{-1}
        \\
                             & Epoch Count & \expnumber{2.91}{4}                  & \bm{\expnumber{2.33}{4}}             & \expnumber{2.68}{4}                  & \expnumber{4.54}{4}
        \\
    \end{tblr}
\end{table*}

The numerical results for the best-performing initialization per method are detailed in \Cref{tb:real_benchmark}, with supplementary results available in \Cref{subsec:real_supplementary}. These reinforce the analysis made in \Cref{subsec:synthetic_exp}: the expected accuracy of all methods falls within each other's uncertainty range (with the EMPS dataset being an outlier), yet the regularized method exhibits the slowest convergence rate. More notably, since the \textit{Fine-Steering Mirror} dynamics are approximated using the largest model configuration (see \Cref{subsec:real_description}), the cost of calculating the matrix norm at each training step makes the epoch time of the regularized method equal to that of the Schur methods. From these two experimental stages, it is possible to place the proposed Schur decomposition-based methods on par with \gls{simba}, and to deem the former better choices w.r.t. the regularized method for deep stable \gls{sysid}.

\section{Conclusions} \label{sec:conclusions}

In this paper, we presented a novel, backpropagation-ready weight-projection method for estimating discrete \gls{lti} \gls{ss} layers that leverages the stabilization scheme from \citet{noferini2021nearest}. By introducing a direct projection method over the state-transition matrix, as well as an alternative pre-factorized formulation, this work provides a robust framework for guaranteeing \gls{ss}-layer stability. This approach distinguishes itself through its ability to ensure stable dynamics without requiring hyperparameter tuning or the introduction of significant over-parameterization, which has insofar been a limitation in the literature. Experimental results across various scenarios, ranging from shallow linear systems to complex nonlinear models, demonstrate the practical utility of the proposed methodology. Both formulations consistently achieve performance on par with the \gls{sota} in terms of accuracy, epoch time, and convergence rate, suggesting that the proposed Schur-based stabilization is not only theoretically sound but also computationally efficient for stacked-model \gls{sysid}.

Future research will focus on further generalizing the projection algorithm to support arbitrary upper bounds on eigenvalue magnitudes, providing the same level of flexible stability control as \gls{simba} while maintaining the efficiency inherent to the Schur-based approach. Furthermore, the development of a more efficient GPU-based implementation for the Schur decomposition remains a hurdle to be surpassed so that the methodology scales effectively for larger datasets and architectures. Finally, more niche \gls{lmi} methods, such as the one presented in \citet{najson2012kalman}, remain insufficiently evaluated for \gls{ml} and could potentially serve as more efficient weight constraints for discrete-time \gls{ss} models.

\begin{ack}
    Funded by the European Union. Views and opinions expressed are, however, those of the author(s) only and do not necessarily reflect those of the European Union or the European Research Executive Agency (REA). Neither the European Union nor the granting authority can be held responsible for them.

    This programme has received funding from the European Union through Marie Skłodowska-Curie actions under project number 101081466: SEED - Systems and Engineering Science Doctorate.

    The authors wish to acknowledge \textit{CSC \textendash IT Center for Science, Finland}, for computational resources.

    Finally, the authors wish to deeply thank Mr. Peter Jones for his support and writing advice during the preparation of the camera-ready version of the manuscript.
\end{ack}

{
\small
\bibliographystyle{unsrtnat}
\bibliography{main}
}

\appendix

\section{Computing the Nearest Real 2-by-2 Schur-Stable Matrix} \label[appendix]{app:nearest_schur_stable}

This appendix is included for the sake of convenience, but it is taken almost verbatim from \citet[Section 5]{noferini2021nearest}. Some symbols, however, are modified to fit the notation used in this manuscript. Some proofs originally included in the source material are omitted.

First, some auxiliary results useful for the definition of the stable-candidate set $\mathcal{\hat{S}}(\bm{A})$ are presented. It is worth recalling, however, that not all elements in $\mathcal{\hat{S}}$ are guaranteed to be stable (e.g., trivially, $\bm{A} \in \mathcal{\hat{S}}(\bm{A})$), so they have to be filtered based on their eigenvalues. This property, however, is trivial to verify for $2 \times 2$ matrices, as presented in \Cref{lemma:eigv_2by2}.

\begin{lemma} \label[lemma]{lemma:eigv_2by2}
    Let $\bm{A} = \begin{bsmallmatrix}
            A_{1,1} & A_{1,2} \\ A_{2,1} & A_{2,2}
        \end{bsmallmatrix} \in \mathbb{R}^{2 \times 2}$. Then, its eigenvalues are given by
    \begin{equation} \label{eq:eigv_2by2}
        \lambda_\pm = \frac{\mathrm{tr}(\bm{A}) \pm \sqrt{\mathrm{tr}(\bm{A})^2 - 4 \cdot \mathrm{det}(\bm{A})}}{2}.
    \end{equation}
    \begin{proof}
        The eigenvalues of $\bm{A}$ are the roots of the characteristic polynomial
        \begin{equation}
            \begin{split}
                p(\lambda)
                 & = \mathrm{det}(\bm{A} - \lambda \bm{I})
                \\
                 & = \left| \begin{bsmallmatrix} A_{1,1} - \lambda & A_{1,2} \\ A_{2,1} & A_{2,2} - \lambda \end{bsmallmatrix} \right|
                \\
                 & = (A_{1,1} - \lambda) (A_{2,2} - \lambda) - A_{1,2} A_{2,1}
                \\
                 & = \lambda^2 + \underbrace{(- A_{1,1} - A_{2,2})}_{- \mathrm{tr}(\bm{A})} \lambda + \underbrace{(A_{1,1} A_{2,2} - A_{1,2} A_{2,1})}_{\mathrm{det}(\bm{A})} .
            \end{split}
        \end{equation}
        Applying the quadratic formula to $p(\lambda)$ yields \Cref{eq:eigv_2by2}.
    \end{proof}
\end{lemma}

Given $\alpha \in \mathbb{R}$, the 2D rotation matrix $\bm{U}(\alpha)$ is defined
\begin{equation}
    \bm{U}(\alpha) \triangleq \begin{bmatrix}
        \cos \alpha & -\sin \alpha
        \\
        \sin \alpha & \cos \alpha
    \end{bmatrix} ,
\end{equation}
where $\alpha \in \mathbb{R}$ is the rotation angle.

\begin{lemma} \label[lemma]{lemma:rotation}
    Let $\bm{A} \in \mathbb{R}^{2 \times 2}$. Then, $\exists \alpha \in [0, \pi/2)$ such that $\bm{\check{A}} = \bm{U}(\alpha) \bm{A} \bm{U}(\alpha)^\intercal$ has $\check{A}_{1,1} = \check{A}_{2,2} = \frac{1}{2} \mathrm{tr}(\bm{\check{A}})$.
\end{lemma}

A non-empty closed set $\mathcal{S} \subseteq \mathbb{R}^{2 \times 2}$ is said to be rotation-invariant if $\bm{X} \in \mathcal{S} \implies \bm{U}(\alpha) \bm{X} \bm{U}(\alpha)^\intercal \in \mathcal{S} \quad \forall \alpha \in \mathbb{R}$.

\begin{lemma} \label[lemma]{lemma:rotation_invariant}
    Let $\mathcal{S} \subseteq \mathbb{R}^{2 \times 2}$ be rotation-invariant, and let $\bm{\hat{A}}$ be a local minimizer of $$\min_{\bm{X} \in \mathcal{S}} \| \bm{A} - \bm{X} \|_F$$ for some $\bm{A} \in \mathbb{R}^{2 \times 2}$ satisfying $A_{1,1} = A_{2,2}$. Then:
    \begin{enumerate}
        \item If $A_{2,1} \neq -A_{1,2}$, then $\hat{A}_{1,1} = \hat{A}_{2,2}$.
        \item If $A_{2,1} = -A_{1,2}$, then $\exists \alpha \in [0, \pi / 2) : \bm{\hat{A}} = \bm{U}(\alpha) \bm{\check{A}} \bm{U}(\alpha)^\intercal$, where $\bm{\check{A}}$ is another local minimizer with the same objective value and $\check{A}_{1,1} = \check{A}_{2,2}$.
    \end{enumerate}
\end{lemma}

A non-empty closed set $\mathcal{S} \subseteq \mathbb{R}^{2 \times 2}$ is doubly rotation-invariant if $\bm{X} \in \mathcal{S} \implies \bm{U}(\alpha) \bm{X} \bm{U}(\beta)^\intercal \in \mathcal{S} \quad \forall \{ \alpha, \beta\} \in \mathbb{R}^2$.

\begin{lemma} \label[lemma]{lemma:doubly_rotation_invariant}
    Let $\mathcal{S} \in \mathbb{R}^{2 \times 2}$ be doubly rotation-invariant, and let $\bm{\hat{A}}$ be a local minimizer of $$\min_{\bm{X} \in \mathcal{S}} \| \bm{A} - \bm{X} \|_F$$ for some diagonal $\bm{A} = diag(\sigma_1, \sigma_2) = \begin{bsmallmatrix}
            \sigma_1 & 0 \\ 0 & \sigma_2
        \end{bsmallmatrix} \in \mathbb{R}^{2 \times 2}$ with $\sigma_1 > 0$ and $\sigma_1 \ge \sigma_2 \ge 0$. Then:
    \begin{enumerate}
        \item If $\sigma_1 \neq \sigma_2$, then $\bm{\hat{A}}$ is diagonal with $\hat{A}_{1,1} - \hat{A}_{2,2} \ge 0$ and $\hat{A}_{1,1} + \hat{A}_{2,2} \ge 0$.
        \item If $\sigma_1 = \sigma_2$, then $\exists \alpha \in [0, \pi/2) : \bm{\hat{A}} = \bm{U}(\alpha) \bm{\check{A}} \bm{U}(\alpha)^\intercal$, where $\bm{\check{A}}$ is another local minimizer with the same objective value and $\check{A}_{1,1} = \check{A}_{2,2} = 0$.
    \end{enumerate}
\end{lemma}

The result in \Cref{lemma:routh_hurwitz} is a variant of the Routh-Hurwitz stability criterion (see \citet[Sec. 26.2]{Hogben2013-hi}).

\begin{lemma} \label[lemma]{lemma:routh_hurwitz}
    Both roots of the polynomial $p(\lambda) = a \lambda^2 + b \lambda + c$, with $a \neq 0$, lie in the closed left half-plane $\{\lambda \in \mathbb{C} : \mathfrak{Re}(\lambda) \le 0\}$ if and only if $a, b, c$ are either all non-negative or all non-positive.
\end{lemma}

From the above results, the set of $2 \times 2$ Schur-stable matrices is explicitly defined as in \Cref{lemma:stable_real}.

\begin{lemma} \label[lemma]{lemma:stable_real}
    $\mathcal{S}(\Omega_S, 2, \mathbb{R}) = \mathcal{S}_- \cap \mathcal{S}_0 \cap \mathcal{S}_+$, where
    \begin{subequations} \label{eq:schur_stable_subsets}
        \begin{align}
            \mathcal{S}_0   & = \{ \bm{X} \in \mathbb{R}^{2 \times 2} : \mathrm{det}(\bm{X}) \le 1 \}
            \\
            \mathcal{S}_\pm & = \{ \bm{X} \in \mathbb{R}^{2 \times 2} : \pm \mathrm{tr}(\bm{X}) \le 1 + \mathrm{det}(\bm{X}) \}
        \end{align}
    \end{subequations}

    \begin{proof}
        Let $t = \mathrm{tr}(\bm{X})$ and $d = \mathrm{det}(\bm{X})$ for brevity. The matrix $\bm{X}$ is Schur stable if and only if its characteristic polynomial $p(\lambda) = \lambda^2 - t \lambda + d$ has both roots inside the closed unit disk. This is equivalent to imposing that its Cayley transform~\cite{Hinrichsen2005},
        $$q(\mu) \triangleq (\mu - 1)^2 p\left(\frac{\mu+1}{\mu-1}\right) = (1-t+d)\mu^2 + 2(1-d)\mu + (1+t+d),$$ has both roots inside the (closed) left half-plane. By \Cref{lemma:routh_hurwitz}, this holds if and only if the coefficients $1-t+d, 2(1-d), 1+t+d$ are all non-negative or all non-positive. These 3 coefficients, however, cannot all be non-positive at the same time: indeed, if the three relations $1-t+d \le 0, 1-d \le 0, 1+t+d \le 0$ hold, then one gets $$2 \le 1+d \le t \le -1-d \le -2,$$ which is impossible. Hence, Schur stability is equivalent to the three coefficients being non-negative, which are the conditions in \Cref{eq:schur_stable_subsets}.
    \end{proof}
\end{lemma}

Now, to obtain an explicit formulation of the candidate set $\mathcal{\hat{S}}$, some additional expressions are necessary. Let $\mathcal{M}(\sigma_1, \sigma_2)$ be the set of critical points $(\tau_1, \tau_2) \in \mathbb{R}^2$ of the function $$(\tau_1 - \sigma_1)^2 + (\tau_2 - \sigma_2)^2$$ subject to $\tau_1 \tau_2 = 1$. The solutions to this problem can be computed exactly for any given pair $(\sigma_1, \sigma_2)$, since solving
\begin{equation}
    0 = \frac{d}{dt}\left( \left( \sigma_1 - t \right)^2 + \left( \sigma_2 - \frac{1}{t} \right)^2 \right) =  \overbrace{\frac{2}{t^3}}^{\text{Never 0 for } t \in (-\infty,\infty)} \underbrace{\left( t^4 - \sigma_1 t^3 + \sigma_2 t - 1 \right)}_{\text{Target Polynomial}}
\end{equation}
amounts to computing the roots of a 4\textsuperscript{th}-degree polynomial. Thus, $\mathcal{M}$ has at most four elements for any choice of $(\sigma_1, \sigma_2)$. With this definition, \Cref{lemma:nearest_stable_real} can be stated.

\begin{lemma} \label[lemma]{lemma:nearest_stable_real}
    Let $\bm{A} \in \mathbb{R}^{2 \times 2}$, and $\bm{G} \in SO(2) \triangleq \left\{ \bm{X} \in O(2) : \mathrm{det}{\left(\bm{X}\right)} = 1 \right\}$ be a matrix such that $\bm{\check{A}} = \bm{G}^\intercal \bm{A} \bm{G}$ satisfies $\check{A}_{1,1} = \check{A}_{2,2}$ ($G$ exists by \Cref{lemma:rotation}). Suppose, moreover, that $\bm{A}$ has \gls{svd} $\bm{A} = \bm{U_0} \begin{bsmallmatrix} \sigma_{0,1} & \\ & \sigma_{0,2} \end{bsmallmatrix} \bm{V_0}^\intercal$ and that $\bm{A} \mp \bm{I}$ have \glspl{svd} $\bm{A} \mp \bm{I} = \bm{U_\pm} \begin{bsmallmatrix} \sigma_{\pm,1} & \\ & \sigma_{\pm,2} \end{bsmallmatrix} \bm{V_{\pm}}^\intercal$. Then, the set
    \begin{equation}
        \mathcal{\hat{S}}(\bm{A}) \triangleq \left\{ \bm{A} , \bm{\hat{A}}_+, \bm{\hat{A}}_- \right\} \cup \mathcal{\hat{A}}_0 \cup \mathcal{\hat{A}}_+ \cup \mathcal{\hat{A}}_- \cup \mathcal{\hat{A}}_* ,
    \end{equation}
    where
    \begin{align*}
        \bm{\hat{A}}_\pm      & = \pm \bm{I} + \bm{U}_\pm
        \begin{bsmallmatrix}
            \sigma_{\pm, 1} & \\ & 0
        \end{bsmallmatrix}
        \bm{V}_\pm ,
        \\
        \mathcal{\hat{A}}_0   & = \left\{
        \bm{U}_0
        \begin{bsmallmatrix}
            \tau_1 & \\ & \tau_2
        \end{bsmallmatrix}
        \bm{V}_0^\intercal : (\tau_1, \tau_2) \in \mathcal{M}(\sigma_{0,1}, \sigma_{0,2})
        \right\} ,
        \\
        \mathcal{\hat{A}}_\pm & = \left\{
        \bm{G}
        \begin{bsmallmatrix}
            \pm 1 & \check{A}_{1,2} \\ 0 & \pm 1
        \end{bsmallmatrix}
        \bm{G}^\intercal ,
        \bm{G}
        \begin{bsmallmatrix}
            \pm 1 & 0 \\ \check{A}_{2,1} & \pm 1
        \end{bsmallmatrix}
        \bm{G}^\intercal
        \right\},
    \end{align*}
    and
    \begin{equation*}
        \mathcal{\hat{A}}_* = \left\{
        \bm{G}
        \begin{bsmallmatrix}
            0 & \tau_1 \\ \tau_2 & 0
        \end{bsmallmatrix}
        \bm{G}^\intercal
        : (\tau_1 , \tau_2) \in \mathcal{M}(\check{A}_{1,2},\check{A}_{2,1})
        \right\} ,
    \end{equation*}
    contains the Schur-stable matrix nearest to $\bm{A}$.

    \begin{proof}
        Let $\bm{A}$ be Schur-unstable (otherwise the result is trivial), and let $\bm{\hat{A}}$ be a Schur-stable matrix nearest to $\bm{A}$. Any such matrix must belong to the boundary of $\mathcal{S}_- \cap \mathcal{S}_0 \cap \mathcal{S}_+$. Observe that $\partial \mathcal{S}_0$ is the set of matrices with determinant 1, whereas $\partial \mathcal{S}_\pm$ is the set of matrices with an eigenvalue $\pm 1$. In particular, the set $\partial \mathcal{S}_- \cap \partial \mathcal{S}_0 \cap \partial \mathcal{S}_+$ where all three constraints are active is empty, since a $2 \times 2$ matrix with both $1$ and $-1$ as eigenvalues cannot have determinant $1$. Hence, one of the following mutually exclusive cases must hold:
        \begin{enumerate}
            \item $\bm{\hat{A}} \in \partial \mathcal{S}_0$, $\bm{\hat{A}} \notin \partial \mathcal{S}_+ \cup \partial \mathcal{S}_-$;
            \item $\bm{\hat{A}} \in \partial \mathcal{S}_+$, $\bm{\hat{A}} \notin \partial \mathcal{S}_0 \cup \partial \mathcal{S}_-$;
            \item $\bm{\hat{A}} \in \partial \mathcal{S}_-$, $\bm{\hat{A}} \notin \partial \mathcal{S}_0 \cup \partial \mathcal{S}_+$;
            \item $\bm{\hat{A}} \in \partial \mathcal{S}_0 \cap \partial \mathcal{S}_+$, $\bm{\hat{A}} \notin \partial \mathcal{S}_-$;
            \item $\bm{\hat{A}} \in \partial \mathcal{S}_0 \cap \partial \mathcal{S}_-$, $\bm{\hat{A}} \notin \partial \mathcal{S}_+$;
            \item $\bm{\hat{A}} \in \partial \mathcal{S}_+ \cap \partial \mathcal{S}_-$, $\bm{\hat{A}} \notin \partial \mathcal{S}_0$.
        \end{enumerate}

        The six cases are treated separately:
        \begin{enumerate}
            \item $\bm{\hat{A}} \in \partial \mathcal{S}_0$, $\bm{\hat{A}} \notin \partial \mathcal{S}_+ \cup \partial \mathcal{S}_-$. Then, $\bm{\hat{A}}$ is a local minimizer of $\| \bm{A} - \bm{X} \|_F$ in the doubly rotation-invariant set $\partial \mathcal{S}_0$, and $\bm{\Xi} = \bm{U_0}^\intercal \bm{\hat{A}} \bm{V_0}$ is a local minimizer of $\left\| \begin{bsmallmatrix} \sigma_{0,1} & \\ & \sigma_{0,2}  \end{bsmallmatrix} - \bm{X} \right\|_F$ in $\partial \mathcal{S}_0$.

                  If $\sigma_{0,1} \neq \sigma_{0,2}$, then $\bm{\Xi}$ is diagonal by \Cref{lemma:doubly_rotation_invariant}; and for $\bm{\Xi} = \begin{bsmallmatrix} \tau_1 & \\ & \tau_2 \end{bsmallmatrix}$ (with $\tau_1 \tau_2 = \mathrm{det}(\bm{\Xi}) = 1$), it is necessary to have $(\tau_1 , \tau_2) \in \mathcal{M}(\sigma_{0,1}, \sigma_{0,2})$. Hence, $\bm{\hat{A}} \in \mathcal{\hat{A}}_0$.

                  Otherwise, if $\sigma_{0,1} = \sigma_{0,2}$, then $\bm{\Xi} = \bm{U}(\alpha) \begin{bsmallmatrix} \tau_1 & \\ & \tau_2 \end{bsmallmatrix} \bm{U}(\alpha)^\intercal$ and $$\bm{\hat{A}} = \bm{\hat{A}_\alpha} = \bm{U_0} \bm{U}(\alpha) \begin{bsmallmatrix} \tau_1 & \\ & \tau_2 \end{bsmallmatrix} \bm{U}(\alpha)^\intercal \bm{V_0}^\intercal$$ for some $\alpha \in [0,\pi/2)$. If $\bm{\hat{A}_0}$ is stable, then $\bm{\hat{A}_0} \in \mathcal{\hat{A}}_0$ is another Schur-stable matrix nearest to $\bm{A}$, and the proof is concluded.

                  If $\bm{\hat{A}_0}$ is not stable, then there is a minimum $\alpha$ such that $\bm{\hat{A}_\alpha}$ is stable, and this is another Schur-stable matrix nearest to $\bm{A}$ in which one more constraint is active; hence, it falls in one of cases 4-6.
            \item $\bm{\hat{A}} \in \partial \mathcal{S}_+$, $\bm{\hat{A}} \notin \partial \mathcal{S}_0 \cup \partial \mathcal{S}_-$. Then, $$\min_{\bm{X} \in \partial \mathcal{S}_+} \|\bm{A} - \bm{X} \|_F = \min_{\bm{X} \in \partial \mathcal{S}_+} \| \bm{U_+}^\intercal (\bm{A} - \bm{I}) \bm{V_+} - \bm{U_+}^\intercal (\bm{X} - \bm{I}) \bm{V_+} \|_F.$$ Hence, $\bm{\hat{A}}$ is a local minimizer of $\|\bm{A} - \bm{X} \|_F$ if and only if $\bm{\Xi} = \bm{U_+}^\intercal (\bm{X} - \bm{I}) \bm{V_+}$ in $\partial \mathcal{S}_d \triangleq \left\{ \bm{X} \in \mathbb{R}^{2 \times 2} : \mathrm{det}(\bm{X}) = 0 \right\}$: indeed, $\bm{X}$ has an eigenvalue 1 if and only if $\bm{\Xi}$ has determinant 0. Thus, either $\bm{\Xi} = \begin{bsmallmatrix} \sigma_{+,1} & \\ & 0 \end{bsmallmatrix}$ and hence $\bm{\hat{A}} = \bm{\hat{A}_+}$, or a minimizer with the same objective value in one of cases 4-6 can be found.
            \item $\bm{\hat{A}} \in \partial \mathcal{S}_-$, $\bm{\hat{A}} \notin \partial \mathcal{S}_0 \cup \partial \mathcal{S}_+$. Proof as in case 2, swapping all plus and minus signs.
            \item $\bm{\hat{A}} \in \partial \mathcal{S}_0 \cap \partial \mathcal{S}_+$, $\bm{\hat{A}} \notin \partial \mathcal{S}_-$. Note that $\partial \mathcal{S}_0 \cap \partial \mathcal{S}_+$ is the rotation-invariant set of matrices with double eigenvalue $+1$ (since they must have eigenvalue $+1$ and determinant $1$). $\bm{\hat{A}}$ is a minimizer if and only if $\bm{\Xi} = \bm{G}^\intercal \bm{\hat{A}} \bm{G}$ is a minimizer of $\| \bm{\check{A}} - \bm{X} \|$ in $\partial \mathcal{S}_0 \cap \partial \mathcal{S}_+$; so, by \Cref{lemma:rotation_invariant}, $\frac{1}{2} tr(\bm{\Xi}) = \Xi_{1,1} = \Xi_{2,2} = \frac{1}{2} tr(\bm{\hat{A}}) = +1$. Since $\mathrm{det}(\bm{\Xi}) = 1$, at least one amongst $\Xi_{1,2}$ and $\Xi_{2,1}$ is zero, and to minimize the projection distance, the other must be equal to the corresponding element of $\bm{\check{A}}$. Thus, $\bm{\hat{A}} \in \mathcal{\hat{A}}_+$.
            \item $\bm{\hat{A}} \in \partial \mathcal{S}_0 \cap \partial \mathcal{S}_-$, $\bm{\hat{A}} \notin \partial \mathcal{S}_+$. Proof as in case 4, swapping all plus and minus signs.
            \item $\bm{\hat{A}} \in \partial \mathcal{S}_+ \cap \partial \mathcal{S}_-$, $\bm{\hat{A}} \notin \partial \mathcal{S}_0$. Note that $\partial \mathcal{S}_+ \cap \partial \mathcal{S}_-$ is the rotation-invariant set of matrices with eigenvalues $1$ and $-1$. Arguing as in case 4, $$\bm{G}^\intercal \bm{\hat{A}} \bm{G} = \bm{\Xi} = \begin{bsmallmatrix} 0 & \tau_1 \\ \tau_2 & 0 \end{bsmallmatrix} , \ \tau_1 \tau_2 = 1 ,$$ as $\bm{\Xi}$ must have $\Xi_{1,1} = \Xi_{2,2} = \frac{1}{2} tr(\bm{\Xi}) = 0$ and $\mathrm{det}(\bm{\Xi}) = 1$. Moreover, to minimize $\| \bm{\check{A}} - \bm{\Xi} \|_F$, $(\tau_1, \tau_2)$ must belong to $\mathcal{M}(\check{A}_{1,2}, \check{A}_{2,1})$. Thus, $\bm{\hat{A}} \in \mathcal{A}_*$.
        \end{enumerate}
    \end{proof}
\end{lemma}

\section{Projection Scheme - Full Benchmark} \label[appendix]{app:projection}

\begin{table}[tb!]
    \caption{Extended results for the truncated projection benchmark.}
    \label{tb:projection_extended}
    \centering
    \begin{tblr}{
        colspec={*{7}{c}},
        colsep={3pt},
        cells={valign=m, font=\footnotesize},
        cell{2-Z}{1,2,4-Z}={mode=math},
        row{1}={font={\footnotesize \bfseries}},
        hline{1,every[2]{2}{-1}}={},
        cell{every[8]{2}{-8}}{1}={r=8}{},
        cell{every[2]{2}{-2}}{2}={r=2}{},
            }
        Case & $\bm{n}$ & Method                    & NSFE                 & NSSR                 & MSVR                  & Time [s]
        \\
        1    & 10       & Bi-level Optimization     & \expnumber{4.75}{-1} & \expnumber{9.20}{-1} & \expnumber{6.81}{-4}  & \expnumber{9.51}{1}
        \\
             &          & Block-Diagonal Projection & \expnumber{9.02}{-1} & \expnumber{9.02}{-1} & \expnumber{1.28}{-14} & \expnumber{1.35}{-4}
        \\
             & 20       & Bi-level Optimization     & \expnumber{4.87}{-1} & \expnumber{9.46}{-1} & \expnumber{3.33}{-2}  & \expnumber{1.54}{2}
        \\
             &          & Block-Diagonal Projection & \expnumber{9.51}{-1} & \expnumber{9.56}{-1} & \expnumber{7.11}{-16} & \expnumber{8.17}{-5}
        \\
             & 50       & Bi-level Optimization     & \expnumber{4.95}{-1} & \expnumber{9.54}{-1} & \expnumber{6.21}{-1}  & \expnumber{5.45}{2}
        \\
             &          & Block-Diagonal Projection & \expnumber{9.80}{-1} & \expnumber{9.80}{-1} & \expnumber{1.14}{-5}  & \expnumber{6.72}{-5}
        \\
             & 100      & Bi-level Optimization     & \expnumber{4.97}{-1} & \expnumber{9.59}{-1} & \expnumber{2.22}{0}   & \expnumber{1.44}{3}
        \\
             &          & Block-Diagonal Projection & \expnumber{9.90}{-1} & \expnumber{9.90}{-1} & \expnumber{0.00}{0}   & \expnumber{7.09}{-5}
        \\
        2    & 10       & Bi-level Optimization     & \expnumber{8.28}{-2} & \expnumber{3.58}{-1} & \expnumber{4.87}{-7}  & \expnumber{5.34}{2}
        \\
             &          & Block-Diagonal Projection & \expnumber{4.61}{-1} & \expnumber{4.35}{-1} & \expnumber{3.64}{-13} & \expnumber{2.93}{-5}
        \\
             & 20       & Bi-level Optimization     & \expnumber{3.87}{-1} & \expnumber{5.69}{-1} & \expnumber{2.82}{-5}  & \expnumber{1.48}{3}
        \\
             &          & Block-Diagonal Projection & \expnumber{4.61}{-1} & \expnumber{4.35}{-1} & \expnumber{1.72}{-8}  & \expnumber{3.88}{-5}
        \\
             & 50       & Bi-level Optimization     & \expnumber{3.81}{-1} & \expnumber{5.73}{-1} & \expnumber{1.90}{-4}  & \expnumber{9.91}{3}
        \\
             &          & Block-Diagonal Projection & \expnumber{5.27}{-1} & \expnumber{4.74}{-1} & \expnumber{1.76}{-2}  & \expnumber{4.55}{-5}
        \\
             & 100      & Bi-level Optimization     & \expnumber{3.96}{-1} & \expnumber{6.09}{-1} & \expnumber{2.96}{-3}  & \expnumber{6.75}{4}
        \\
             &          & Block-Diagonal Projection & \expnumber{4.97}{-1} & \expnumber{4.63}{-1} & \expnumber{9.56}{-2}  & \expnumber{5.82}{-5}
        \\
        3    & 10       & Bi-level Optimization     & \expnumber{4.08}{-1} & \expnumber{5.57}{-1} & \expnumber{2.18}{-11} & \expnumber{1.02}{3}
        \\
             &          & Block-Diagonal Projection & \expnumber{6.67}{-1} & \expnumber{4.06}{-1} & \expnumber{2.43}{-6}  & \expnumber{2.43}{-6}
        \\
             & 20       & Bi-level Optimization     & \expnumber{4.55}{-1} & \expnumber{6.83}{-1} & \expnumber{9.10}{-9}  & \expnumber{1.24}{3}
        \\
             &          & Block-Diagonal Projection & \expnumber{3.42}{-1} & \expnumber{5.54}{-1} & \expnumber{1.36}{-6}  & \expnumber{9.49}{-5}
        \\
             & 50       & Bi-level Optimization     & \expnumber{4.42}{-1} & \expnumber{7.48}{-1} & \expnumber{1.36}{-3}  & \expnumber{8.20}{3}
        \\
             &          & Block-Diagonal Projection & \expnumber{3.24}{-1} & \expnumber{6.13}{-1} & \expnumber{5.66}{-2}  & \expnumber{4.76}{-5}
        \\
             & 100      & Bi-level Optimization     & \expnumber{4.85}{-1} & \expnumber{7.61}{-1} & \expnumber{2.65}{-2}  & \expnumber{7.05}{4}
        \\
             &          & Block-Diagonal Projection & \expnumber{3.73}{-1} & \expnumber{6.28}{-1} & \expnumber{4.17}{-1}  & \expnumber{7.03}{-5}
        \\
        4    & 10       & Bi-level Optimization     & \expnumber{1.58}{-1} & \expnumber{7.15}{-1} & \expnumber{3.17}{-7}  & \expnumber{1.02}{3}
        \\
             &          & Block-Diagonal Projection & \expnumber{4.80}{-1} & \expnumber{5.59}{-1} & \expnumber{5.12}{-14} & \expnumber{6.90}{-5}
        \\
             & 20       & Bi-level Optimization     & \expnumber{1.94}{-1} & \expnumber{7.92}{-1} & \expnumber{2.35}{-4}  & \expnumber{1.34}{3}
        \\
             &          & Block-Diagonal Projection & \expnumber{6.02}{-1} & \expnumber{7.01}{-1} & \expnumber{3.34}{-8}  & \expnumber{4.98}{-5}
        \\
             & 50       & Bi-level Optimization     & \expnumber{2.46}{-1} & \expnumber{7.98}{-1} & \expnumber{8.52}{-2}  & \expnumber{1.44}{4}
        \\
             &          & Block-Diagonal Projection & \expnumber{7.44}{-1} & \expnumber{8.14}{-1} & \expnumber{1.60}{-6}  & \expnumber{6.82}{-5}
        \\
             & 100      & Bi-level Optimization     & \expnumber{4.37}{-1} & \expnumber{8.82}{-1} & \expnumber{3.38}{-2}  & \expnumber{4.46}{4}
        \\
             &          & Block-Diagonal Projection & \expnumber{7.92}{-1} & \expnumber{8.63}{-1} & \expnumber{1.53}{-3}  & \expnumber{8.33}{-5}
    \end{tblr}
\end{table}

\begin{figure}[tb!]
    \centering
    \begin{subfigure}{\linewidth}
        \begin{subfigure}{0.24\linewidth}
            \includegraphics[width=\linewidth]{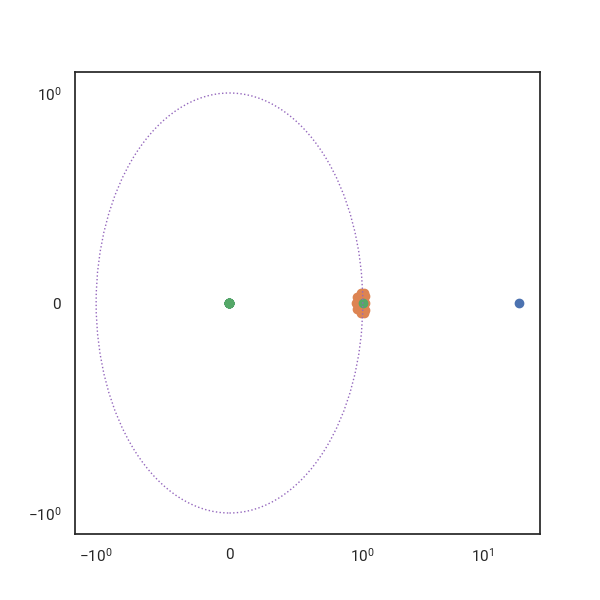}
            \caption{$n=10$}
        \end{subfigure}
        \hfill
        \begin{subfigure}{0.24\linewidth}
            \includegraphics[width=\linewidth]{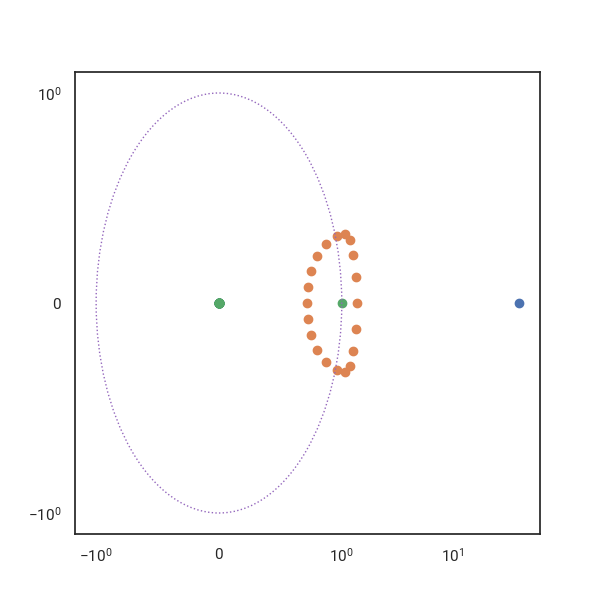}
            \caption{$n=20$}
        \end{subfigure}
        \hfill
        \begin{subfigure}{0.24\linewidth}
            \includegraphics[width=\linewidth]{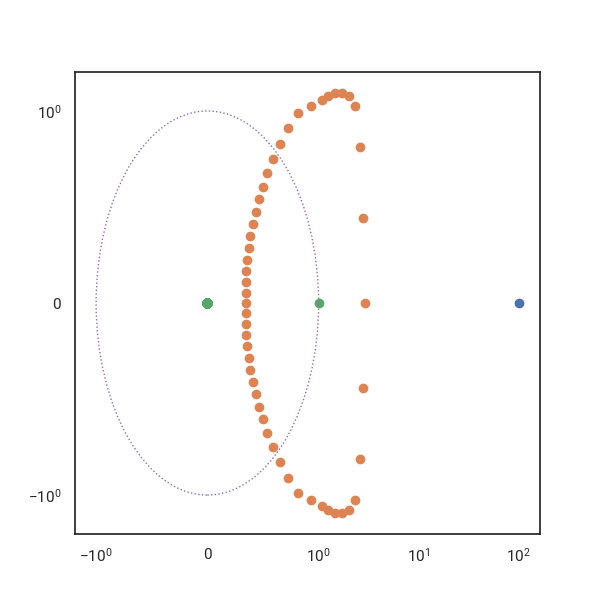}
            \caption{$n=50$}
        \end{subfigure}
        \hfill
        \begin{subfigure}{0.24\linewidth}
            \includegraphics[width=\linewidth]{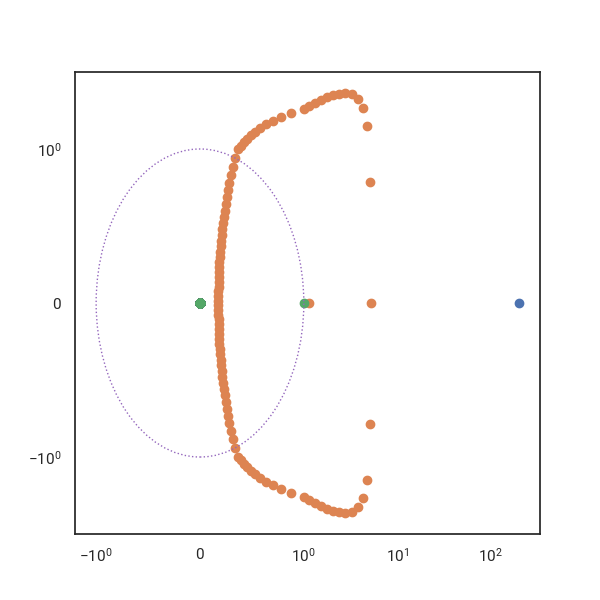}
            \caption{$n=100$}
        \end{subfigure}
        \caption{Case 1}
    \end{subfigure}
    \\
    \begin{subfigure}{\linewidth}
        \begin{subfigure}{0.24\linewidth}
            \includegraphics[width=\linewidth]{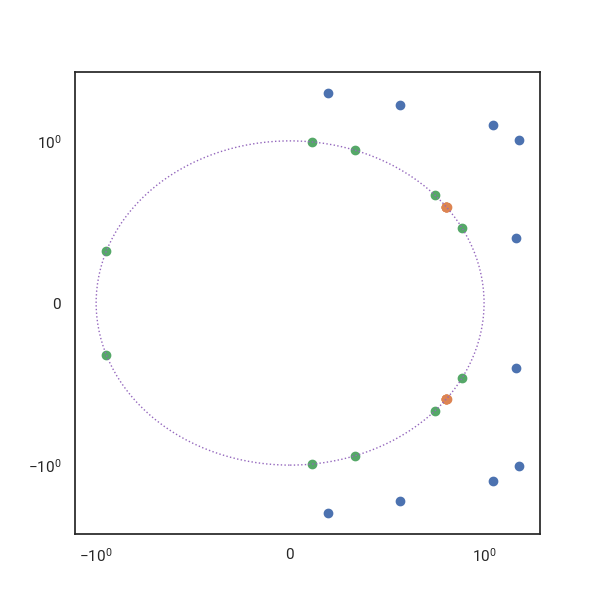}
            \caption{$n=10$}
        \end{subfigure}
        \hfill
        \begin{subfigure}{0.24\linewidth}
            \includegraphics[width=\linewidth]{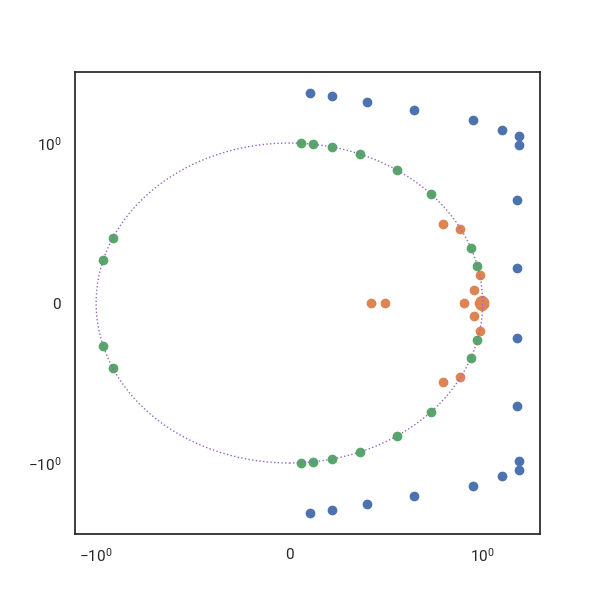}
            \caption{$n=20$}
        \end{subfigure}
        \hfill
        \begin{subfigure}{0.24\linewidth}
            \includegraphics[width=\linewidth]{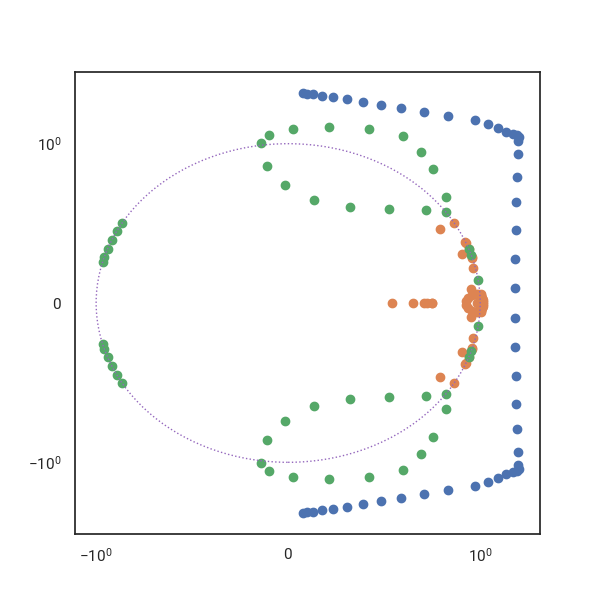}
            \caption{$n=50$}
        \end{subfigure}
        \hfill
        \begin{subfigure}{0.24\linewidth}
            \includegraphics[width=\linewidth]{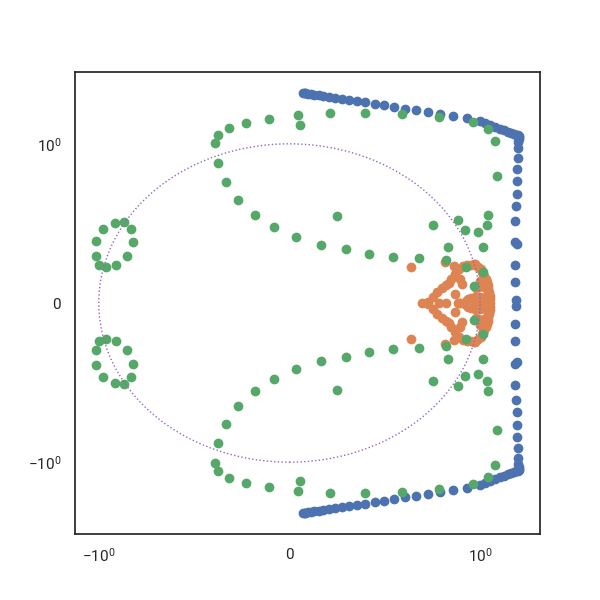}
            \caption{$n=100$}
        \end{subfigure}
        \caption{Case 2}
    \end{subfigure}
    \\
    \begin{subfigure}{\linewidth}
        \begin{subfigure}{0.24\linewidth}
            \includegraphics[width=\linewidth]{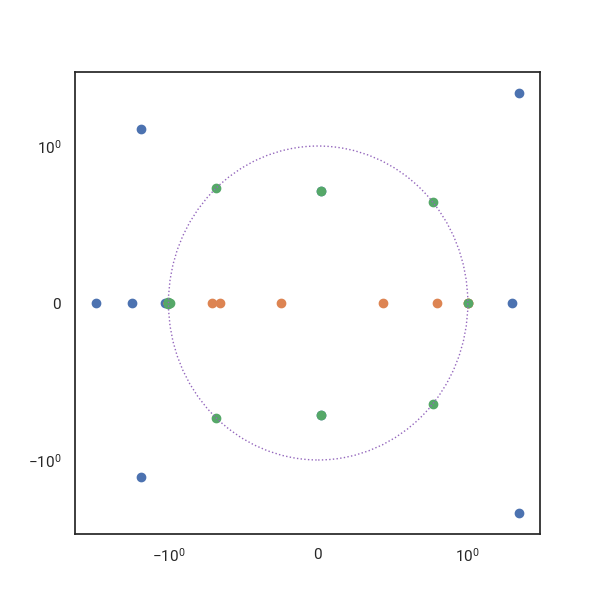}
            \caption{$n=10$}
        \end{subfigure}
        \hfill
        \begin{subfigure}{0.24\linewidth}
            \includegraphics[width=\linewidth]{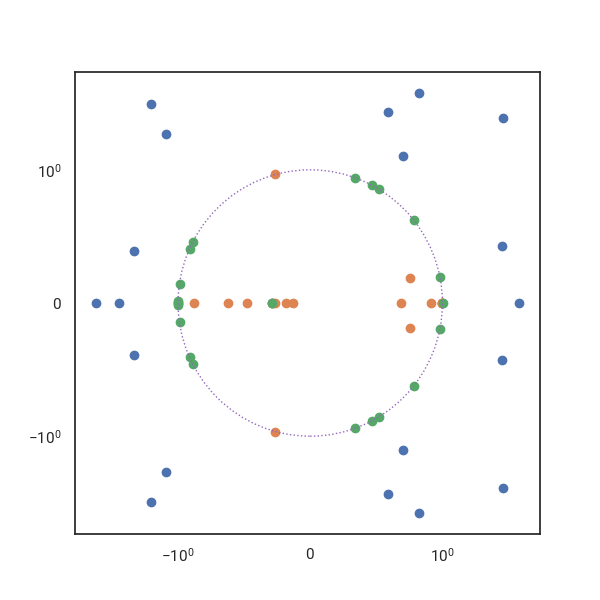}
            \caption{$n=20$}
        \end{subfigure}
        \hfill
        \begin{subfigure}{0.24\linewidth}
            \includegraphics[width=\linewidth]{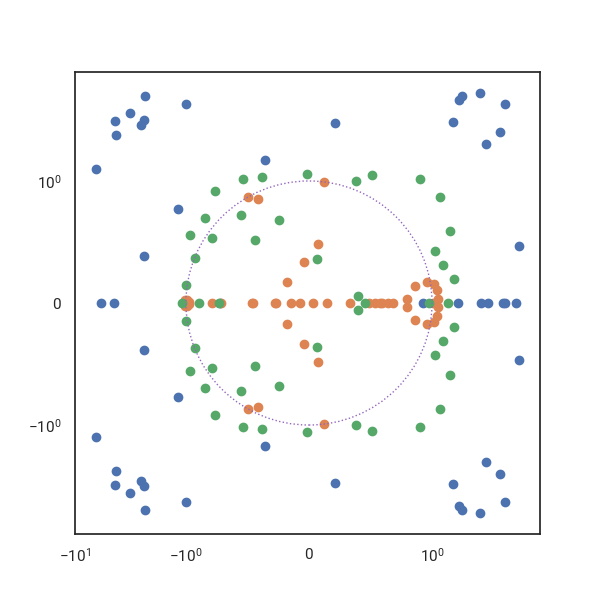}
            \caption{$n=50$}
        \end{subfigure}
        \hfill
        \begin{subfigure}{0.24\linewidth}
            \includegraphics[width=\linewidth]{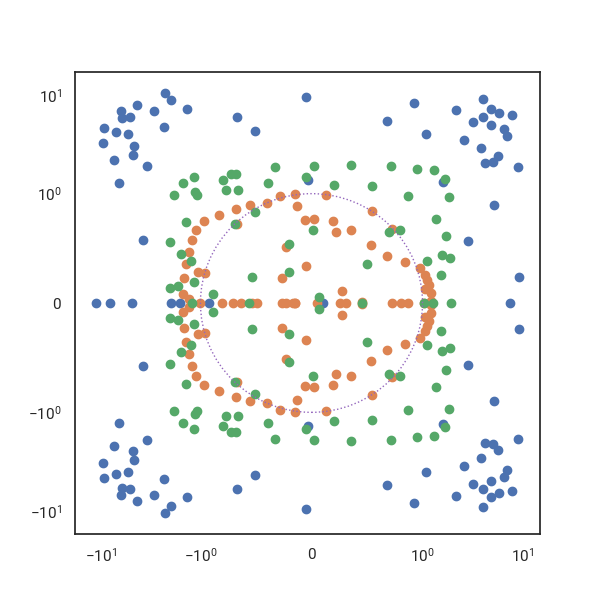}
            \caption{$n=100$}
        \end{subfigure}
        \caption{Case 3}
    \end{subfigure}
    \\
    \begin{subfigure}{\linewidth}
        \begin{subfigure}{0.24\linewidth}
            \includegraphics[width=\linewidth]{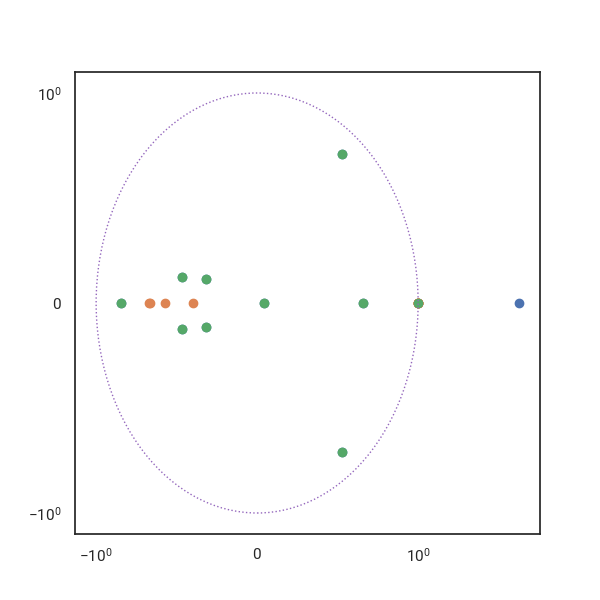}
            \caption{$n=10$}
        \end{subfigure}
        \hfill
        \begin{subfigure}{0.24\linewidth}
            \includegraphics[width=\linewidth]{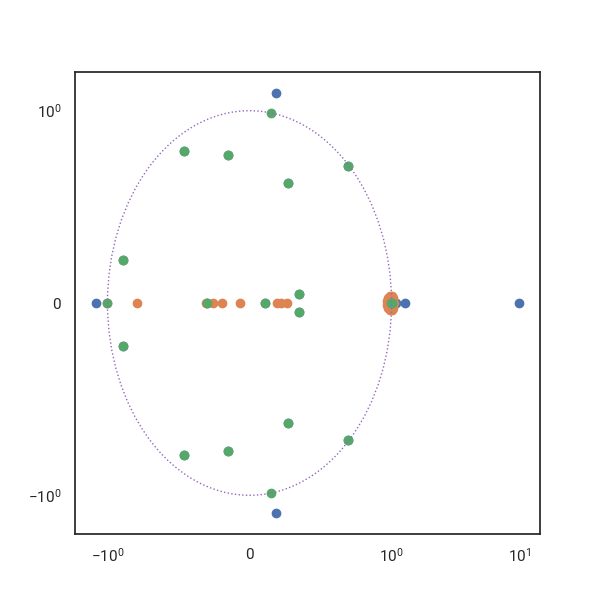}
            \caption{$n=20$}
        \end{subfigure}
        \hfill
        \begin{subfigure}{0.24\linewidth}
            \includegraphics[width=\linewidth]{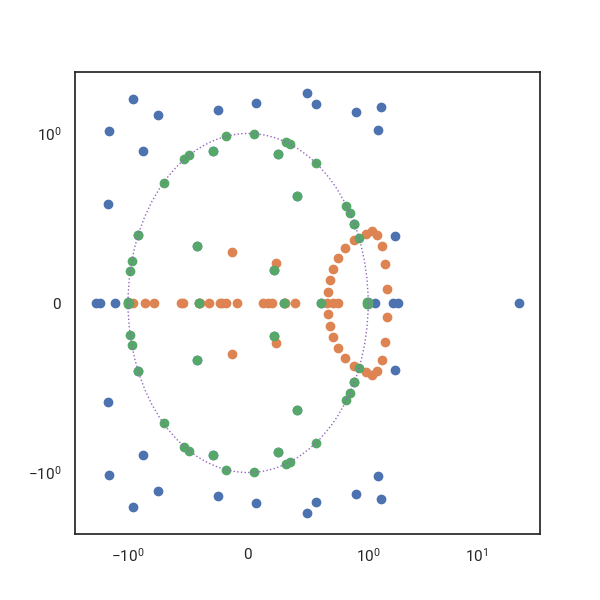}
            \caption{$n=50$}
        \end{subfigure}
        \hfill
        \begin{subfigure}{0.24\linewidth}
            \includegraphics[width=\linewidth]{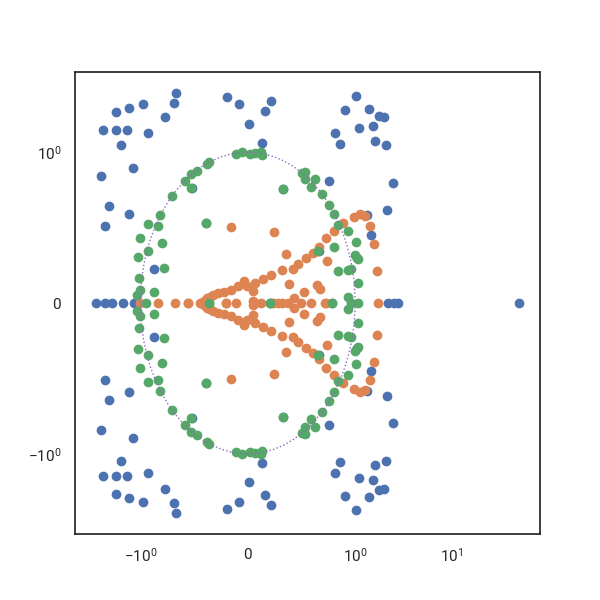}
            \caption{$n=100$}
        \end{subfigure}
        \caption{Case 4}
    \end{subfigure}
    \caption{Eigenvalues of the reference matrix (blue), full projection (orange), and truncated projection (green) for the cases evaluated in the scheme comparison (symlog axes with unitary linear threshold).}
    \label{fig:projection_eigenvalues}
\end{figure}

The experiments in this appendix, as well as those yielding the values reported in \Cref{tb:projection}, are executed on an isolated HPC instance with 4 Xeon Gold 6230 logical cores and 8 GiB of RAM. \Cref{tb:projection_extended} reports the same metrics as \Cref{tb:projection} for the 4 real-matrix cases originally considered in \citet[Section 7.4]{noferini2021nearest}, which were in turn taken from \citet{gillis2017computing,gillis2019approximating}:
\begin{enumerate}
    \item $A_{i,j} = 2 \quad \forall (i,j) \in \{1, \dots, n\}^2$
    \item $A_{i,j} =
              \begin{cases}
                  -1 \  & i-j = -1 ,                 \\
                  -1 \  & j-i \in \{ 0, 1, 2, 3 \} , \\
                  0 \   & \mathrm{otherwise}.
              \end{cases}$
    \item Random with independent entries in $\mathcal{N}(0, 1)$.
    \item Random with independent entries in $\mathcal{U}([0, 1])$.
\end{enumerate}
The resulting eigenvalues are shown in the complex plane in \Cref{fig:projection_eigenvalues}.

Following the setup from \citet{noferini2021nearest}, the maximum number of iterations and execution time were set to $1000$ and $1000$ seconds, respectively. However, the iteration limit only controls the maximum number of outer iterations, meaning that the inner iteration count can be much higher. Moreover, the time-limit check is only performed once per outer iteration, explaining why some of the execution times are much longer than $1000$ seconds.

The most noticeable pattern in \Cref{tb:projection_extended} is the difference in execution time, which is consistent across cases. Even without running the algorithm, one would expect the truncated method to be faster since, as pointed out in \Cref{subsec:schur_ss_trunc}, \Cref{alg:schur_proj} is equivalent to running a single inner iteration of the optimization problem. However, the actual gap turns out to be 5 orders of magnitude at best, which makes it plainly unfeasible as a projection method during training.

It is worth pointing out that the timing results in this appendix are vastly different from those in \citet{noferini2021nearest}. This is because their experiments were carried out in MATLAB using the \textit{Manopt}~\cite{manopt} toolbox, whereas the experiments in this paper used a re-implementation of their codebase in Python using the \textit{Pymanopt}~\cite{pymanopt} module along with its JAX interface. This was necessary to simulate the algorithm's performance as a component within a Python \gls{ml}-model training pipeline, but \textit{Pymanopt} currently lags behind the MATLAB package in terms of performance optimizations (e.g., no result caching), severely impacting its execution time.

\section{GPU Implementation} \label[appendix]{app:gpu}

\begin{algorithm}[tb!]
    \caption{Reduction to Hessenberg form via Householder reflections~\citep[Section 4.3]{arbenz2016numerical}.}
    \label{alg:hessenberg}

    \KwIn{Matrix $\bm{A} \in \mathbb{R}^{n \times n}$}
    \KwOut{Matrices $\bm{H} \in \mathbb{R}^{n \times n}$ Hessenberg and $\bm{U} \in O(n)$ s.t. $\bm{A} = \bm{U} \bm{H} \bm{U}^\intercal$}

    \Begin{
        Initialize $\bm{H}, \bm{U} \gets \bm{A}, \bm{I}$ \;
        \For{$k \gets 1$ to $n-2$}{
            Generate the Householder reflector $\bm{P_k} \triangleq \bm{I} - 2 \bm{u_k} \bm{u_k}^\intercal$ for the $k$-th column of $\bm{H}$ \;
            $\bm{H}_{k+1:n, k:n} \gets \bm{H}_{k+1:n, k:n} - 2 \bm{u_k} \left( \bm{u_k}^\intercal \bm{H}_{k+1:n, k:n} \right)$ \tcp*{Apply $\bm{P_k}$ from the left}
            $\bm{H}_{1:n, k+1:n} \gets \bm{H}_{1:n, k+1:n} - 2 \left( \bm{H}_{1:n, k+1:n} \bm{u_k} \right) \bm{u_k}^\intercal$ \tcp*{Apply $\bm{P_k}$ from the right}
        }

        \For{$k \gets n-2$ downto $1$}{
            $\bm{U}_{k+1:n,k+1:n} \gets \bm{U}_{k+1:n,k+1:n} - 2 \bm{u_k} \left( \bm{u_k}^\intercal \bm{U}_{k+1:n,k+1:n} \right)$ \tcp*{Update $\bm{U} \gets \bm{P_k} \bm{U}$}
        }
    }
\end{algorithm}

\begin{algorithm}[tb!]
    \caption{The Francis (explicit) double-step QR algorithm~\citep[Section 4.5]{arbenz2016numerical}.}
    \label{alg:francis}

    \KwIn{Matrix $\bm{H} \in \mathbb{R}^{n \times n}$ upper Hessenberg, $\epsilon \ge 0$ relative tolerance}
    \KwOut{Matrices $\bm{T} \in \mathbb{R}^{n \times n}$ quasi-triangular and $\bm{Z} \in O(n)$ s.t. $\bm{H} = \bm{Z} \bm{H} \bm{Z}^\intercal$}

    \Begin{
        Initialize $\bm{T} \gets \bm{H}, \bm{Z} \gets \bm{I}$ \;
        $p \gets n$ \tcp*{$p$ indicates the \textit{active} matrix size}
        \While{$p > 2$}{
            $\bm{G} \gets \bm{T}_{p-1:p,p-1:p}$ \;
            $s, t \gets \mathrm{tr}(\bm{g}), \mathrm{det}(\bm{G})$ \;
            $\bm{M} \gets \bm{T}^2 - s \bm{T} + t \bm{I}$ \;
            $\bm{Q}, \bm{R} \gets \mathrm{QR}(\bm{M})$ \tcp*{QR factorization of $\bm{M}$}
            $\bm{T} \gets \bm{Q}^\intercal \bm{T} \bm{Q}$ \;
            $\bm{Z} \gets \bm{Z} \bm{Q}$ \;

            \BlankLine
            \tcp{Check for deflation}
            \uIf(\tcp*[f]{$1 \times 1$ block found}){$|\bm{T}_{p,p-1}| < \epsilon (|\bm{T}_{p-1,p-1}| + |\bm{T}_{p,p}|)$}{
                $\bm{T}_{p,p-1}, p \gets 0, p-1$
            }
            \ElseIf(\tcp*[f]{$2 \times 2$ block found}){$|\bm{T}_{p-1,p-2}| < \epsilon (|\bm{T}_{p-2,p-2}| + |\bm{T}_{p-1,p-1}|)$}{
                $\bm{T}_{p-1,p-2} \gets 0$; $p \gets p-2$
            }
        }
    }
\end{algorithm}

As of version \texttt{0.10.0}, the JAX library only implements the Schur decomposition on CPU. Moreover, we set ourselves to implement the necessary components in a (mostly) backend-agnostic way, so that the model could be instantiated as a Keras component regardless of the chosen backend library (PyTorch, JAX, \dots). Thus, based on the lecture notes in \citet[Chapter 4]{arbenz2016numerical}, \Cref{alg:hessenberg,alg:francis} were implemented using the Keras OPS API. It is worth noting that the original notes recommend using the implicit Q theorem to reduce the computational overhead introduced by the QR factorization, but this optimization is targeted towards a CPU implementation and turned out to be detrimental to the algorithm's speed on the GPU during early implementations.

Some of the most notable \gls{jit}-compilable optimizations were the following:
\begin{itemize}
    \item The final reduction for $\bm{U}$ in \Cref{alg:hessenberg} was replaced by an associative scan multiplying all reflectors together.
    \item The while loop in \Cref{alg:francis} is replaced with a fixed-length \texttt{for} loop executing null iterations (i.e., without modifying the result) once $p \le 2$.
    \item Following LAPACKS's \texttt{SLAHQR} subroutine implementation, the convergence criterion in \Cref{alg:francis} is replaced by the one proposed in \citet{ahues1997new}; i.e., $$|\bm{X}_{2,1}| |\bm{X}_{1,2}| \le \epsilon |\bm{X}_{2,2}| |\bm{X}_{2,2} - \bm{X}_{1,1}|,$$ where $\bm{X} \in \mathbb{R}^{2 \times 2}$ takes the value of $\bm{T}_{p-1:p,p-1:p}$ and $\bm{T}_{p-2:p-1,p-2:p-1}$ for the first and second deflation checks, respectively.
    \item Upper off-triangularity is explicitly enforced at every iteration to account for numerical error introduced by the matrix products.
    \item The $p$ steps, representing the block-diagonal pattern, are stored and returned as an n-dimensional vector $\bm{b} \in \{0,1,2\}^n$, where a null value indicates that the corresponding element of $\bm{T}$'s diagonal is the second diagonal element of a $2 \times 2$ block.
\end{itemize}

\begin{figure}[tb!]
    \centering
    \begin{subfigure}{0.49\textwidth}
        \includegraphics[width=\textwidth]{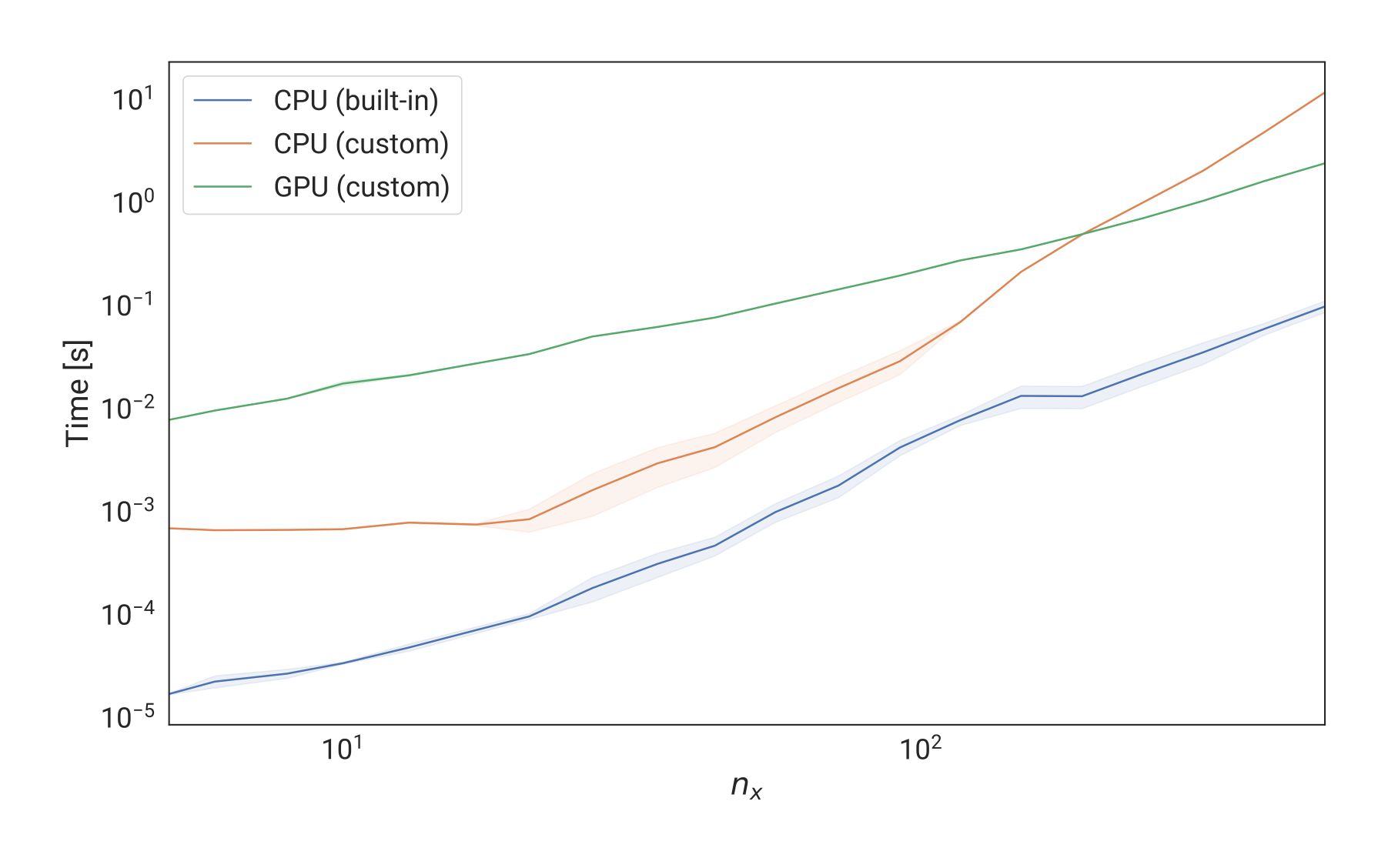}
        \caption{Execution time per implementation.}
        \label{subfig:gpu_timings}
    \end{subfigure}
    \hfill
    \begin{subfigure}{0.49\textwidth}
        \includegraphics[width=\textwidth]{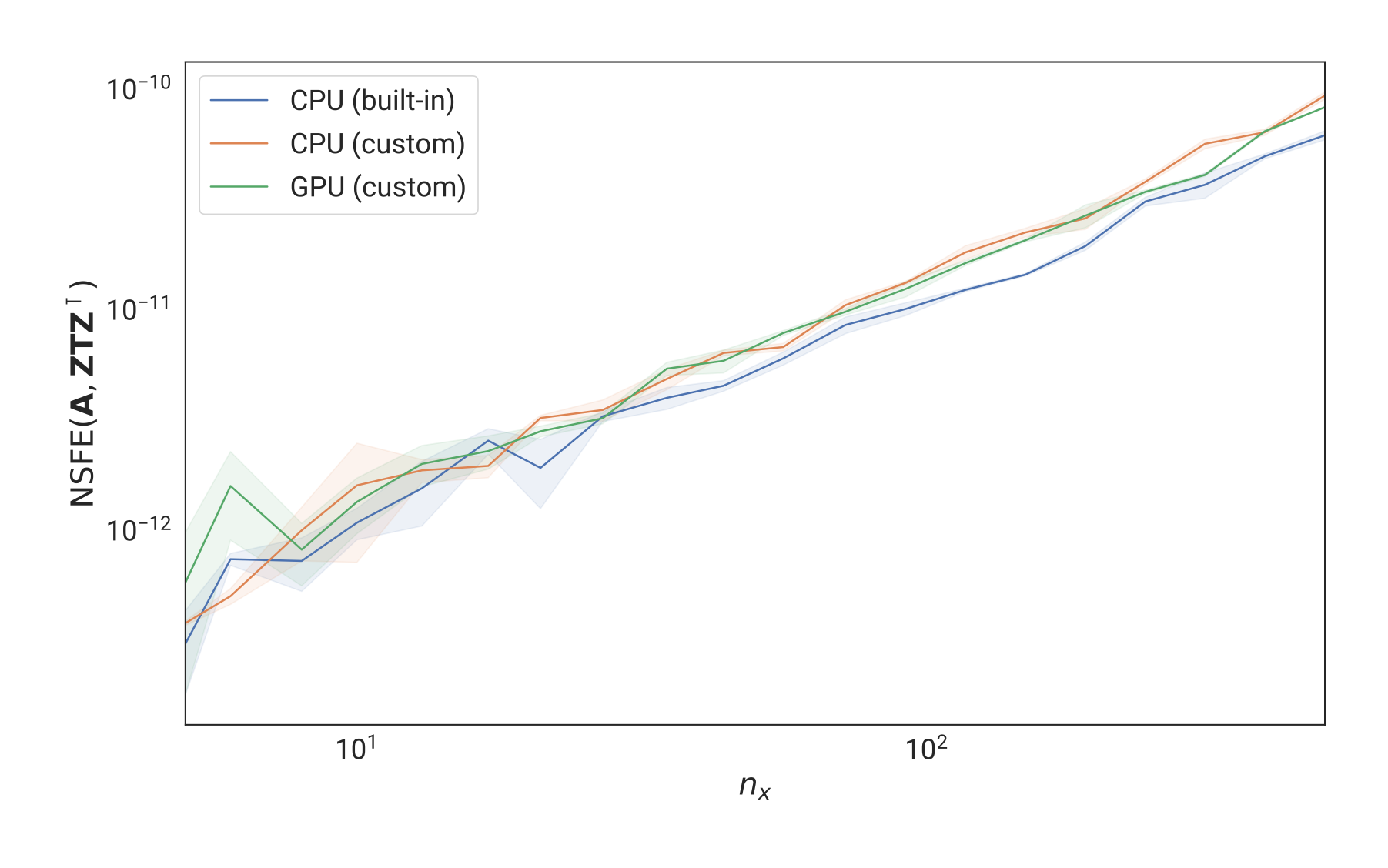}
        \caption{Reconstruction error per implementation.}
        \label{subfig:gpu_errors}
    \end{subfigure}
    \\
    \begin{subfigure}{0.49\textwidth}
        \includegraphics[width=\textwidth]{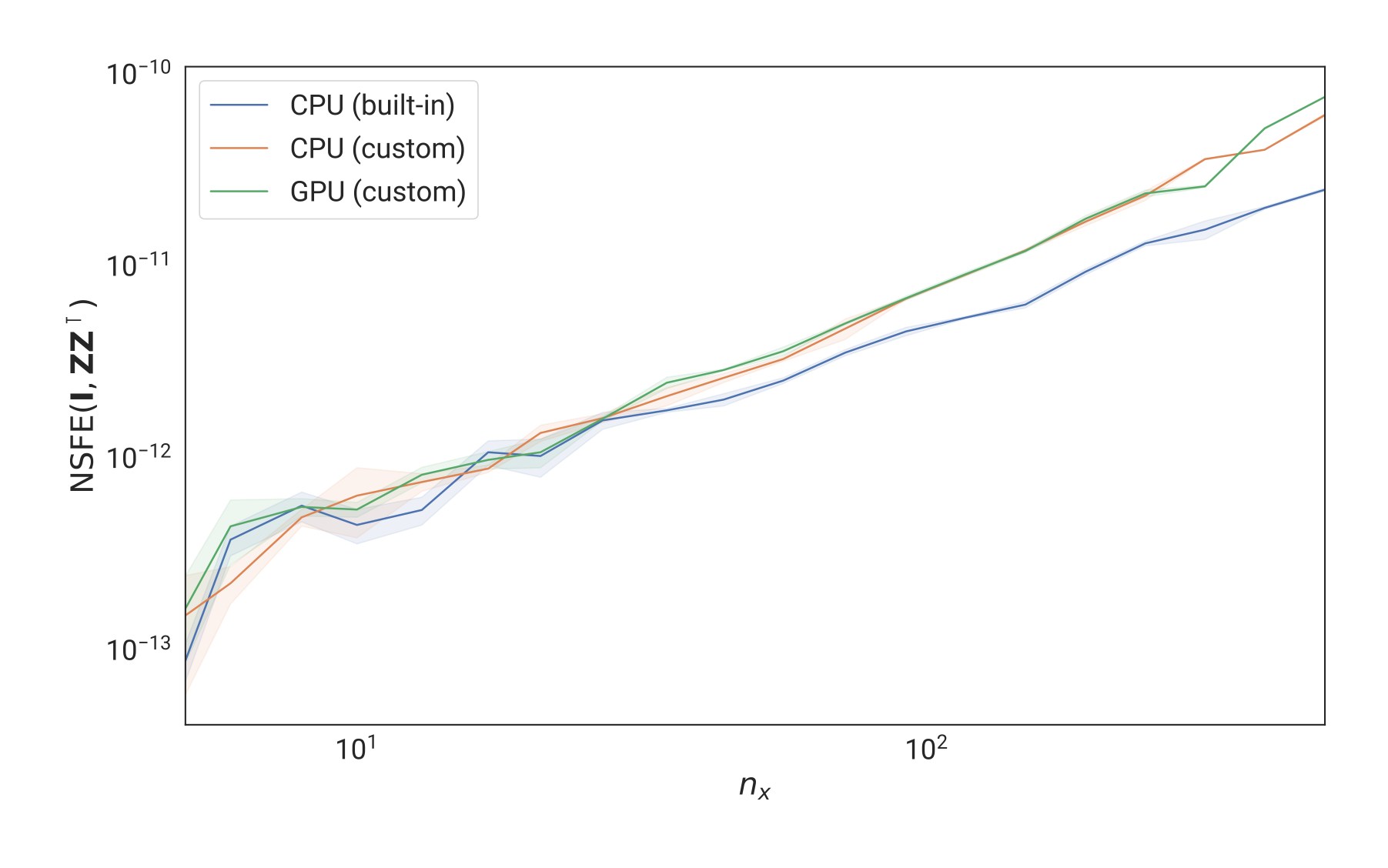}
        \caption{Orthogonality error per implementation.}
        \label{subfig:gpu_orthogonality}
    \end{subfigure}
    \hfill
    \begin{subfigure}{0.49\textwidth}
        \includegraphics[width=\textwidth]{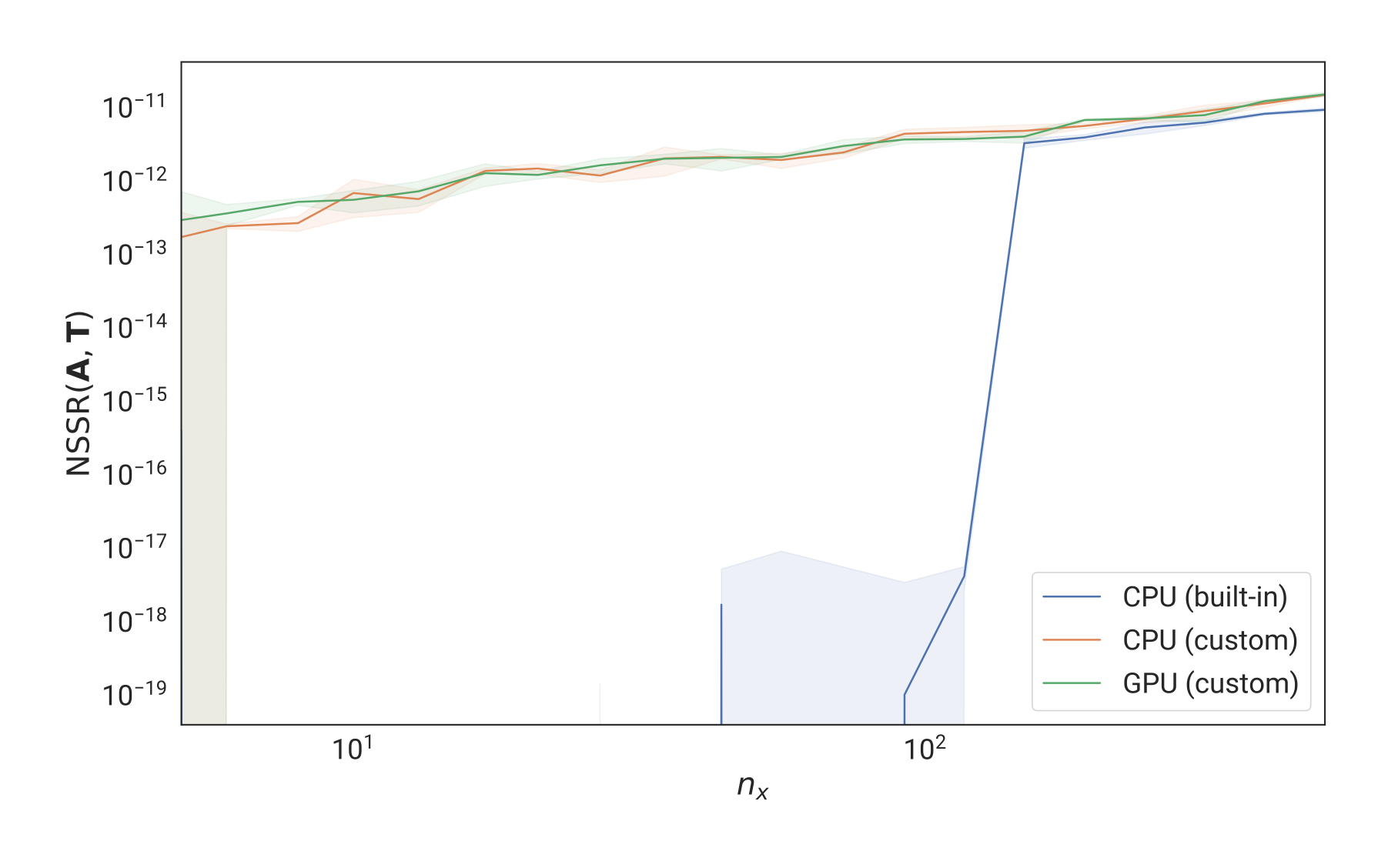}
        \caption{Spectral error per implementation.}
        \label{subfig:gpu_spectral}
    \end{subfigure}
    \caption{GPU benchmark figures.}
    \label{fig:gpu_benchmark}
\end{figure}

The performance of the custom Schur decomposition implementation, running both on CPU and GPU, compared to the built-in \texttt{schur} method as a function of the matrix size $n_x$ (illustrated in \Cref{fig:orthogonal_benchmark}), is evaluated using 4 metrics:
\begin{enumerate}
    \item The execution time of each method (\Cref{subfig:orthogonal_timings}).
    \item The reconstruction error (\Cref{subfig:orthogonal_distances}), defined as the \gls{nsfe} between the original matrix $\bm{A}$ and the product of its Schur factors $\bm{Z} \bm{T} \bm{Z}^\intercal$. This metric indicates whether or not each implementation yields a valid factorization, regardless of the factors' orthogonality and sparsity pattern.
    \item The orthogonality error (\Cref{subfig:orthogonal_errors}), defined as the \gls{nsfe} between the identity matrix and the symmetric matrix $\bm{Z} \bm{Z}^\intercal$. This metric indicates whether $\bm{Z}$ fulfils the orthogonality requirement.
    \item The spectral error (\Cref{subfig:orthogonal_errors}), defined as the \gls{nssr} between the original matrix $A$ and the quasi-triangular factor $\bm{T}$. By \Cref{theorem:real_schur}, under a correct factorization, this metric should only be bounded by numerical error.
\end{enumerate}
This experiment is executed on an isolated HPC instance with 4 Xeon Gold 6230 logical cores, 32 GiB of RAM, and a single Nvidia Volta V100 with 32 GiB of VRAM. Each metric is sampled 5 times per $n_x$ (20 logarithmically spaced values from 5 to 500), reporting the median as the expected value and half of the \gls{iqr} as its uncertainty. For the timings, each sample corresponds to the average time of 20 consecutive runs. Likewise, each error sample is drawn from a random matrix with normally distributed coefficients.

From this experiment, the following conclusions are drawn:
\begin{itemize}
    \item \Cref{subfig:gpu_errors,subfig:gpu_orthogonality} prove that the Schur decomposition is correctly implemented, with all approximation errors being on par with the built-in function. Although \Cref{subfig:gpu_spectral} would at first glance seem to indicate otherwise, the scale of the \gls{nssr} lies below the level of accuracy offered by 32-bit floating-point math. However, the authors recognize that more sophisticated numerical robustness measures must have been put in place for the built-in method to yield this level of spectral accuracy.
    \item Unsurprisingly, a Python rewrite of a function that is originally implemented as an XLA binding results in lower performance, as illustrated by \Cref{subfig:gpu_timings}. Moreover, as is the case for most GPU algorithm implementations, the advantages of parallelization only outweigh the drawbacks of the slower cores for a sufficiently large array size, which in this case happens at $n_x \gtrapprox 200$. Nevertheless, it is also worth highlighting that the execution-time slope of the GPU implementation is not as steep as that of both CPU implementations, meaning that there is potential for a well-optimized GPU rewrite.
\end{itemize}

Due to the aforementioned reasons, as well as the low batch count and state sizes assessed in \Cref{sec:exp}, experiments are carried out using the CPU and built-in Schur decomposition implementation.

\section{Nearest Orthogonal Matrix} \label[appendix]{app:orthogonal}

As mentioned near the end of \Cref{subsec:schur_ss_trunc}, it is possible to independently optimize the Schur decomposition factors of the state matrix to avoid having to compute the decomposition itself after every training batch. However, this introduces a new challenge: efficiently projecting $\bm{Z} \in \mathbb{R}^{n_x \times n_x}$ to its nearest orthogonal matrix $\bm{\hat{Z}}$; i.e.,

\begin{equation} \label{eq:orthogonal_procrustes}
    \bm{\hat{Z}} \triangleq \argmin_{X \in O(n_x)} \left\| \bm{Z} - \bm{X} \right\| .
\end{equation}

The constrained optimization problem in \Cref{eq:orthogonal_procrustes} is a special case of the Orthogonal Procrustes problem~\cite{gower2004procrustes}, a well-studied mathematical problem with closed-form solutions and derived iterative approximations. In this appendix, the following approaches are evaluated:
\begin{enumerate}
    \item The method from \citet{golub1996matrix}, which finds the nearest orthogonal matrix by multiplying together the unitary factors from the \acrlong{svd} $\bm{Z} = \bm{U} \bm{\Sigma} \bm{V}^\intercal$; i.e.
          \begin{equation}
              \bm{\hat{Z}} = \bm{U} \bm{V}^*.
          \end{equation}
    \item The method from \citet{horn1988closed}, which defines the closed form $$\bm{\hat{Z}} = \bm{Z} \left( \bm{Z}^\intercal \bm{Z} \right)^{-1/2},$$
          calculating the inverse square root of the product in parentheses through its eigenvalue decomposition $\left( \bm{Z}^\intercal \bm{Z} \right) = \bm{Q} \bm{\Lambda} \bm{Q}^{-1}$ as
          \begin{equation}
              \left( \bm{Z}^\intercal \bm{Z} \right)^{-1/2} = \bm{Q} \bm{\Lambda}^{-1/2} \bm{Q}^{-1}
              =
              \bm{Q} \begin{bsmallmatrix}
                  \frac{1}{\sqrt{\lambda_1}} & & \\ & \ddots & \\ && \frac{1}{\sqrt{\lambda_{n_x}}}
              \end{bsmallmatrix} \bm{Q}^{-1} .
          \end{equation}
    \item The same method as in \citet{horn1988closed}, but instead approximating the inverse square root using the stable iterative scheme from \citet[Section 3]{Sherif1991}; i.e., via the recursion
          \begin{equation} \label{eq:inv_sqrt_recursion}
              \bm{X}^{-1/2} = \lim_{r \rightarrow \infty} \bm{\Xi_r}
              \; s.t. \;
              \begin{cases}
                  \bm{\Xi_{r+1}} = \bm{\Xi_r} \left( \bm{I} + \bm{E}_r \right)   , & \bm{\Xi_0} = \bm{I} ,
                  \\
                  \bm{E_{r+1}} = \bm{E_r}^2 \left( 2 \bm{I} - \bm{S_r}^2 \right) , & \bm{E_0} = \left( \bm{I} - \bm{X} \right) \left( \bm{I} + \bm{X} \right)^{-1} .
              \end{cases}
          \end{equation}
          Since the methods are \gls{jit}-compiled during training, the stop condition for \Cref{eq:inv_sqrt_recursion} cannot be a function of $\bm{X}$'s coefficients; thus, its performance is evaluated by setting a fixed number of iterations equal to $n_x$.
\end{enumerate}

\begin{figure}[tb!]
    \centering
    \begin{subfigure}{0.75\textwidth}
        \includegraphics[width=\textwidth]{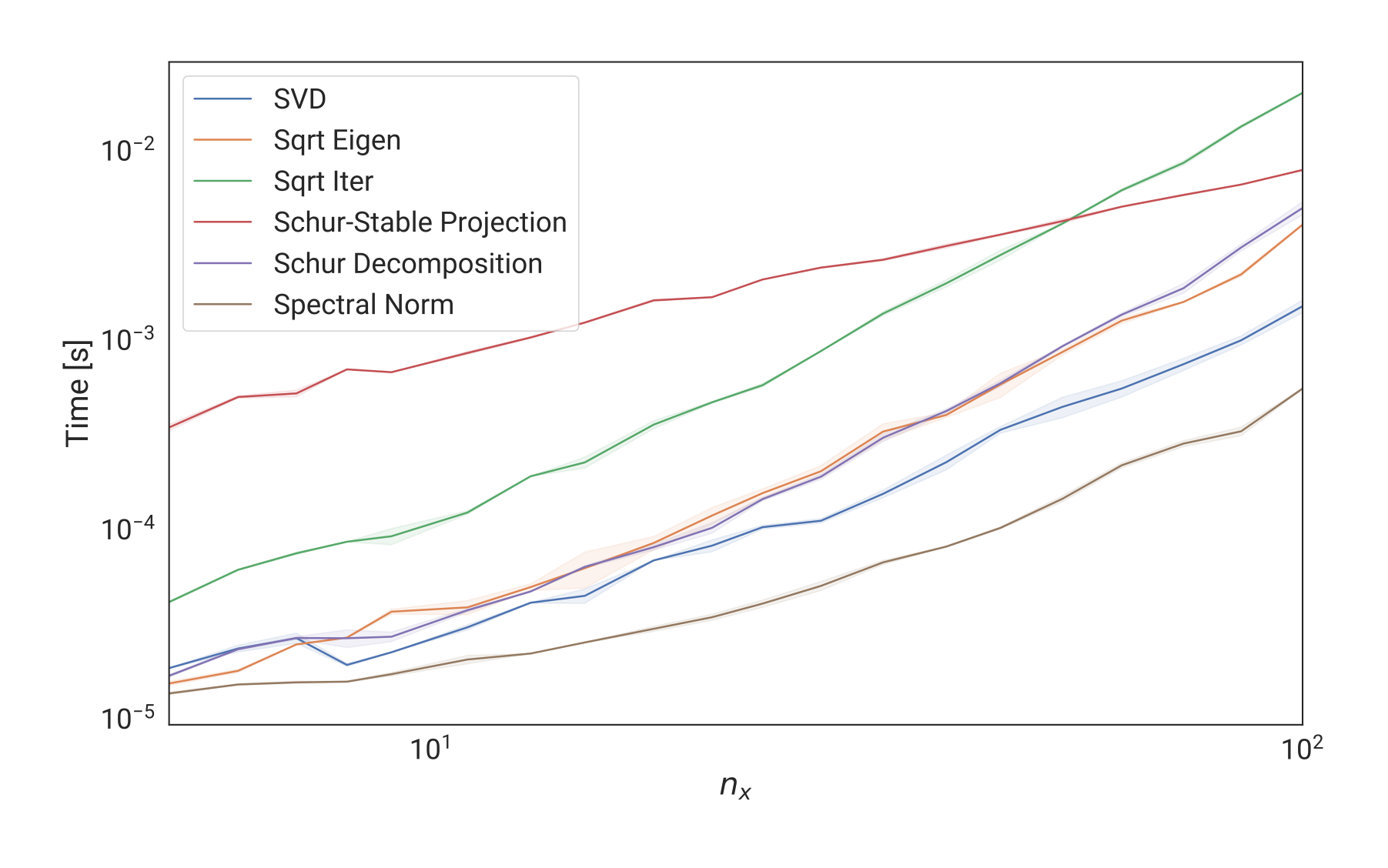}
        \caption{Execution time per method. The Schur-stable projection, Schur decomposition, and spectral norm execution times are included for reference.}
        \label{subfig:orthogonal_timings}
    \end{subfigure}
    \\
    \begin{subfigure}{0.75\textwidth}
        \includegraphics[width=\textwidth]{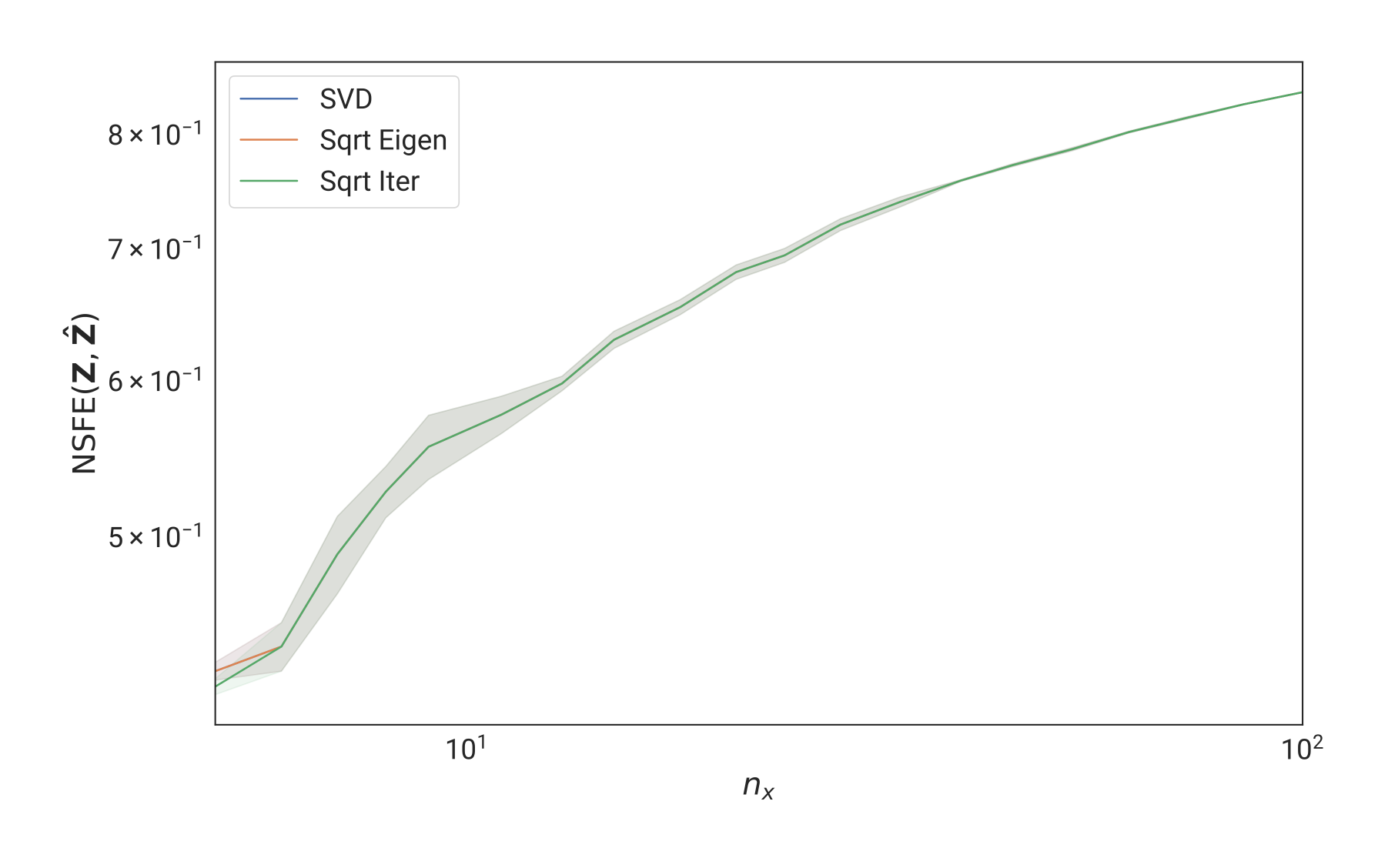}
        \caption{Projection error per method. The high error for the iterative methods at $n_x = 5$ can be attributed to insufficient iterations.}
        \label{subfig:orthogonal_distances}
    \end{subfigure}
    \\
    \begin{subfigure}{0.75\textwidth}
        \includegraphics[width=\textwidth]{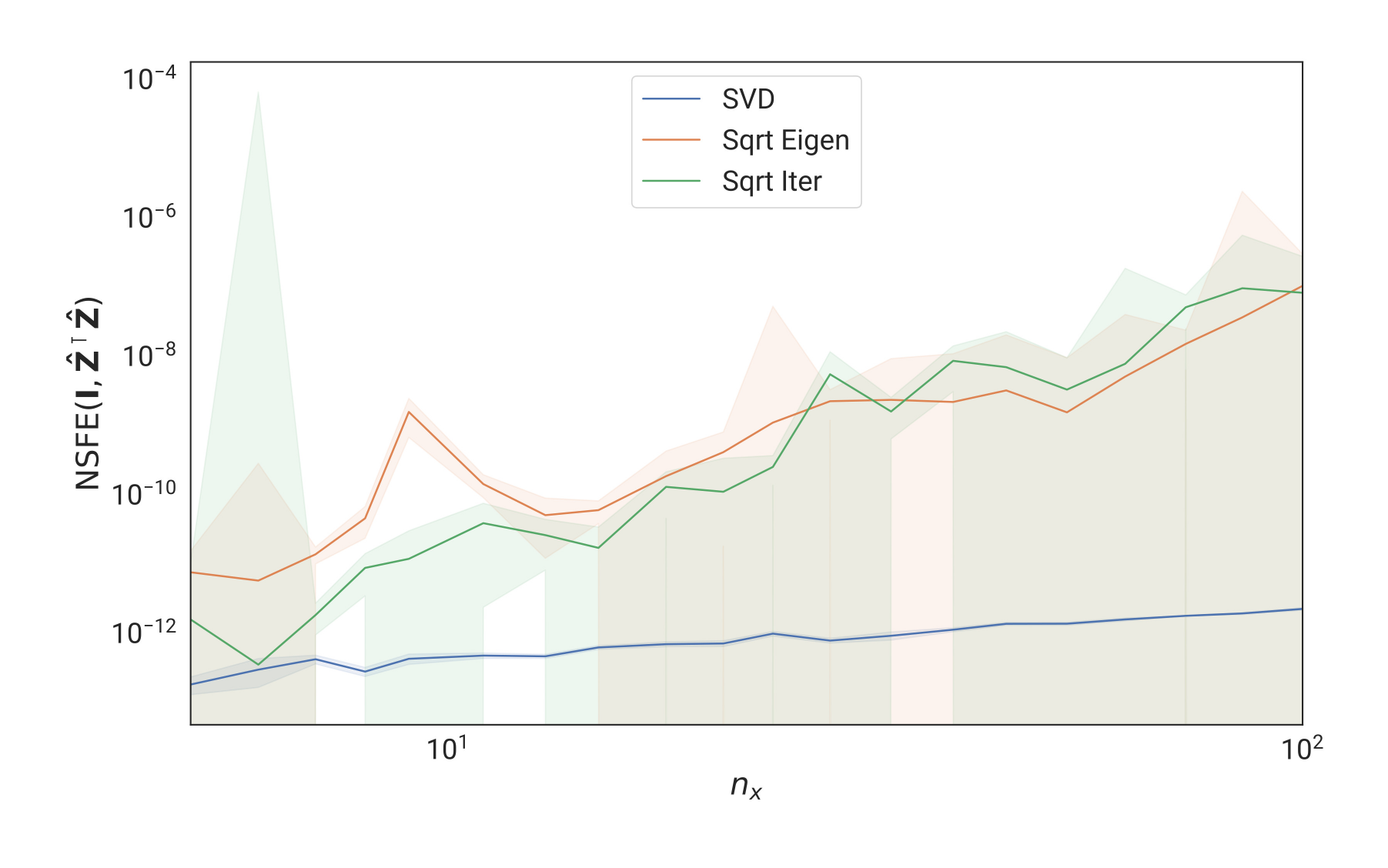}
        \caption{Orthogonality numerical error per method.}
        \label{subfig:orthogonal_errors}
    \end{subfigure}
    \caption{Orthogonal benchmark figures.}
    \label{fig:orthogonal_benchmark}
\end{figure}

The performance of the above methods, illustrated as a function of the matrix size $n_x$ in \Cref{fig:orthogonal_benchmark}, is evaluated using 3 metrics:
\begin{enumerate}
    \item The execution time of each method (\Cref{subfig:orthogonal_timings}). Execution times for the Schur-stable projection, the Schur decomposition, and the spectral norm are also included for scale and relative-cost analysis.
    \item The projection distance (\Cref{subfig:orthogonal_distances}), defined as the \gls{nsfe} between the original matrix $\bm{Z}$ and its projection $\bm{\hat{Z}}$. This metric should only be evaluated relative to other methods, as a null value should only be achievable if $\bm{Z}$ is already orthogonal.
    \item The orthogonality numerical error (\Cref{subfig:orthogonal_errors}), defined as the \gls{nsfe} between the identity matrix and $\bm{\hat{Z}}^\intercal \bm{\hat{Z}}$. Unlike the projection distance, if $\bm{\hat{Z}} \in O(n_x)$, then this metric should be approximately 0 (bounded by numerical error).
\end{enumerate}
This experiment is executed on an interactive HPC instance with 4 Xeon Gold 6230 logical cores and 8 GiB of RAM. Each metric is sampled 5 times per $n_x$ (20 logarithmically spaced values from 5 to 100), reporting the median as the expected value and half of the \gls{iqr} as its uncertainty. For the timings, each sample corresponds to the average time of 100 consecutive runs. Likewise, each error sample is drawn from a random matrix with normally distributed coefficients.

From this experiment, the following conclusions are drawn:
\begin{itemize}
    \item All orthogonal-projection methods in \Cref{subfig:orthogonal_timings} exhibit a similar slope, and can thus be considered to be of equivalent asymptotic complexity (i.e., $\mathcal{O}\left( n_x^3 \right)$). Nevertheless, it is clear that the \gls{svd} method consistently outperforms the other two, although the iterative approximation could be more efficient with a dynamic termination condition. More importantly, the \gls{svd} approach also proves to be faster than the Schur decomposition, supporting the claims regarding increased computational efficiency made in \Cref{subsec:schur_ss_trunc}.
    \item For the $n_x$ values considered in \Cref{sec:exp}, the execution times for the Schur-stable projection are an order of magnitude larger than those of the Schur decomposition and non-iterative nearest-orthogonal projection. That is not to say that their computational cost is negligible, as the asymptotic complexity of \Cref{alg:schur_proj} is $\mathcal{O}(n_x)$ (and can reach $\mathcal{O}(1)$ if the block projections are parallelized). Thus, its execution time is bound to be exceeded by all considered orthogonal-projection methods for a sufficiently large matrix size.
    \item All methods find very similar solutions in terms of projection distance, as illustrated by \Cref{subfig:orthogonal_distances}, but the solution found using the first method is more numerically robust, as demonstrated by \Cref{subfig:orthogonal_errors}. It is worth noting, however, that the expected orthogonality error for all methods is small enough to be acceptable for 32-bit floating-point arithmetic, and would have been a trade-off worth considering if the less precise methods were more computationally efficient. Nevertheless, as mentioned above, this is not the case.
\end{itemize}

For the aforementioned reasons, the method from \citet{golub1996matrix} is selected to constrain the orthogonal factor.

\section{Synthetic-Data Benchmark} \label[appendix]{app:synthetic}

\subsection{Target-System \& Model Parameters} \label[appendix]{subsec:synthetic_description}

\begin{table}[tb!]
    \caption{Experimental parameters for the cases evaluated in the synthetic-data benchmarks.}
    \label{tb:synthetic_setup}
    \centering
    \begin{tblr}{
        colspec={*{6}{c}},
        hspan=minimal,
        cells={valign=m, font=\small},
        cell{1}{1}={r=2}{},
        cell{1}{2}={c=5}{},
        cell{3-Z}{2-Z}={mode=math},
        row{1-2}={font={\small \bfseries}},
        hline{1-3,Z}={},
            }
        Parameter                      & Setup
        \\
                                       & Smaller           & Small             & Original          & Extended          & Large
        \\
        System Order $(n_x, n_u, n_y)$ & 5, 3, 3           & 5, 3, 3           & 5, 3, 3           & 5, 3, 3           & 10, 6, 6
        \\
        No. of Systems                 & 100               & 100               & 50                & 100               & 20
        \\
        Sequences per Partition        & 1                 & 1                 & 1                 & 1                 & 8
        \\
        Samples per Sequence           & 64                & 128               & 300               & 1024              & 512
        \\
        Output Noise Scale $(\sigma)$  & 0.01              & 0.01              & 0.25              & 0.01              & 0.01
        \\
        $|\lambda|$ upper bound        & 0.99              & 0.99              & 0.99              & 0.95              & 0.90
        \\
        Epoch Limit                    & \expnumber{5}{4}  & \expnumber{5}{4}  & \expnumber{5}{4}  & \expnumber{5}{4}  & \expnumber{1}{5}
        \\
        Patience                       & \expnumber{1}{4}  & \expnumber{1}{4}  & \infty            & \expnumber{1}{4}  & \expnumber{1}{4}
        \\
        Optimizer                      & \text{AdamW}      & \text{AdamW}      & \text{AdamW}      & \text{AdamW}      & \text{AdamW}
        \\
        Learning Rate                  & \expnumber{1}{-3} & \expnumber{1}{-3} & \expnumber{1}{-3} & \expnumber{1}{-3} & \expnumber{1}{-4}
    \end{tblr}
\end{table}

\Cref{tb:synthetic_setup} contains the experimental parameters used for the synthetic-data benchmarks. The following modifications with respect to the original experimental setup are highlighted:
\begin{itemize}
    \item The output white-noise scale was reduced from $0.25$ to $0.01$. This was due to the variability of the error metrics being too large, which was not an issue in the original paper, as the performance gap between methods was more pronounced.
    \item For the sake of numerical stability with 32-bit floating-point arithmetic, the eigenvalue-magnitude upper bound was set to $0.99$ for the \textit{original} and \textit{small(er)} settings. Furthermore, tighter bounds equal to $0.95$ and $0.9$ were necessary due to the increased sequence length and system size of the \textit{extended} and \textit{large} setups, respectively.
    \item The number of systems was doubled for all experiments of equal system size to further reduce the uncertainty via increased sample size. However, this was not necessary for the \textit{large} setup, as the variability of the results was already low enough with 20 random systems.
\end{itemize}

Other parameter changes follow a similar line of reasoning. Each experiment is carried out on an interactive HPC instance with 4 Xeon Gold 6230 logical cores and 6 GiB of RAM.

\subsection{Supplementary Results} \label[appendix]{subsec:synthetic_supplementary}

\begin{figure}[tb!]
    \centering
    \begin{subfigure}{0.49\textwidth}
        \includegraphics[width=\linewidth]{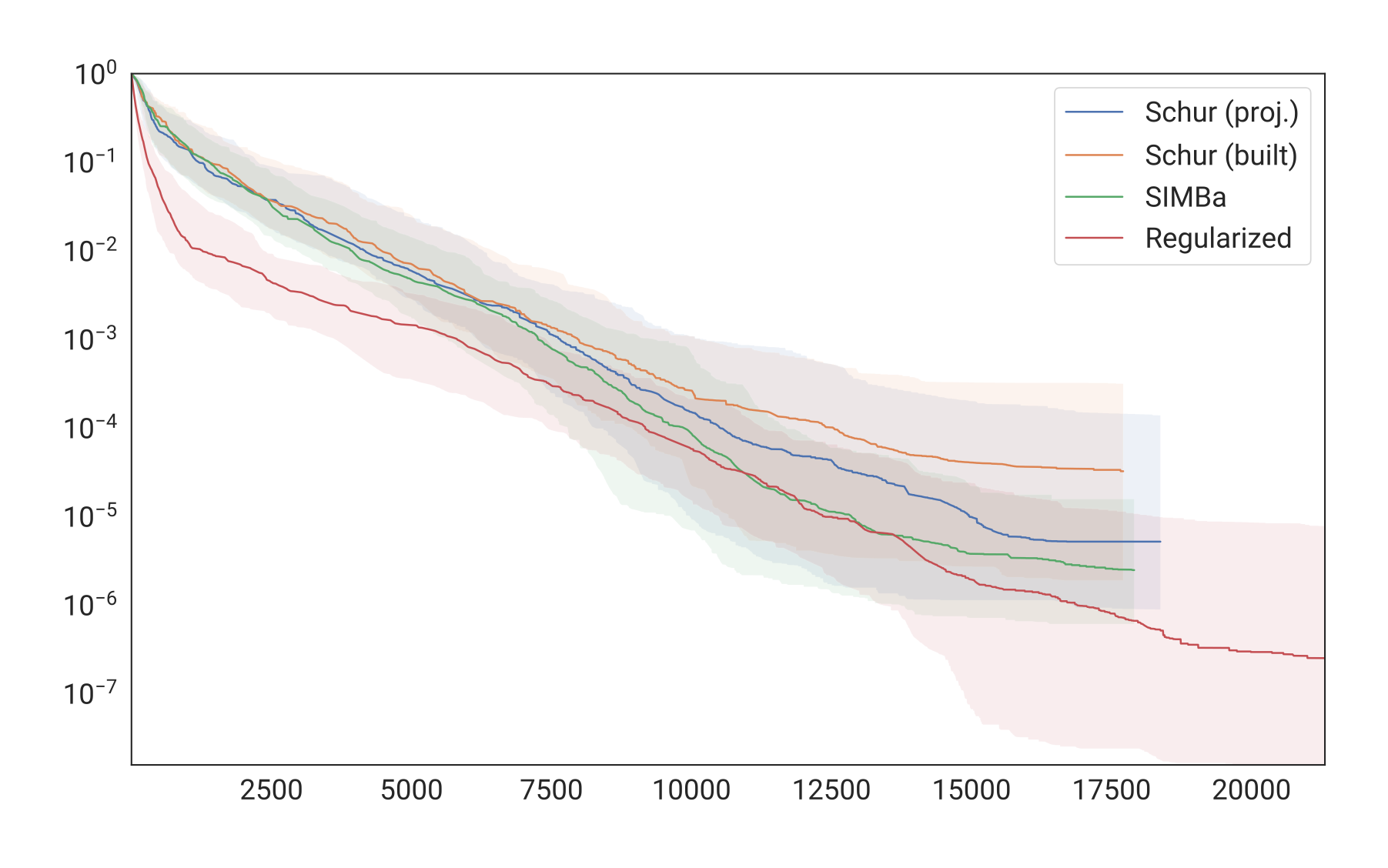}
        \caption{Smaller setup.}
    \end{subfigure}
    \hfill
    \begin{subfigure}{0.49\textwidth}
        \includegraphics[width=\linewidth]{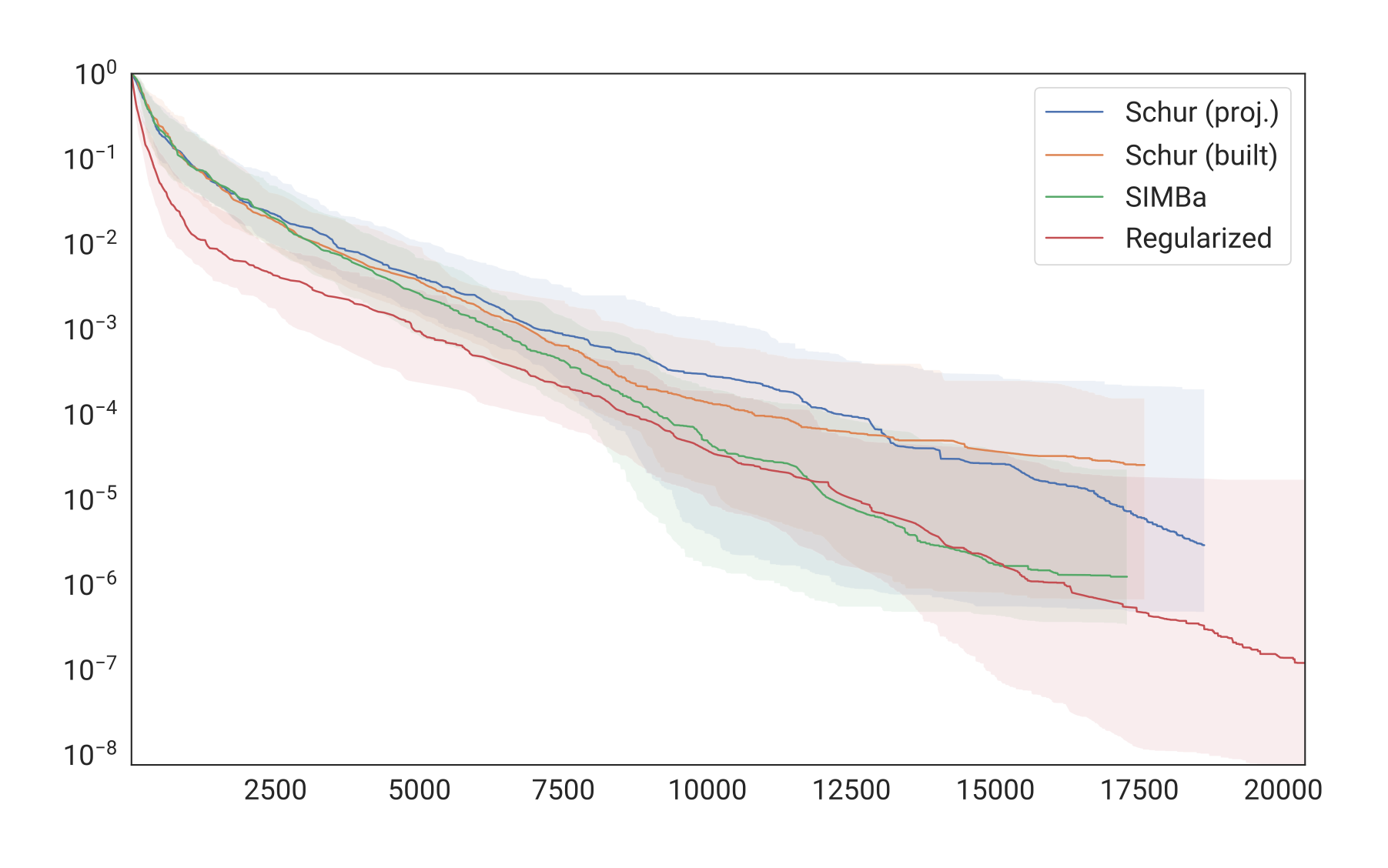}
        \caption{Small setup.}
    \end{subfigure}
    \\
    \begin{subfigure}{0.49\textwidth}
        \includegraphics[width=\linewidth]{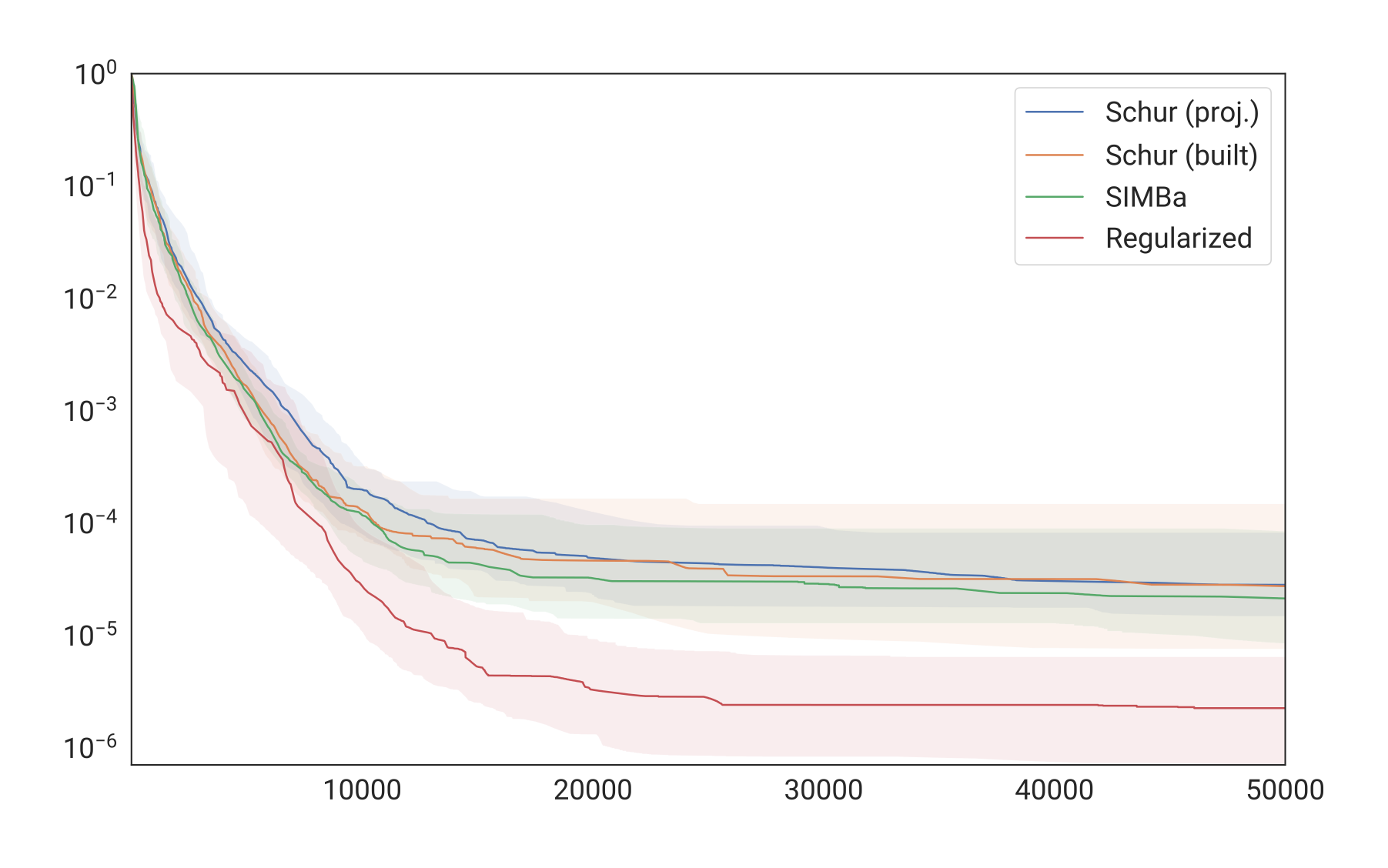}
        \caption{Original setup.}
    \end{subfigure}
    \hfill
    \begin{subfigure}{0.49\textwidth}
        \includegraphics[width=\linewidth]{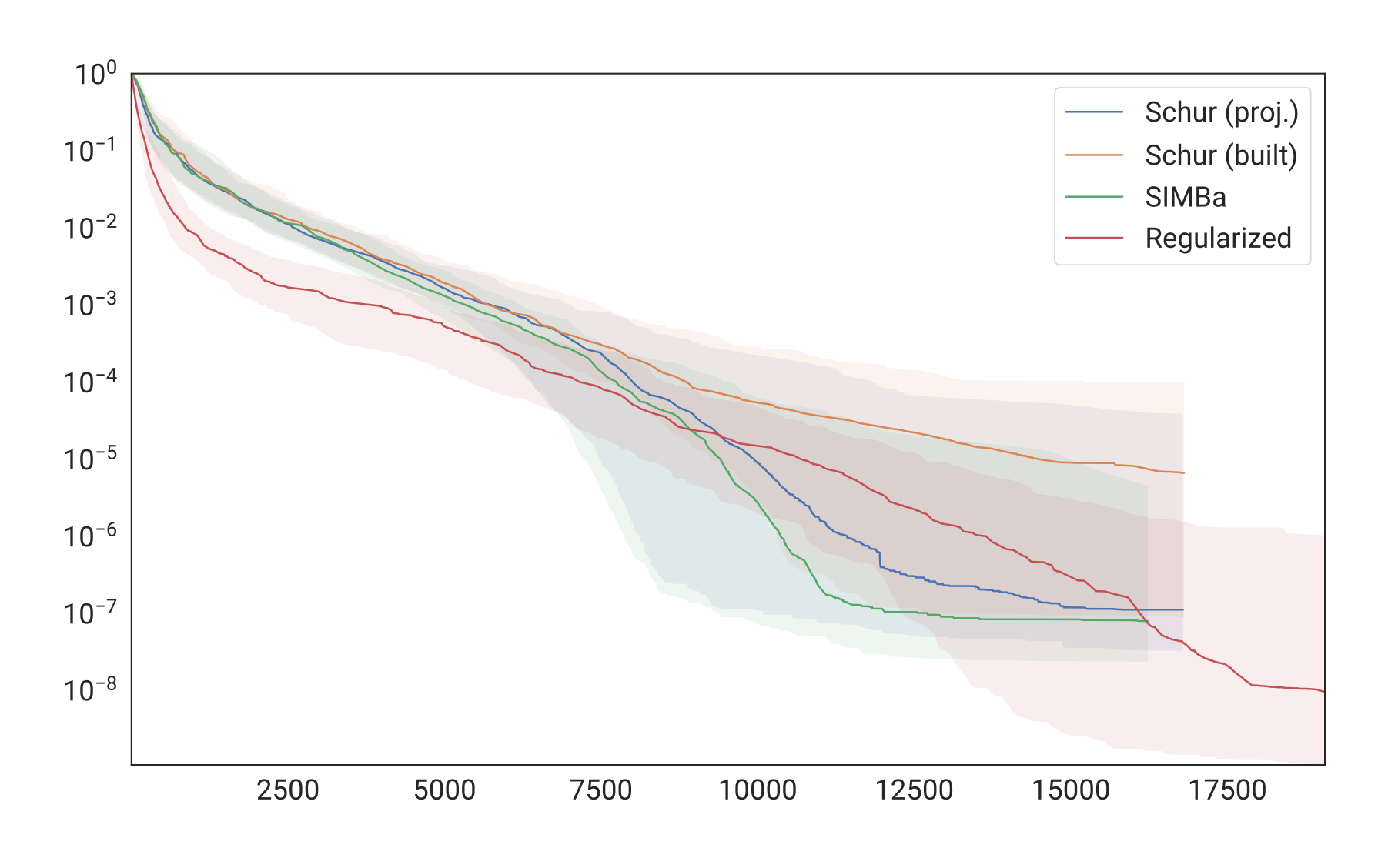}
        \caption{Extended setup.}
    \end{subfigure}
    \\
    \begin{subfigure}{0.49\textwidth}
        \includegraphics[width=\linewidth]{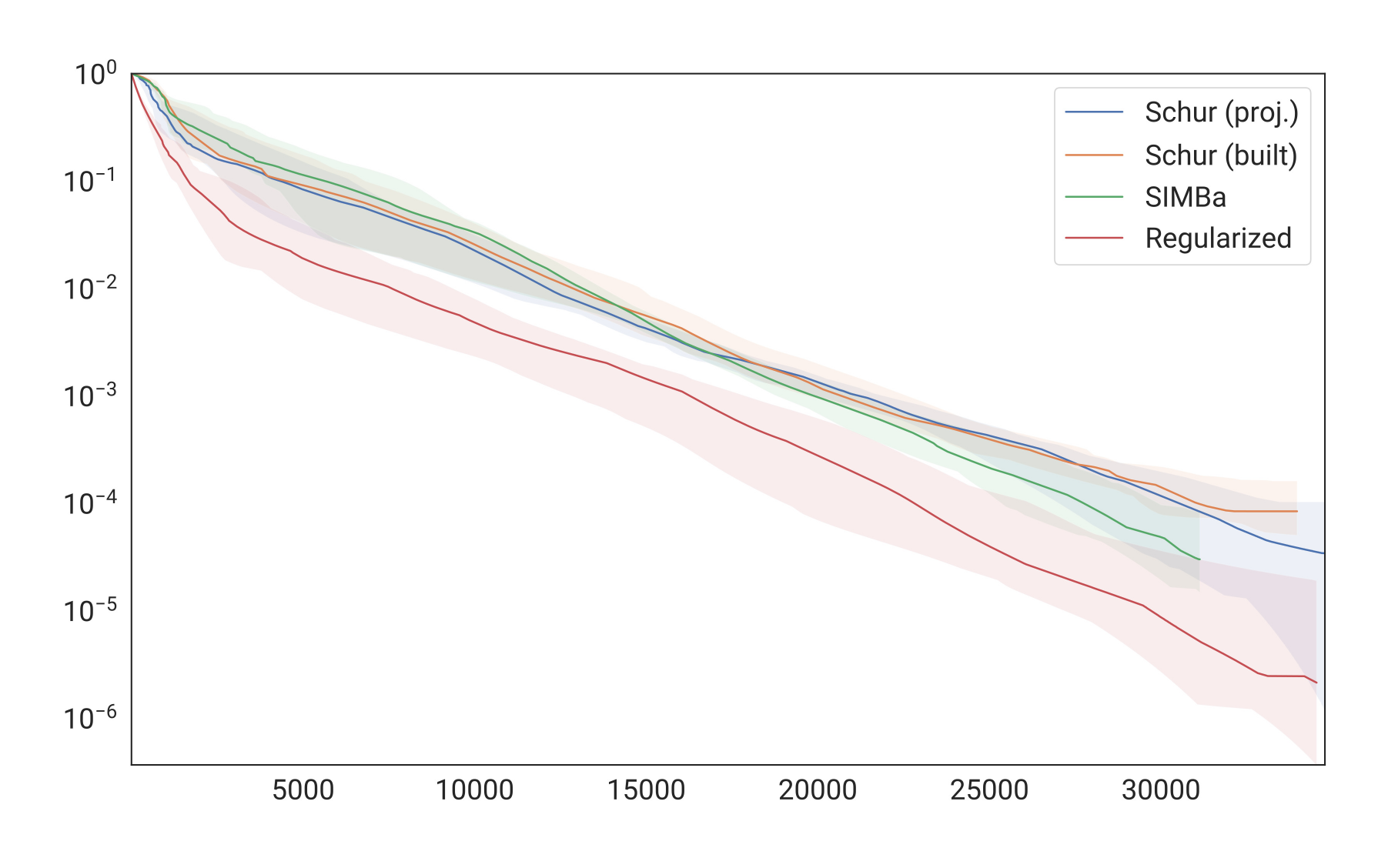}
        \caption{Large setup.}
    \end{subfigure}
    \caption{Effective (cumulative minimum) validation-loss history (normalized w.r.t. the initial value) median for the Extended setup. The training length in epochs corresponds to the median number of iterations, whereas the shaded areas' lower and upper bounds correspond to the 25\textsuperscript{th} and 75\textsuperscript{th} percentiles per epoch, respectively.}
    \label{fig:synthetic_history}
\end{figure}

\begin{figure}[tb!]
    \centering
    \begin{subfigure}{0.49\linewidth}
        \begin{subfigure}{0.49\linewidth}
            \includegraphics[width=\linewidth]{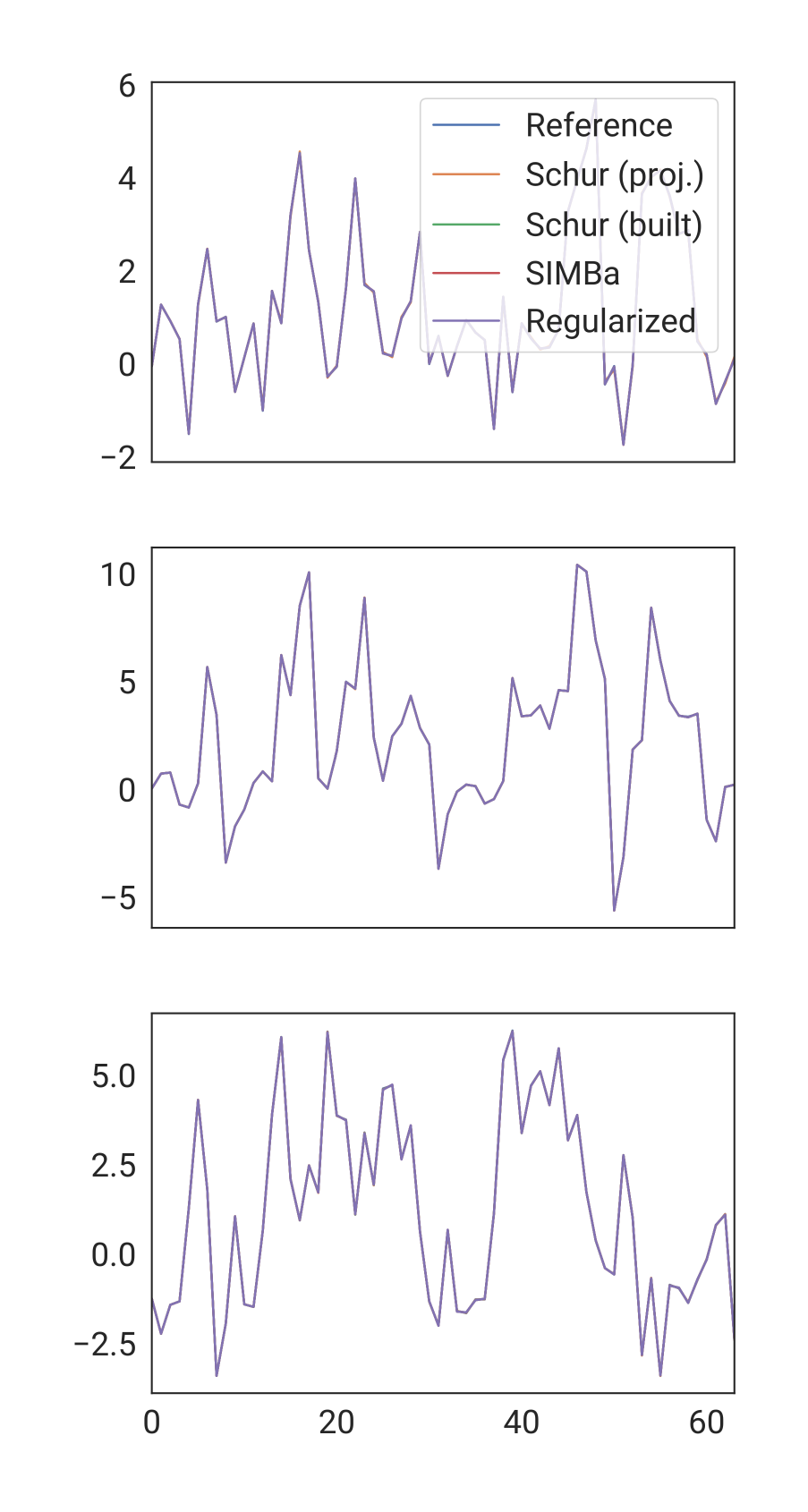}
            \caption{Smaller setup (best).}
        \end{subfigure}
        \hfill
        \begin{subfigure}{0.49\linewidth}
            \includegraphics[width=\linewidth]{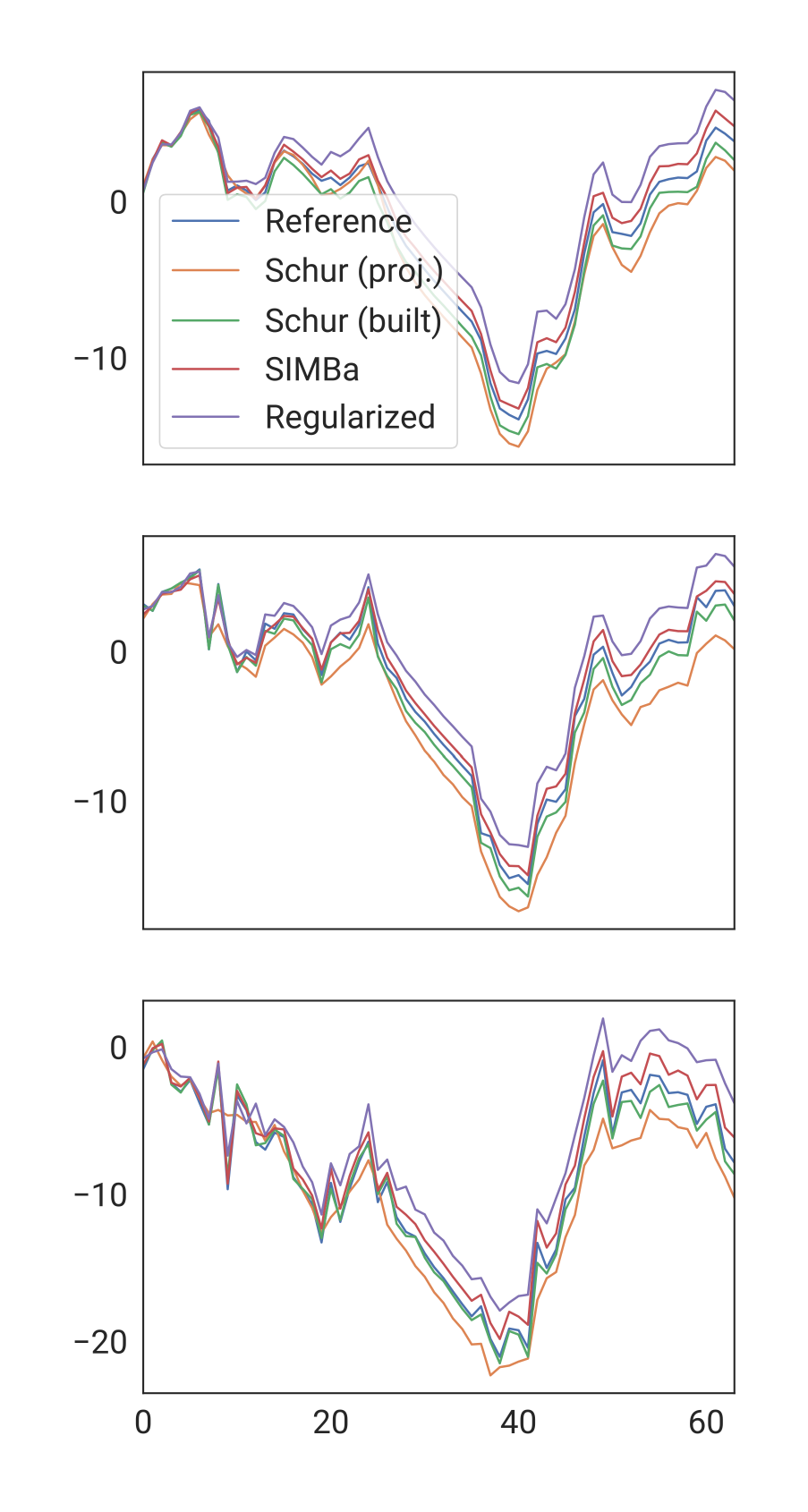}
            \caption{Smaller setup (worst).}
        \end{subfigure}
    \end{subfigure}
    \hfill
    \begin{subfigure}{0.49\linewidth}
        \begin{subfigure}{0.49\linewidth}
            \includegraphics[width=\linewidth]{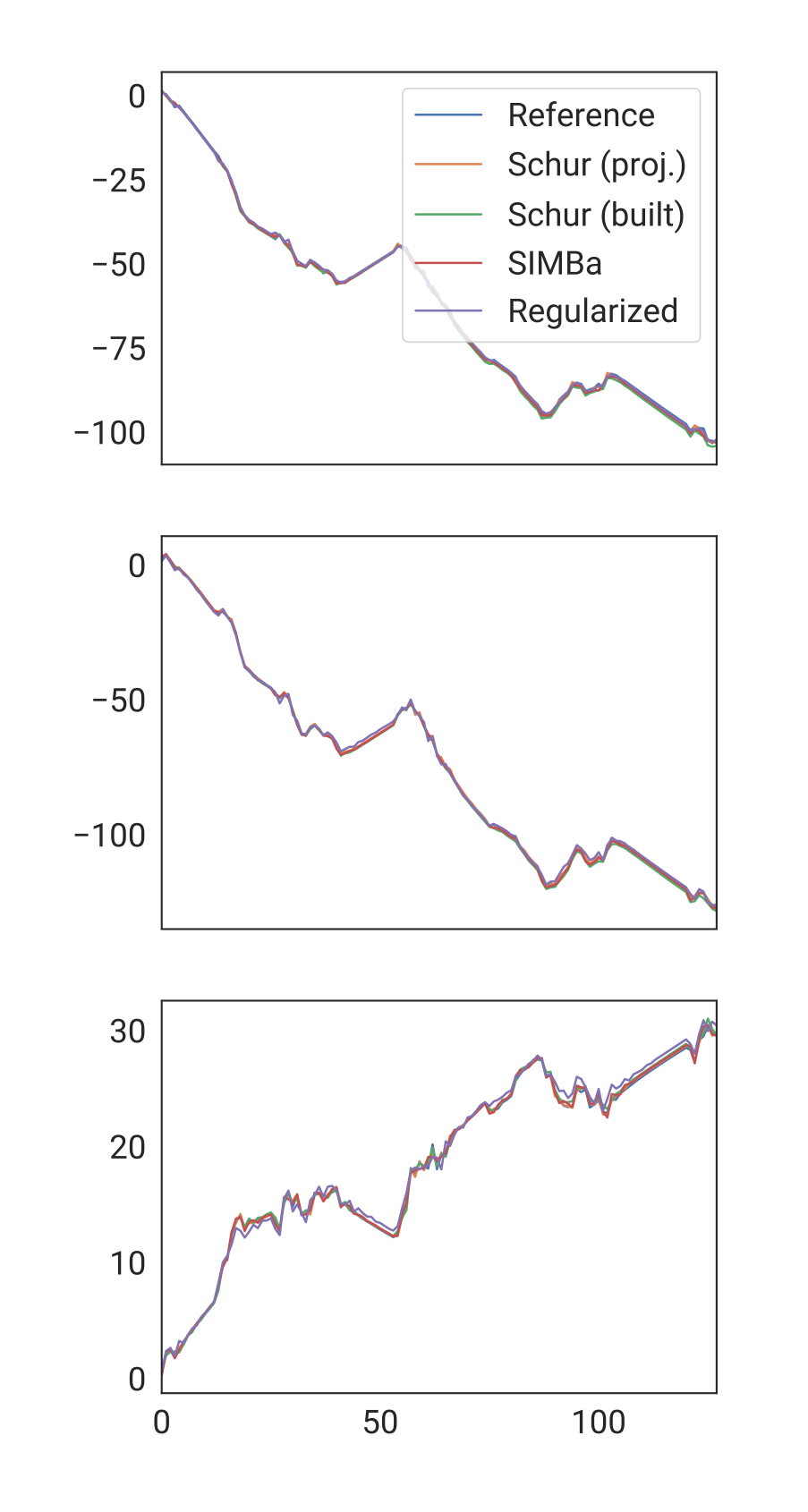}
            \caption{Small setup (best).}
        \end{subfigure}
        \hfill
        \begin{subfigure}{0.49\linewidth}
            \includegraphics[width=\linewidth]{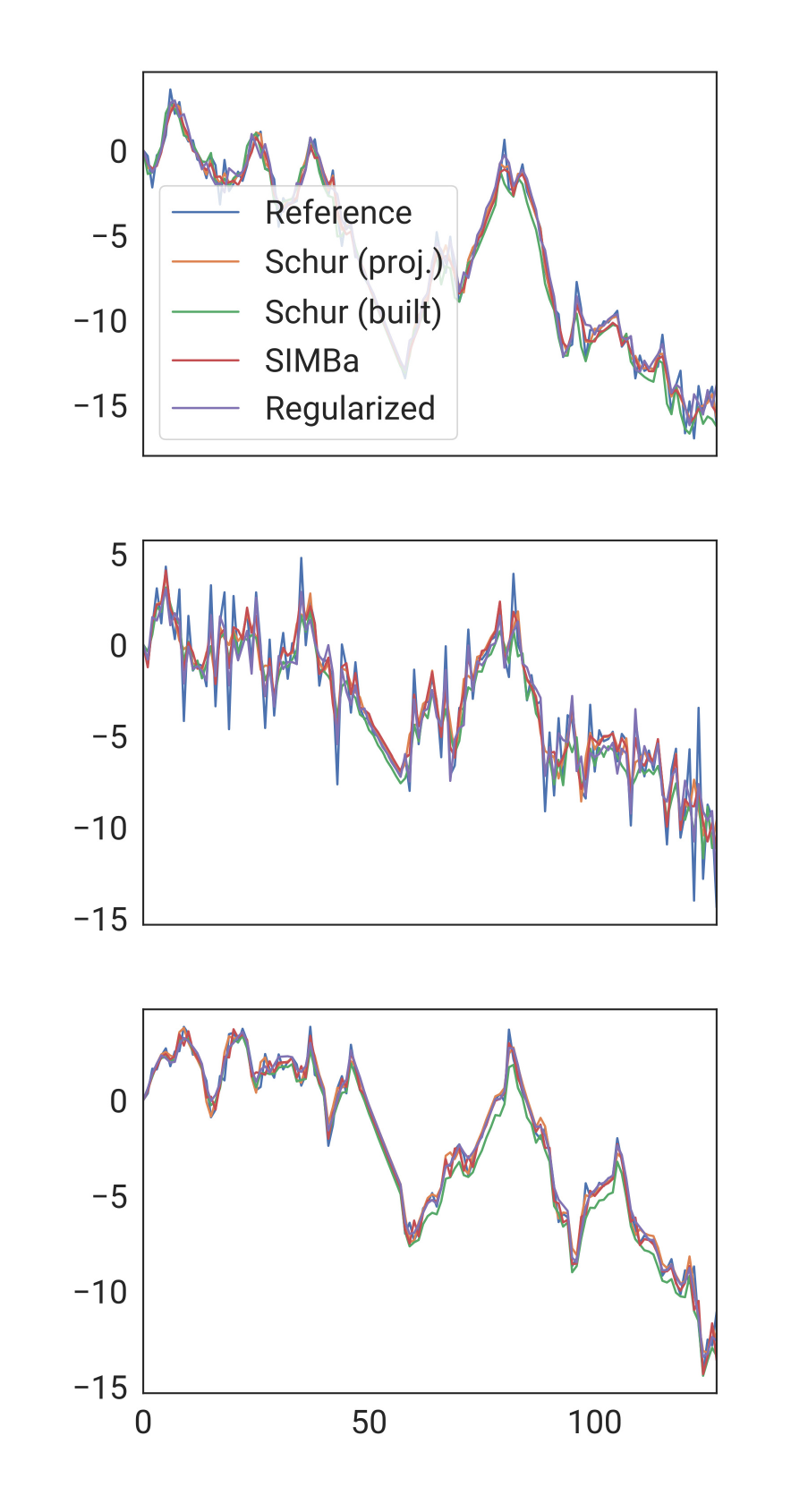}
            \caption{Small setup (worst).}
        \end{subfigure}
    \end{subfigure}
    \\
    \begin{subfigure}{0.49\linewidth}
        \begin{subfigure}{0.49\linewidth}
            \includegraphics[width=\linewidth]{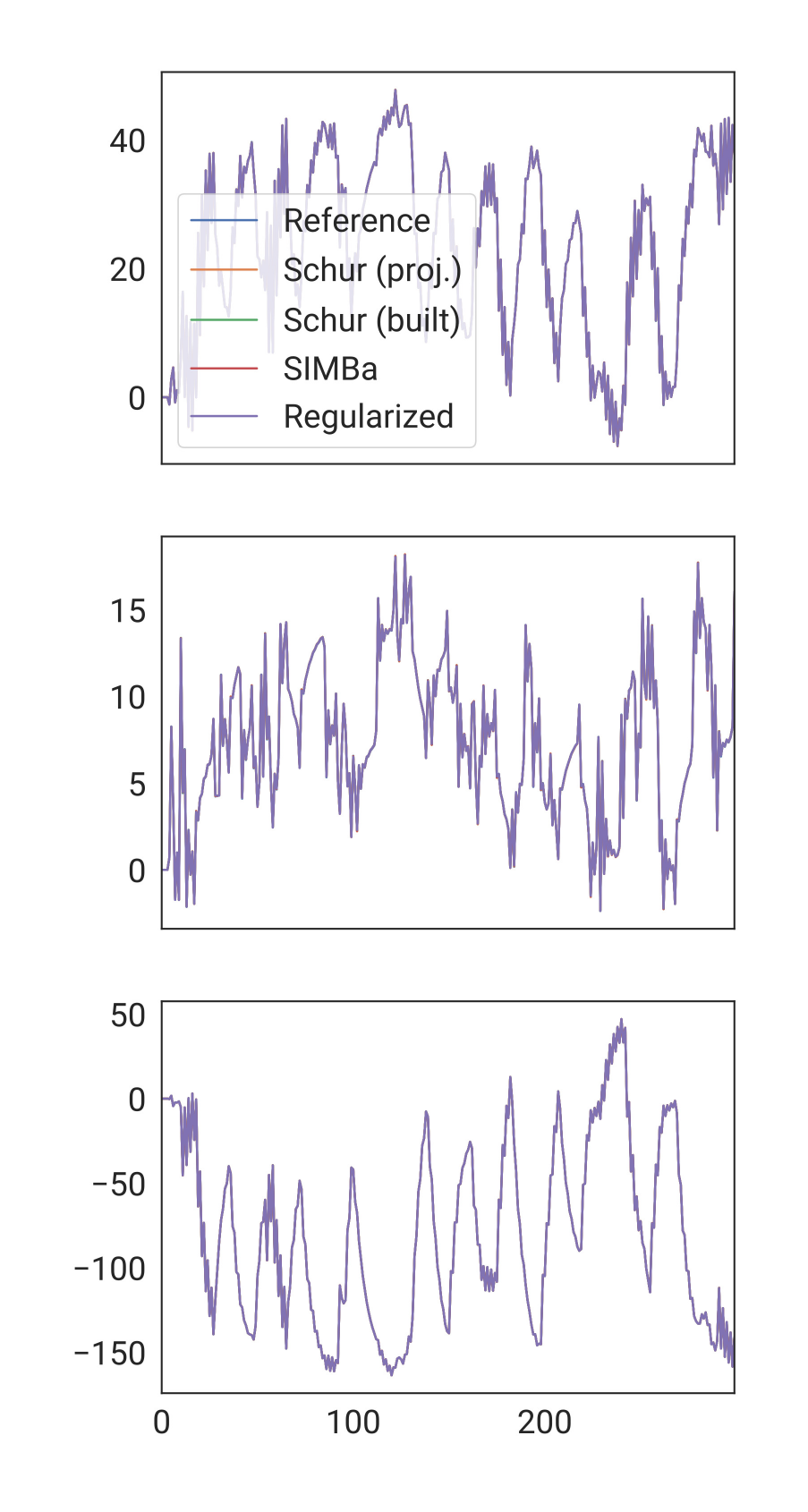}
            \caption{Original setup (best).}
        \end{subfigure}
        \hfill
        \begin{subfigure}{0.49\linewidth}
            \includegraphics[width=\linewidth]{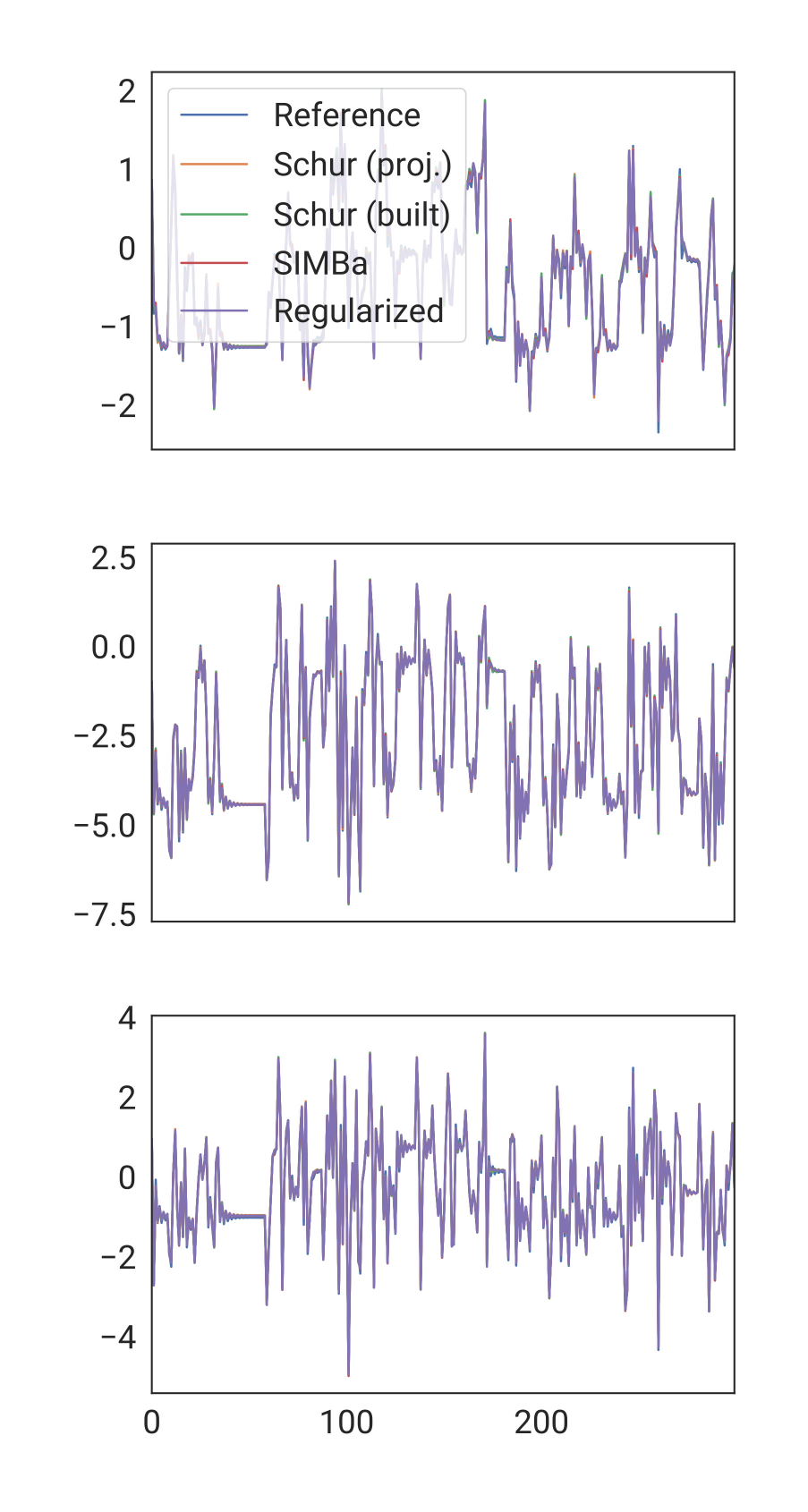}
            \caption{Original setup (worst).}
        \end{subfigure}
    \end{subfigure}
    \hfill
    \begin{subfigure}{0.49\linewidth}
        \begin{subfigure}{0.49\linewidth}
            \includegraphics[width=\linewidth]{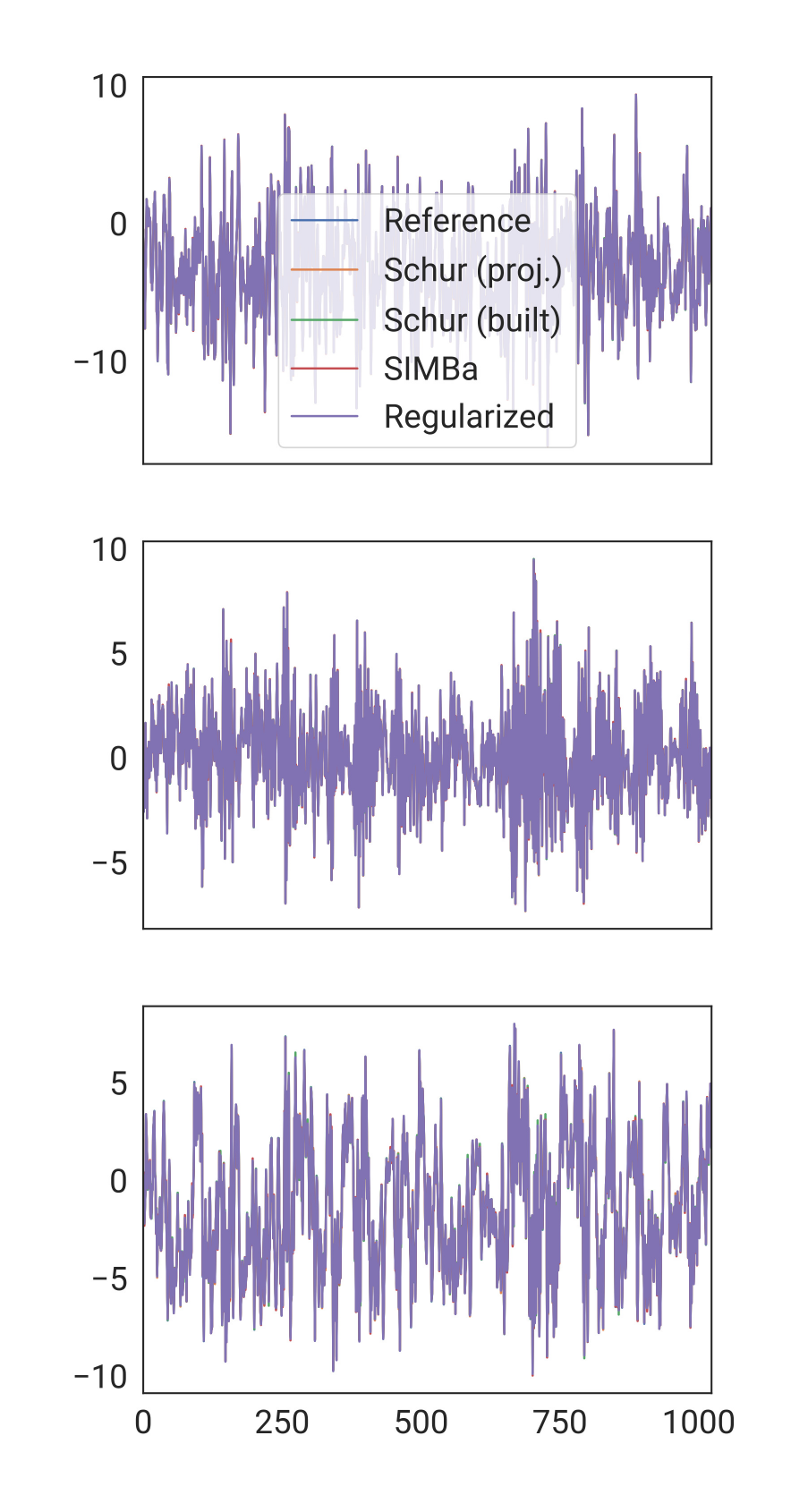}
            \caption{Extended setup (best).}
        \end{subfigure}
        \hfill
        \begin{subfigure}{0.49\linewidth}
            \includegraphics[width=\linewidth]{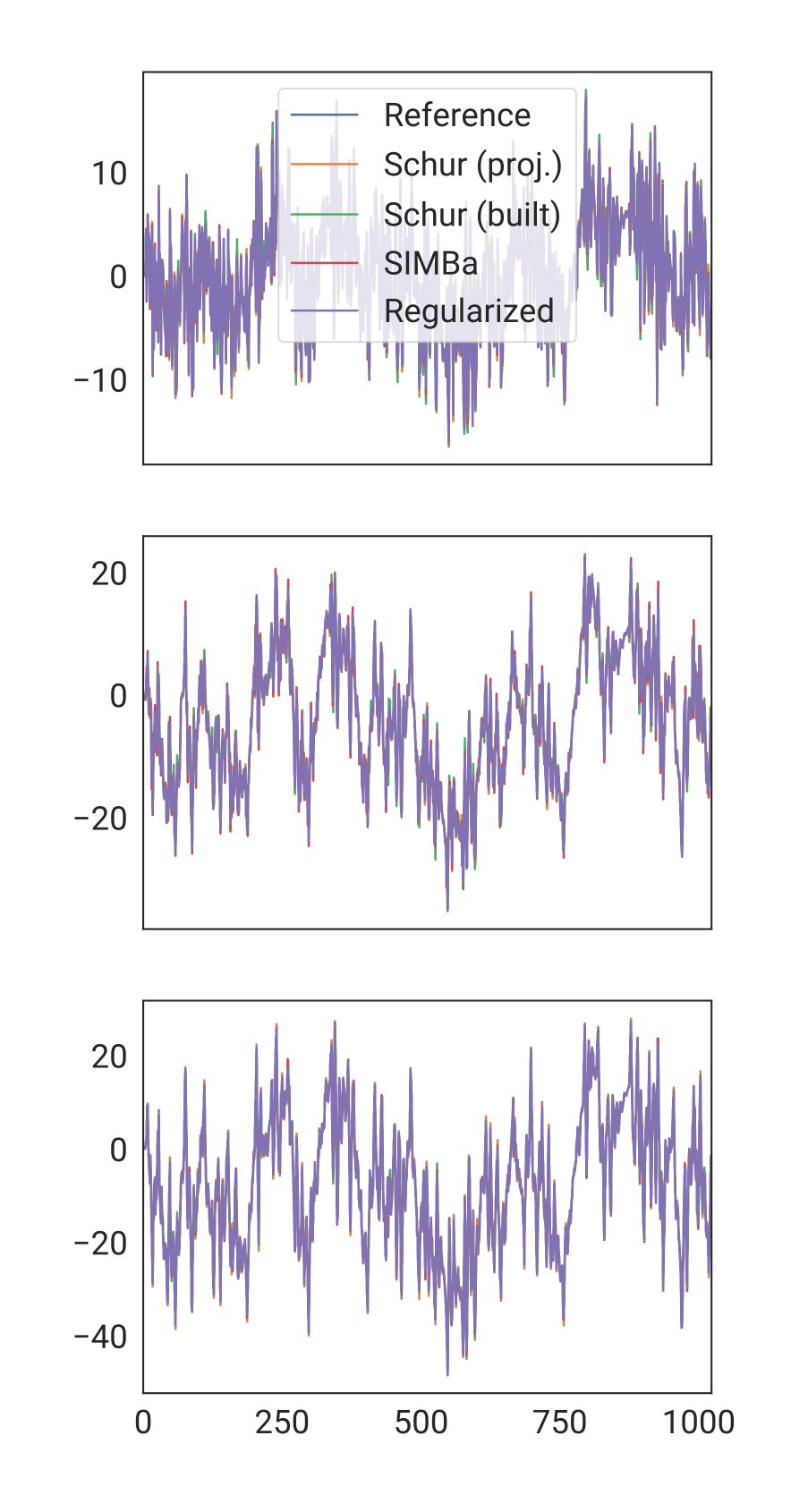}
            \caption{Extended setup (worst).}
        \end{subfigure}
    \end{subfigure}
    \caption{Best and worst cases amongst the random systems (testing partition outputs).}
    \label{fig:synthetic_outputs}
\end{figure}

The best and worst-case outputs per experimental setup (\textit{Large} excluded due to the large channel and sequence count) are shown in \Cref{fig:synthetic_outputs}. Consistent with the expected error and uncertainty reported in \Cref{subsec:synthetic_exp}, the \textit{Smaller} setup is the only one where the low sample count results in a larger probability of misidentification and signal drift. For all other configurations, even in the worst-case scenario, all the evaluated methods manage to converge (albeit at different speeds, as shown in \Cref{fig:synthetic_history}) to a proper approximation of the target system.

\section{Real-Data Benchmarks} \label[appendix]{app:real}

\subsection{Datasets \& Model Configurations} \label[appendix]{subsec:real_description}

The following datasets, obtained from the official dataloader repository of \url{nonlinearbenchmark.org/}~\cite{schoukens2026NLBench} (released under the BSD 3 license), were used for the real-data experiments in \Cref{subsec:real_exp}:
\begin{enumerate}
    \item \textbf{Silverbox}~\cite{Wigren2013}: Electronic implementation of the Duffing oscillator, built as a 2nd order \gls{lti} system with a 3rd degree polynomial static nonlinearity around it in feedback.
    \item \textbf{CED (Coupled Electric Drives)}~\cite{wigren2017coupled}: A setup consisting of two electric motors that drive a pulley using a flexible belt. The pulley is held by a spring, resulting in a lightly damped dynamic mode. The electric drives can be individually controlled, allowing the tension and the speed of the belt to be simultaneously controlled. The drive control is symmetric around zero; hence, both clockwise and counterclockwise movement are possible. The focus is only on the speed control system, and the angular speed of the pulley is measured as an output with a pulse counter insensitive to the sign of the velocity. The available data sets are short, which constitutes a challenge when performing identification.
    \item \textbf{EMPS (Electro-Mechanical Positioning System)}~\cite{janot2019data}: A standard configuration of a drive system for the prismatic joint of robots or machine tools. The main source of nonlinearity is caused by friction effects that are present in the setup. Due to the presence of a pure integrator in the system, the measurements are taken in a closed-loop setting. The original dataset was sampled at 1 kHz, but it had to be subsampled by a factor of 20 due to the slow dynamics of the system and the magnitude of the input perturbations.
    \item \textbf{Industrial Robot}~\cite{weigand2022dataset}: An identification benchmark dataset for a full robot movement with a KUKA KR300 R2500 ultra SE industrial robot. The robot has a nominal payload capacity of 300 kg, a weight of 1120 kg, and a reach of 2500 mm, exhibiting 12 states accounting for position and velocity for each of the 6 joints. The robot encounters backlash in all joints, pose-dependent inertia, pose-dependent gravitational loads, pose-dependent hydraulic forces, pose- and velocity-dependent centripetal and Coriolis forces, as well as nonlinear friction, which is temperature-dependent and therefore potentially time-varying.
    \item \textbf{Fine-Steering Mirror}~\cite{floren2024data}: A multi-input multi-output high-precision control platform used in a small satellite. Voltages applied to three piezo-actuators serve as inputs, while the mirror displacements measured at three non-collocated reference points serve as outputs. The excitation signals are orthogonal random-phase multisines spanning a wide frequency range, and they are applied at three different amplitude levels. The system behaves mostly linearly, but the presence of hysteresis in the piezo-actuators introduces dynamic nonlinearities, making the dataset well-suited for benchmarking nonlinear identification methods.
\end{enumerate}

\begin{table}[tb!]
    \caption{Dataset properties.}
    \label{tb:dataset_properties}
    \centering
    \begin{tblr}{
        colspec={*{6}{c}},
        cells={valign=m, font=\small},
        row{1-2}={font={\small \bfseries \boldmath}},
        cell{1}{1-3}={r=2}{},
        cell{1}{4}={c=3}{},
        cell{3-Z}{2-Z}={mode=math},
        hline{1-3,Z}={},
            }
        Dataset              & $n_u$ & $n_y$ & No. of Samples
        \\
                             &       &       & Training       & Validation & Test
        \\
        Silverbox            & 1     & 1     & 52050          & 13012      & 21688
        \\
        CED                  & 2     & 2     & 320            & 80         & 100
        \\
        EMPS                 & 1     & 1     & 994            & 248        & 1242
        \\
        Industrial Robot     & 6     & 6     & 31990          & 7998       & 3636
        \\
        Fine-Steering Mirror & 3     & 3     & 73728          & 73728      & 73728
    \end{tblr}
\end{table}

\begin{figure}
    \centering
    \begin{subfigure}[m]{0.49\textwidth}
        \includegraphics[width=\linewidth]{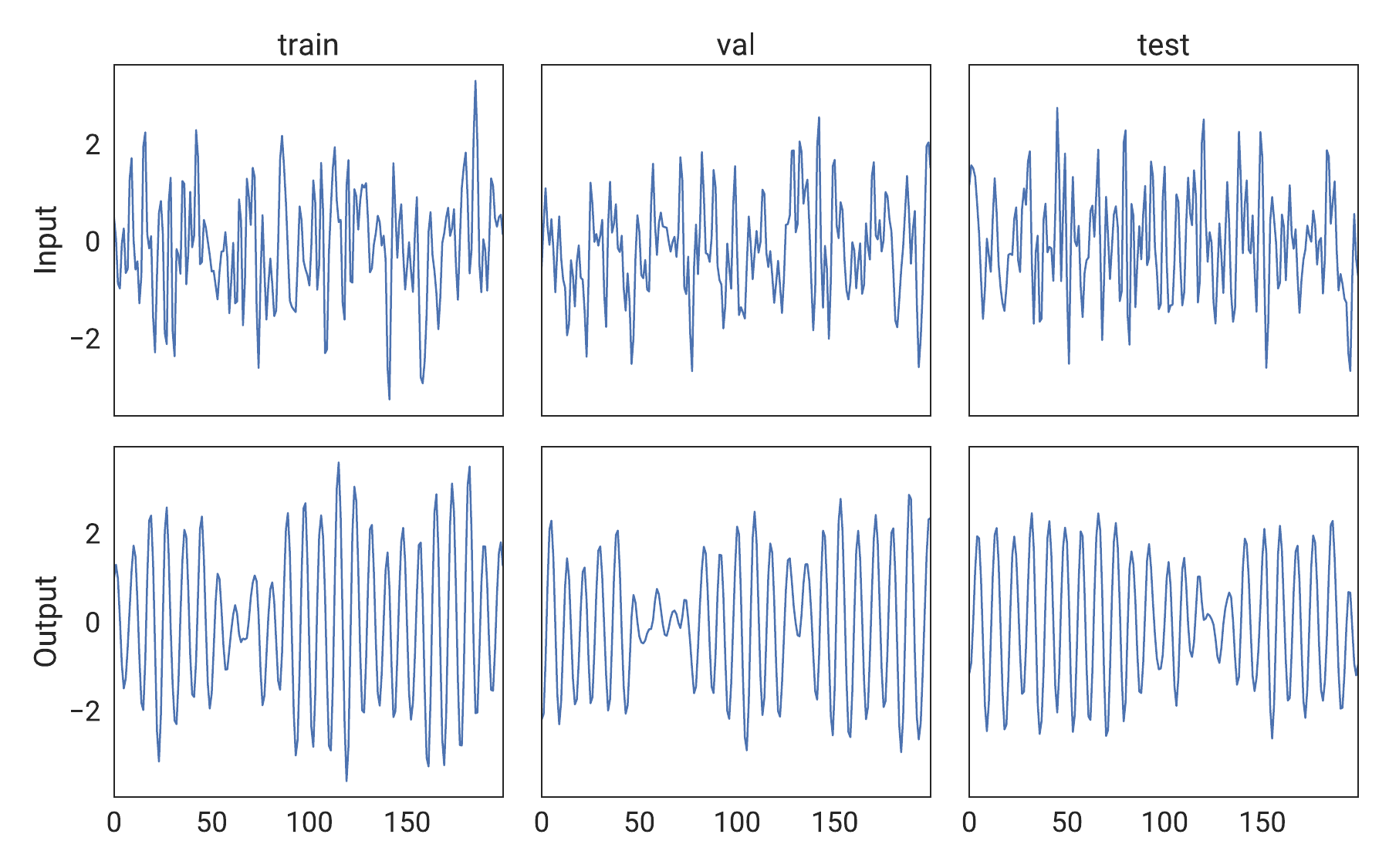}
        \caption{Silverbox}
        \label{subfig:silverbox_visualization}
    \end{subfigure}
    \hfill
    \begin{subfigure}[m]{0.49\textwidth}
        \includegraphics[width=\linewidth]{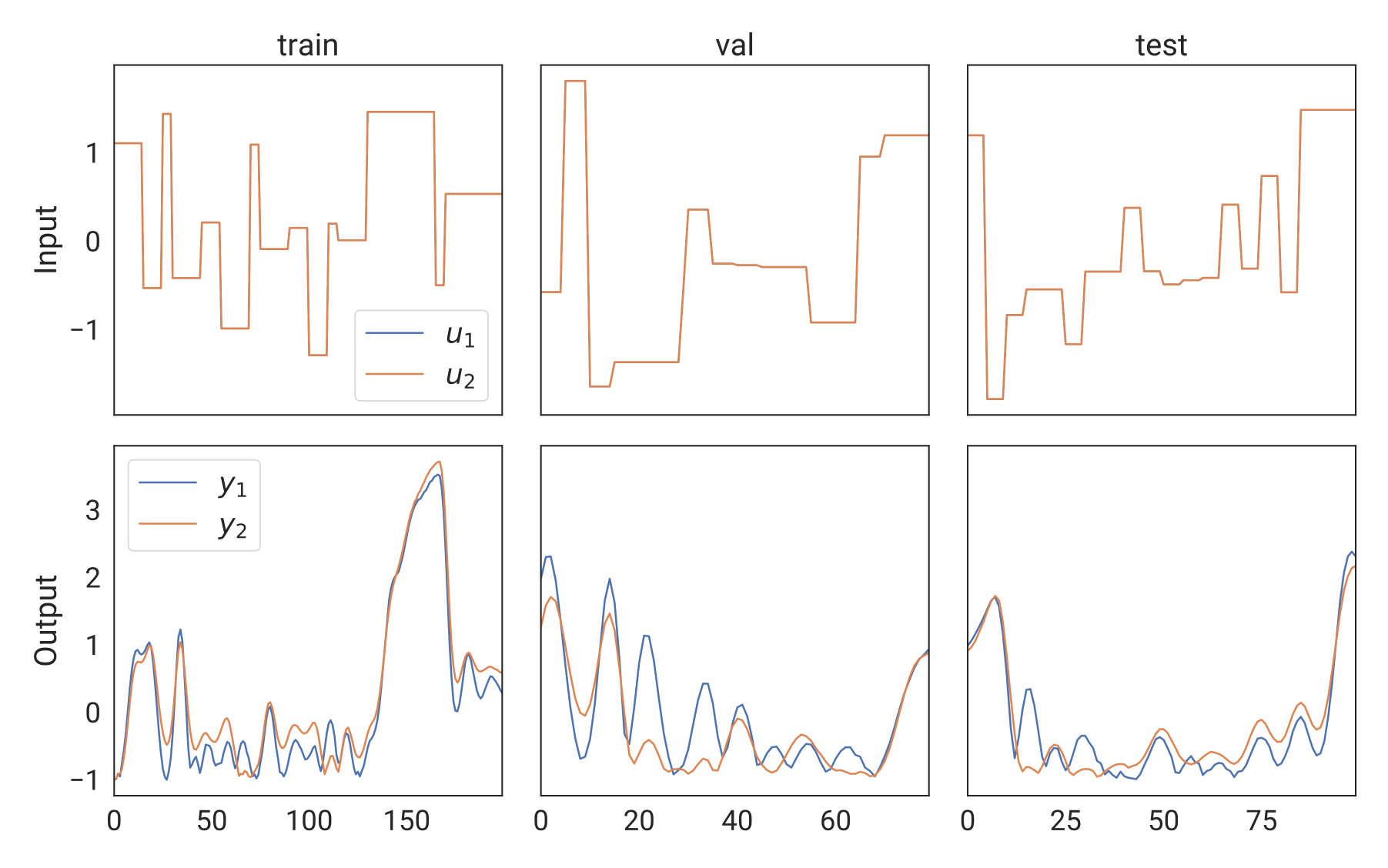}
        \caption{CED}
        \label{subfig:ced_visualization}
    \end{subfigure}
    \\
    \begin{subfigure}[m]{0.49\textwidth}
        \includegraphics[width=\linewidth]{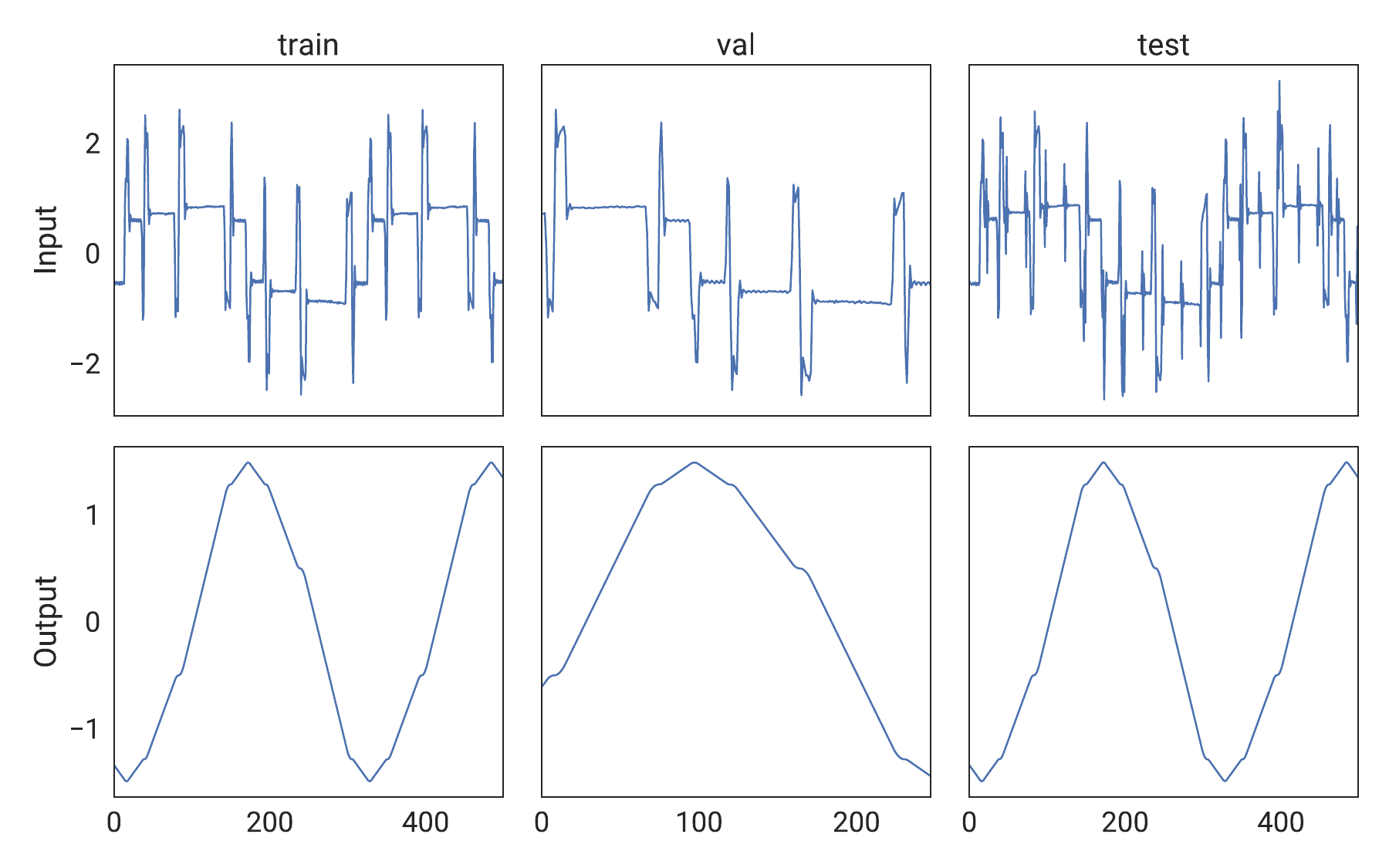}
        \caption{EMPS}
        \label{subfig:emps_visualization}
    \end{subfigure}
    \hfill
    \begin{subfigure}[m]{0.49\textwidth}
        \includegraphics[width=\linewidth]{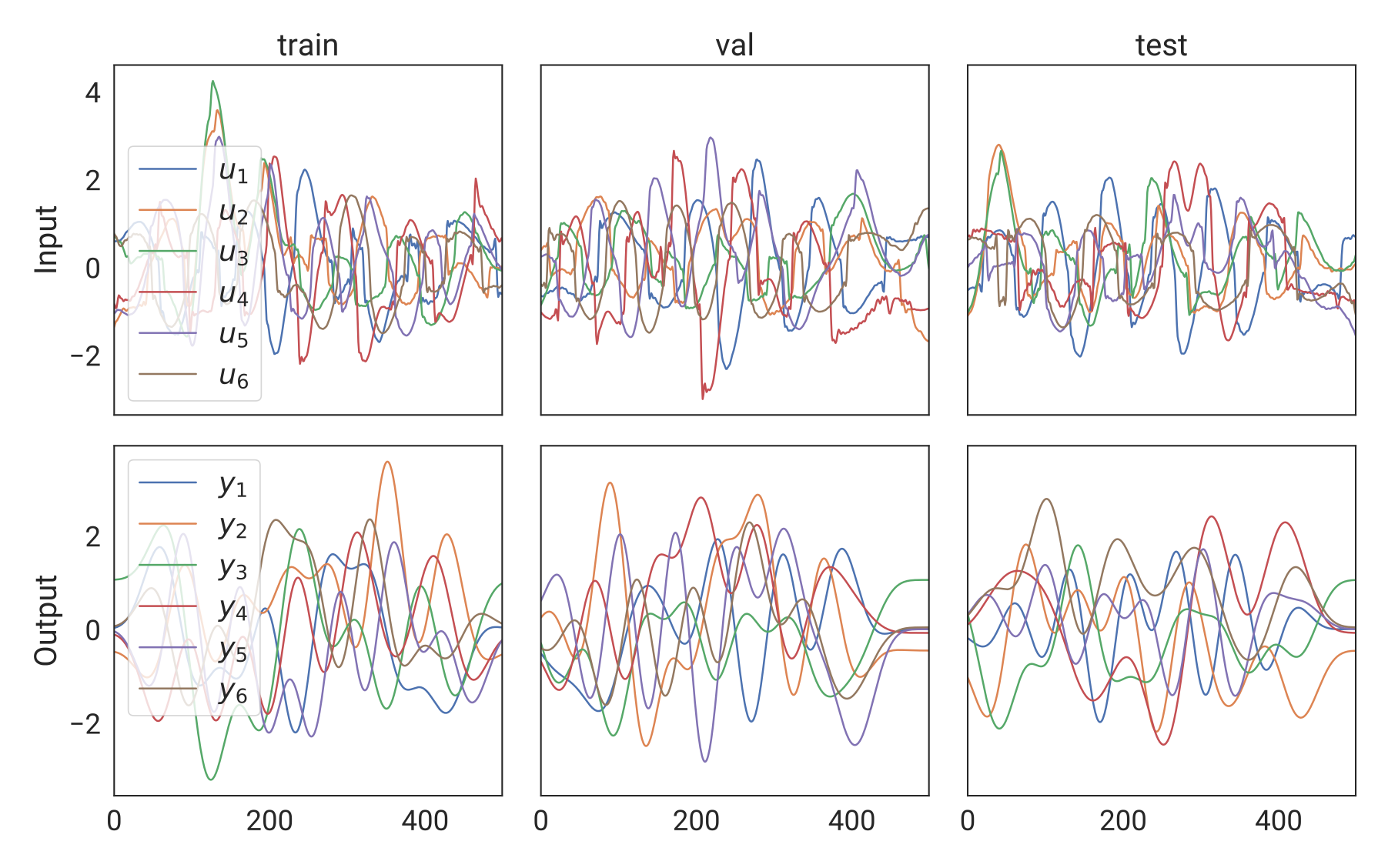}
        \caption{Industrial Robot}
        \label{subfig:industrial_robot_visualization}
    \end{subfigure}
    \\
    \begin{subfigure}[m]{0.49\textwidth}
        \includegraphics[width=\linewidth]{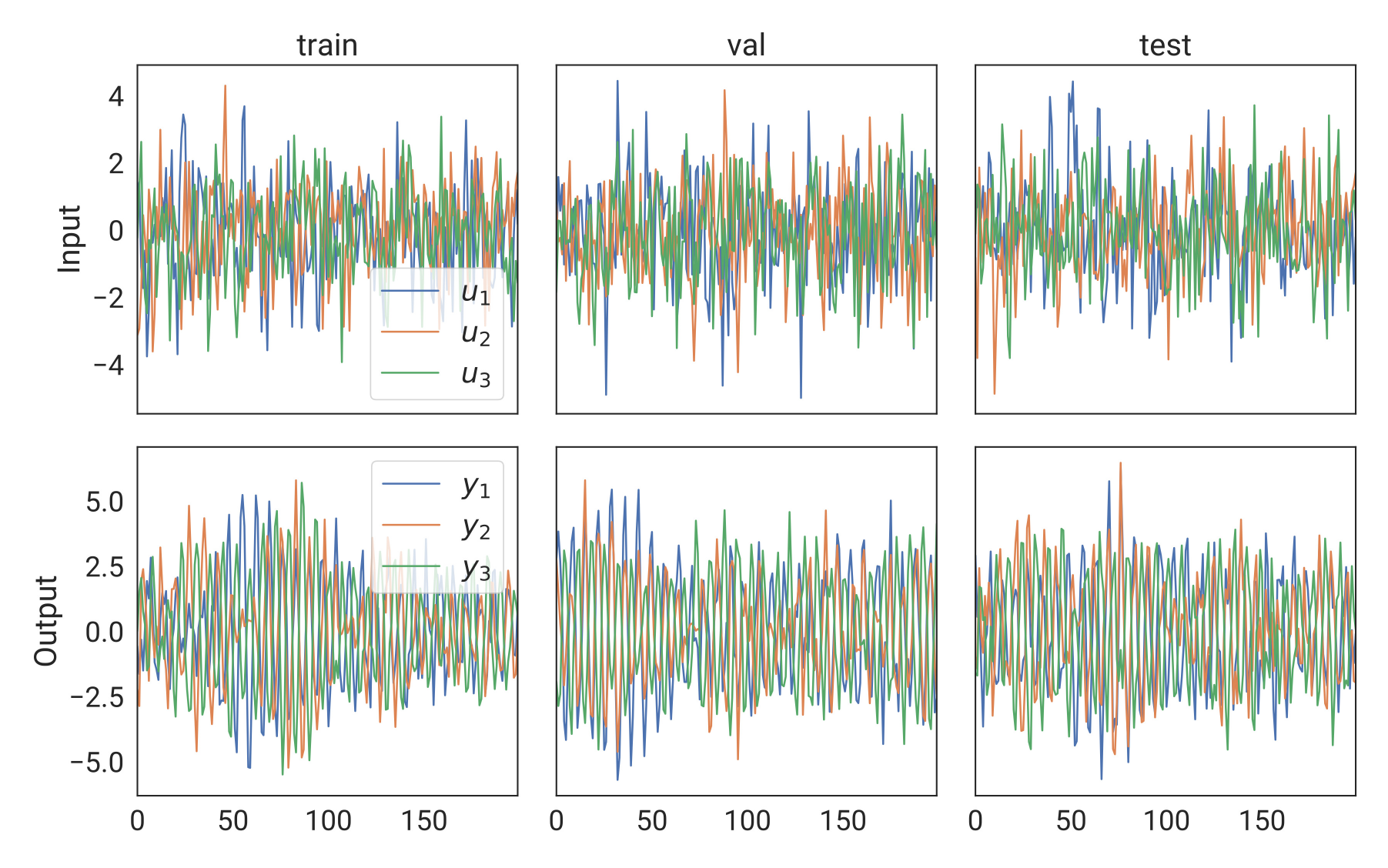}
        \caption{Fine-Steering Mirror}
        \label{subfig:fsr_visualization}
    \end{subfigure}
    \caption{Real-world dataset visualization.}
    \label{fig:real_data_visualization}
\end{figure}

The window of highest output variance per partition for each dataset can be observed in \Cref{fig:real_data_visualization}. The data and partitioning schemes were taken from \citet{schoukens2026NLBench} (details in \Cref{tb:dataset_properties}), with the last $20\%$ of the samples being reserved for validation when the training and validation samples are merged together (datasets 1-4).

Each experiment is carried out on an interactive HPC instance with 4 Xeon Gold 6230 logical cores and 4 GiB of RAM. Moreover, to facilitate best-model selection and real-time monitoring, the progress of the experiment was logged using \textit{Optuna}~\cite{optuna_2019} and hosted on a private PostgreSQL instance. Thus, to replicate these results, either a PostgreSQL instance has to be set up, or the script has to be modified to use a local SQLite file (which should not be used when benchmarking multiple datasets concurrently).

\begin{table}[tb!]
    \caption{Real-data benchmark parameters.}
    \begin{subtable}[m]{.49\linewidth}
        \caption{Model sizes per dataset.}
        \label{tb:real_model_sizes}
        \centering
        \begin{tblr}{
            colspec={*{5}{c}},
            cells={valign=m, font=\small},
            column{2-Z}={mode=math},
            row{1}={font={\small \bfseries \boldmath}},
            hline{1,2,Z}={},
                }
            Dataset              & n_f & n_x & n_g & LR
            \\
            Silverbox            & 5   & 2   & 2   & \expnumber{1}{-2}
            \\
            CED                  & 7   & 3   & 5   & \expnumber{1}{-2}
            \\
            EMPS                 & 10  & 4   & 7   & \expnumber{1}{-3}
            \\
            Industrial Robot     & 36  & 12  & 24  & \expnumber{1}{-3}
            \\
            Fine-Steering Mirror & 16  & 28  & 8   & \expnumber{1}{-3}
        \end{tblr}
    \end{subtable}
    \hfill
    \begin{subtable}[m]{.49\linewidth}
        \caption{Shared parameters.}
        \label{tb:real_shared_params}
        \centering
        \begin{tblr}{
            colspec={*{2}{c}},
            cells={valign=m, font=\small},
            cell{2-Z}{2-Z}={mode=math},
            row{1}={font={\small \bfseries}},
            hline{1,2,Z}={},
                }
            Parameter                       & Value
            \\
            Optimizer                       & \text{AdamW}
            \\
            Initial $|\lambda|$ upper-bound & 0.95
            \\
            No. of initializations          & 5
            \\
            Maximum number of epochs        & \expnumber{5}{4}
            \\
            Patience                        & \expnumber{5}{3}
            \\
            Minimum delta                   & \expnumber{1}{-3}
            \\
            EKF+RTS estimator iterations    & 20
        \end{tblr}
    \end{subtable}
\end{table}

As mentioned in \Cref{subsec:real_exp}, the model follows a Hammerstein-Wiener structure with the following components:
\begin{itemize}
    \item The input nonlinearity $f$ is implemented using a \gls{slp} with $n_f$ neurons and the \gls{silu} activation function. An output layer here would be redundant, since the output of $f$ is linearly projected by $\bm{B}$.
    \item The \gls{ss} layer, of state-size $n_x$, is \textit{transparent}; i.e., $\bm{C} = \bm{I_{n_x, n_x}}$ and $\bm{D} = \bm{0_{n_x, n_f}}$.
    \item The output nonlinearity $g$ is implemented using a \gls{mlp} with a single hidden layer of size $n_g$ and the \gls{silu} activation function (only for the hidden layer).
\end{itemize}
The model dimensions, reported in \Cref{tb:real_model_sizes}, were taken from \citet{champneys2024baseline} for the first 3 datasets, and from \citet{bemporad2025bfgs} and \citet{floren2026baseline} for datasets 4 and 5, respectively. On the other hand, the learning rate was initially set to $\expnumber{1}{-3}$ for all datasets and then increased to $\expnumber{1}{-2}$ if this change only impacted the convergence rate across the board and not the expected accuracy. Shared experimental parameters are reported in \Cref{tb:real_shared_params}.

\subsection{Supplementary Results} \label[appendix]{subsec:real_supplementary}

\begin{table*}[tb!]
    \caption{Channel-wise \gls{nmse} for multiple-output datasets.}
    \label{tb:real_channel_error}
    \centering
    \begin{tblr}{
        colspec={*{6}{c}},
        colsep={3pt},
        cells={valign=m,font=\footnotesize},
        cell{2}{1}={r=2}{},
        cell{4}{1}={r=6}{},
        cell{10}{1}={r=3}{},
        row{1}={font={\bfseries \footnotesize}},
        cell{2-Z}{2-Z}={mode=math},
        hline{1,2,4,10,Z}={},
            }
        Dataset              & Channel & Schur (proj.)                   & Schur (built)                   & SIMBa                           & Regularized
        \\
        CED                  & y_1     & \expnumber{(4.39 \pm 9.37)}{-2} & \expnumber{(5.33 \pm 11.6)}{-2} & \expnumber{(6.07 \pm 13.4)}{-2} & \expnumber{(5.79 \pm 10.2)}{-2}
        \\
                             & y_2     & \expnumber{(2.22 \pm 4.40)}{-2} & \expnumber{(2.08 \pm 3.72)}{-2} & \expnumber{(2.18 \pm 3.32)}{-2} & \expnumber{(2.37 \pm 5.83)}{-2}
        \\
        Industrial Robot     & y_1     & \expnumber{(4.71 \pm 6.66)}{-1} & \expnumber{(3.91 \pm 5.53)}{-1} & \expnumber{(4.46 \pm 6.06)}{-1} & \expnumber{(5.66 \pm 1.32)}{-1}
        \\
                             & y_2     & \expnumber{(5.71 \pm 1.10)}{-1} & \expnumber{(6.67 \pm 1.63)}{-1} & \expnumber{(4.90 \pm 8.44)}{-1} & \expnumber{(7.71 \pm 20.5)}{-1}
        \\
                             & y_3     & \expnumber{(4.38 \pm 6.79)}{-1} & \expnumber{(5.22 \pm 8.46)}{-1} & \expnumber{(5.15 \pm 8.09)}{-1} & \expnumber{(4.48 \pm 6.74)}{-1}
        \\
                             & y_4     & \expnumber{(7.14 \pm 11.5)}{-1} & \expnumber{(6.53 \pm 11.3)}{-1} & \expnumber{(8.60 \pm 13.9)}{-1} & \expnumber{(9.92 \pm 26.7)}{-1}
        \\
                             & y_5     & \expnumber{(1.21 \pm 2.16)}{0}  & \expnumber{(6.83 \pm 12.2)}{-1} & \expnumber{(1.31 \pm 3.63)}{0}  & \expnumber{(9.82 \pm 19.5)}{-1}
        \\
                             & y_6     & \expnumber{(5.69 \pm 7.35)}{-1} & \expnumber{(5.39 \pm 7.84)}{-1} & \expnumber{(7.25 \pm 9.94)}{-1} & \expnumber{(6.82 \pm 8.79)}{-1}
        \\
        Fine-Steering Mirror & y_1     & \expnumber{(3.02 \pm 5.57)}{-2} & \expnumber{(2.60 \pm 5.00)}{-2} & \expnumber{(2.46 \pm 4.25)}{-2} & \expnumber{(2.59 \pm 4.73)}{-2}
        \\
                             & y_2     & \expnumber{(5.35 \pm 1.16)}{-2} & \expnumber{(4.52 \pm 9.58)}{-2} & \expnumber{(3.29 \pm 7.15)}{-2} & \expnumber{(4.57 \pm 9.81)}{-2}
        \\
                             & y_3     & \expnumber{(5.91 \pm 9.47)}{-2} & \expnumber{(4.75 \pm 7.82)}{-2} & \expnumber{(5.11 \pm 7.80)}{-2} & \expnumber{(4.81 \pm 7.87)}{-2}
    \end{tblr}
\end{table*}

\begin{figure}[tb!]
    \centering
    \begin{subfigure}{0.49\textwidth}
        \includegraphics[width=\linewidth]{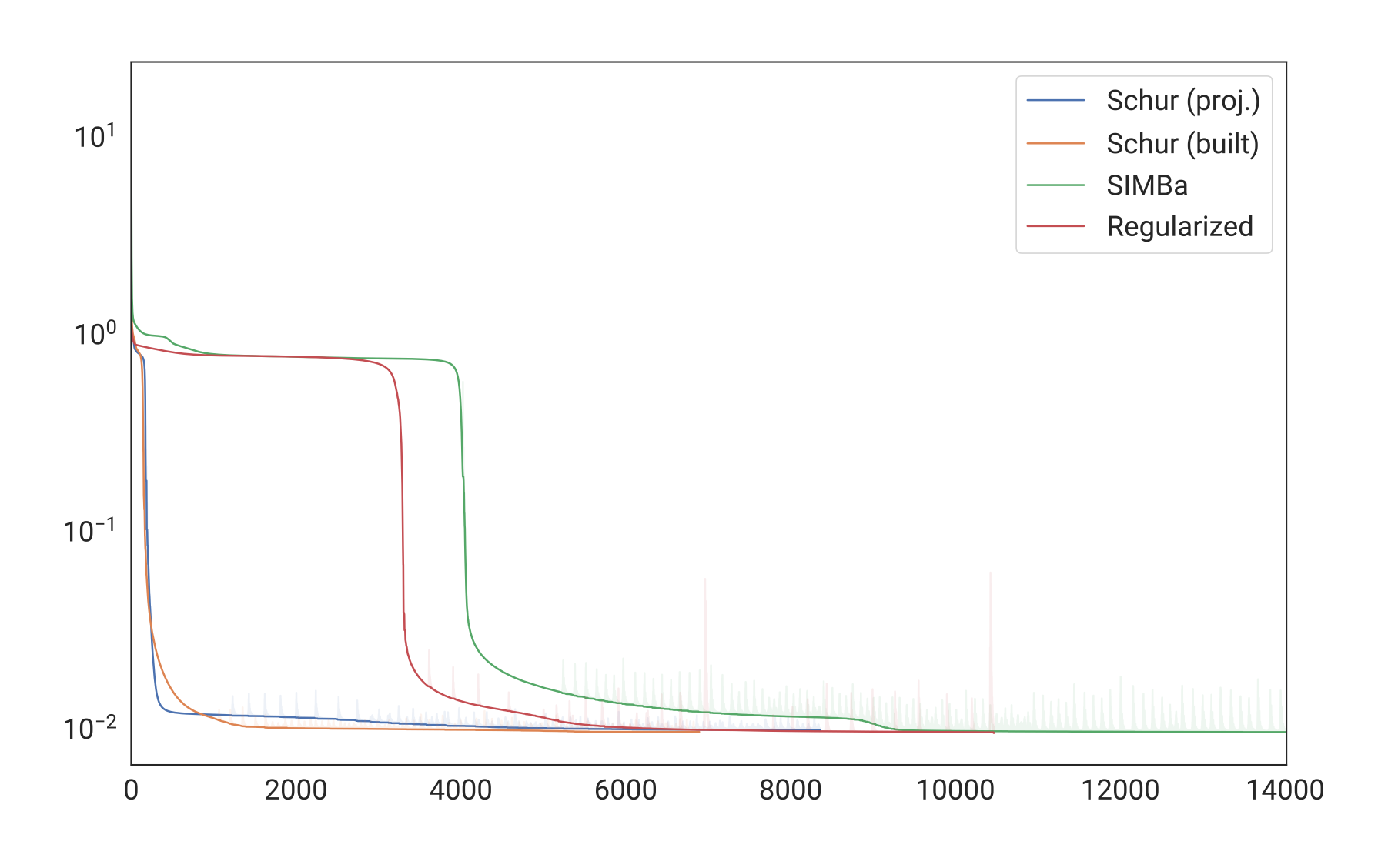}
        \caption{Silverbox}
    \end{subfigure}
    \hfill
    \begin{subfigure}{0.49\textwidth}
        \includegraphics[width=\linewidth]{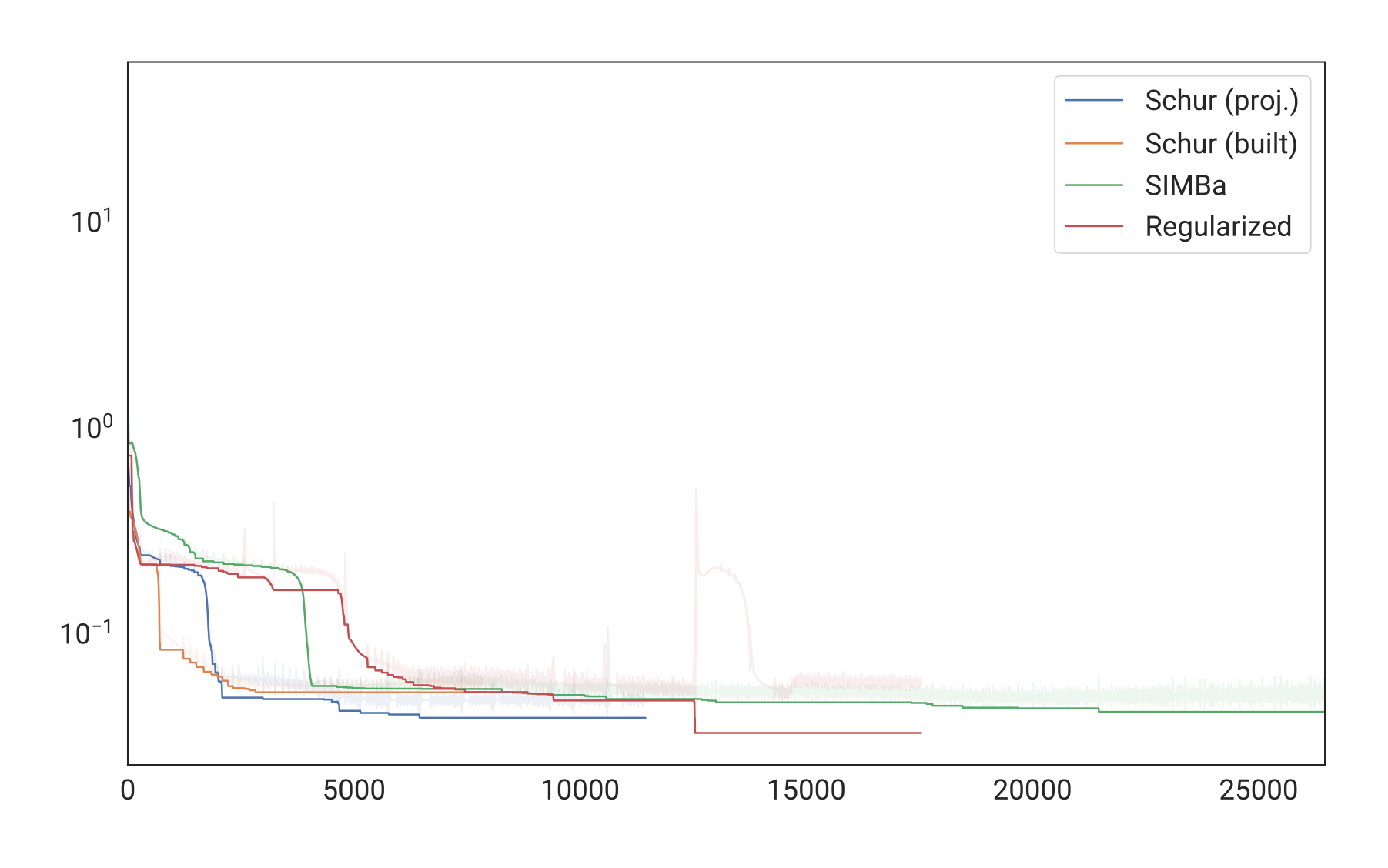}
        \caption{CED}
    \end{subfigure}
    \\
    \begin{subfigure}{0.49\textwidth}
        \includegraphics[width=\linewidth]{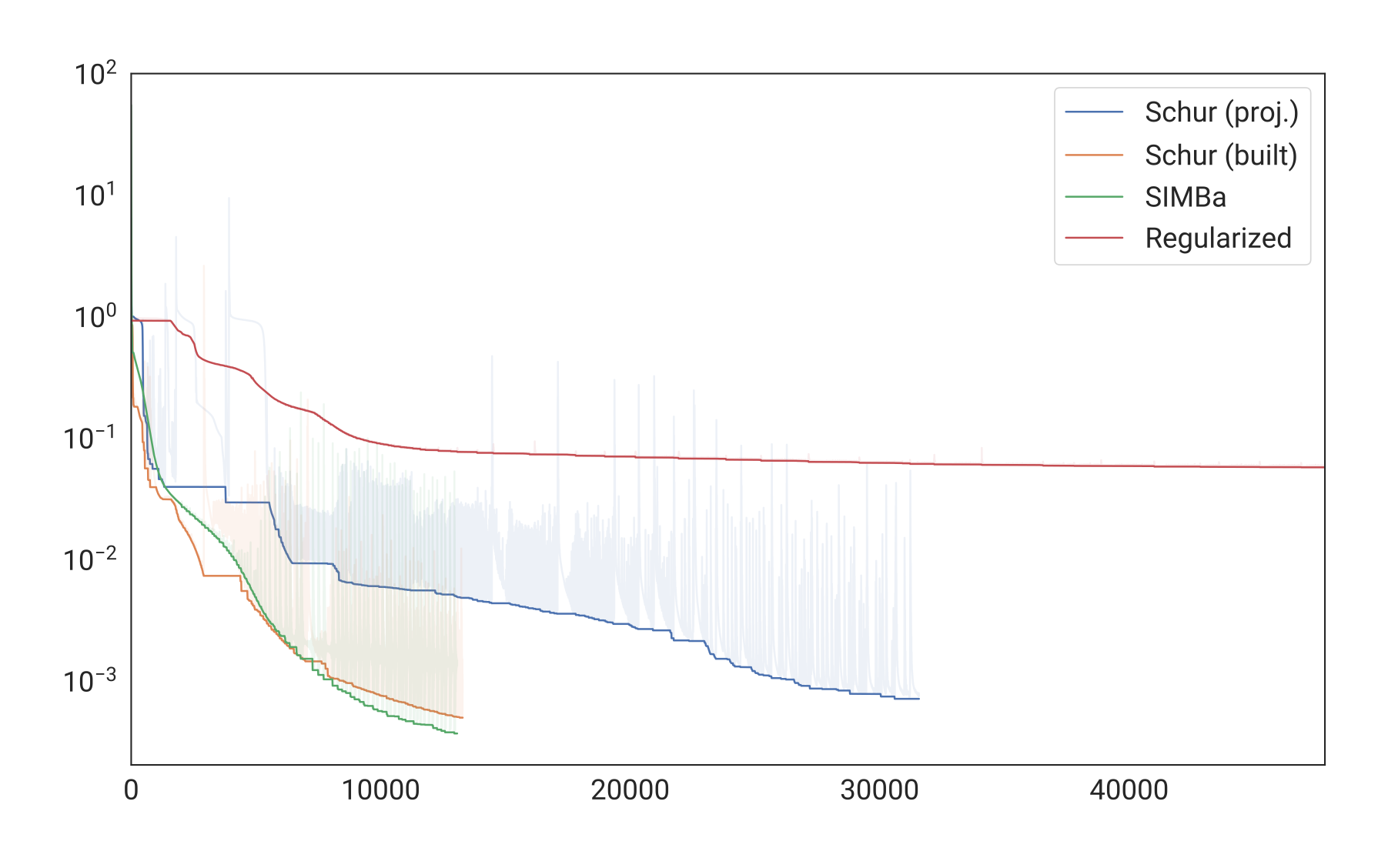}
        \caption{EMPS}
    \end{subfigure}
    \hfill
    \begin{subfigure}{0.49\textwidth}
        \includegraphics[width=\linewidth]{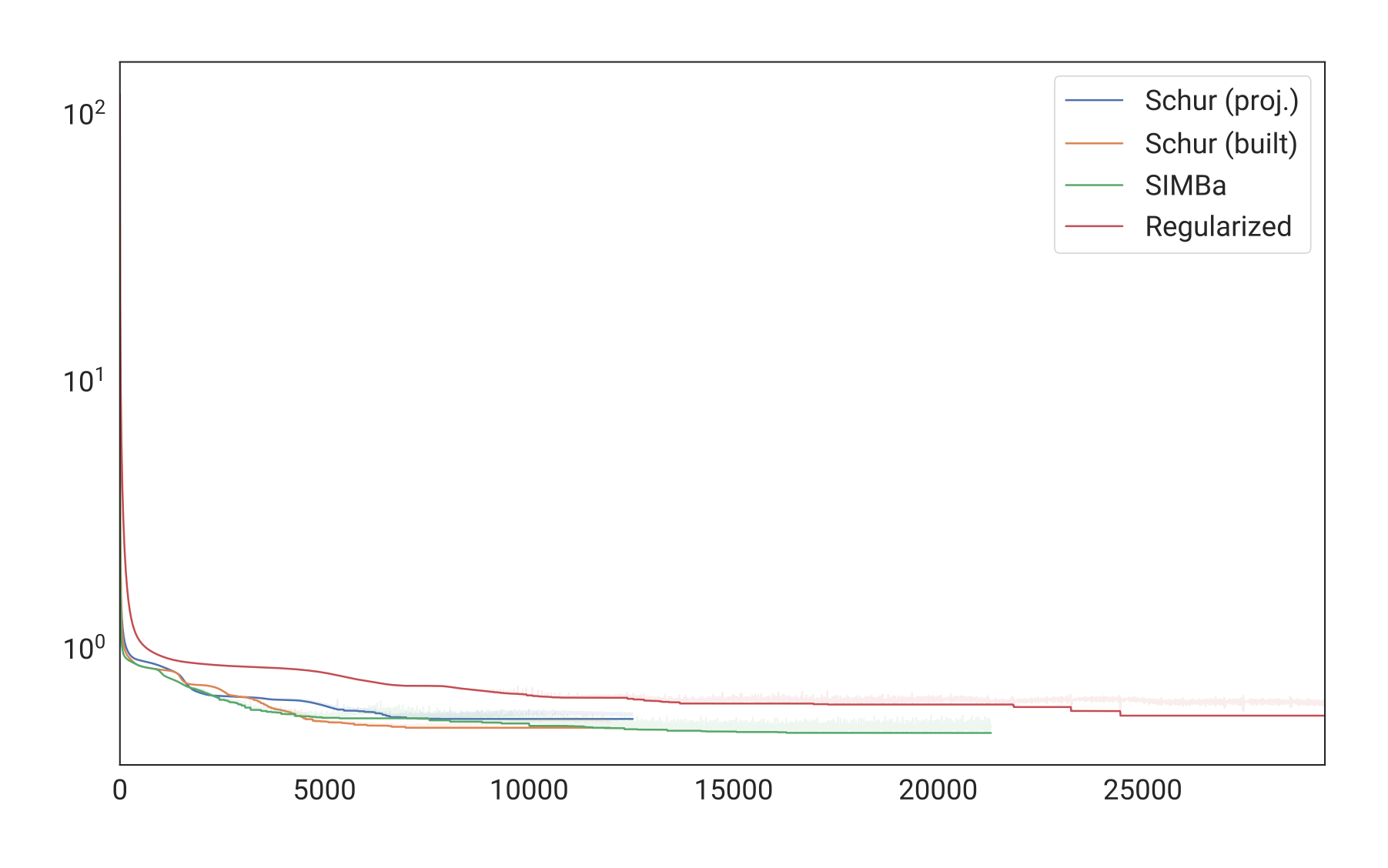}
        \caption{Industrial robot}
    \end{subfigure}
    \\
    \begin{subfigure}{0.49\textwidth}
        \includegraphics[width=\linewidth]{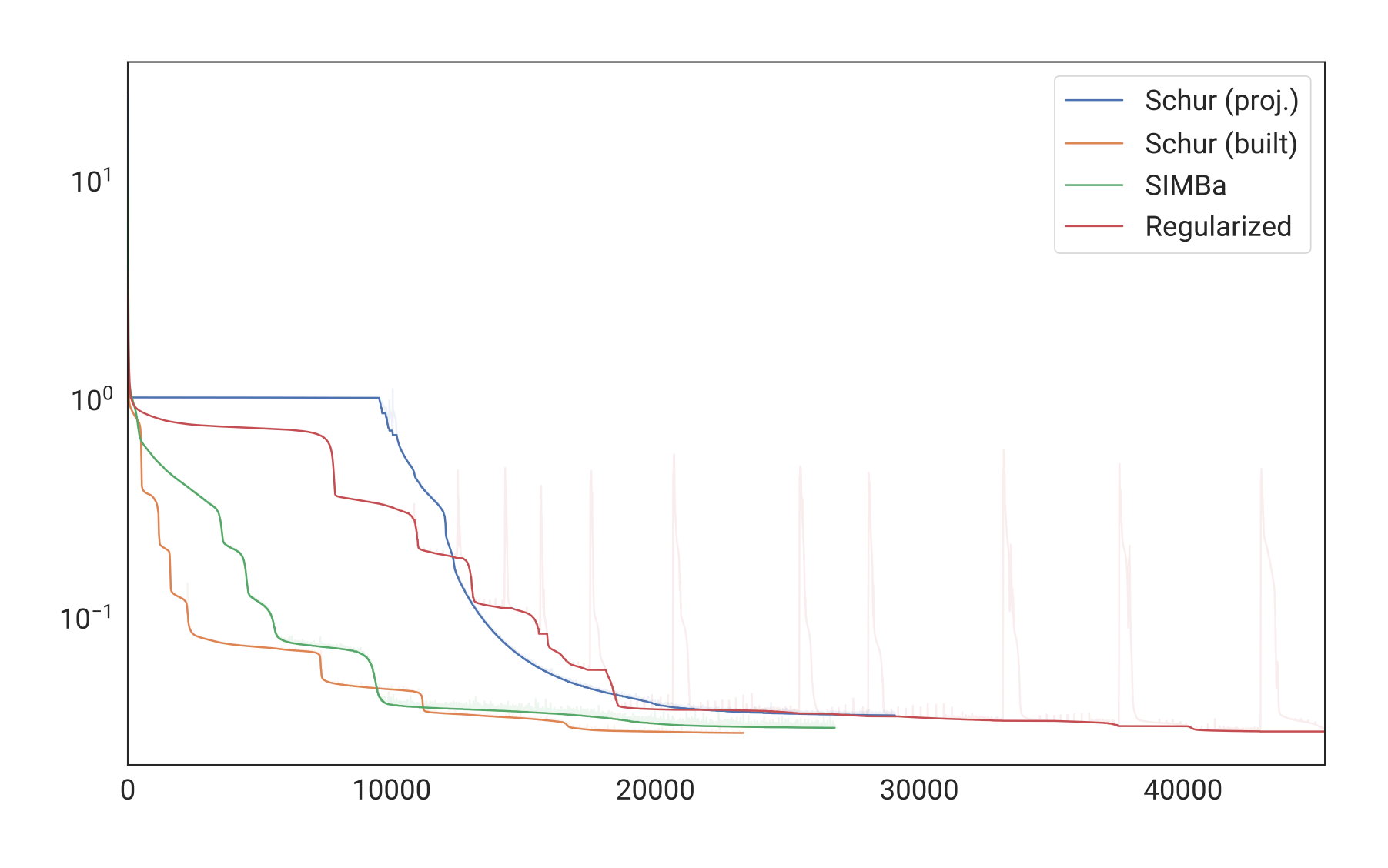}
        \caption{Fine-steering mirror}
    \end{subfigure}
    \caption{Validation \gls{nmse} history for the best-performing models. The shaded lines correspond to the recorded values per epoch, whereas the solid curves denote the cumulative minima.}
    \label{fig:real_history}
\end{figure}

\begin{figure}[tb!]
    \centering
    \includegraphics[width=0.75\textwidth]{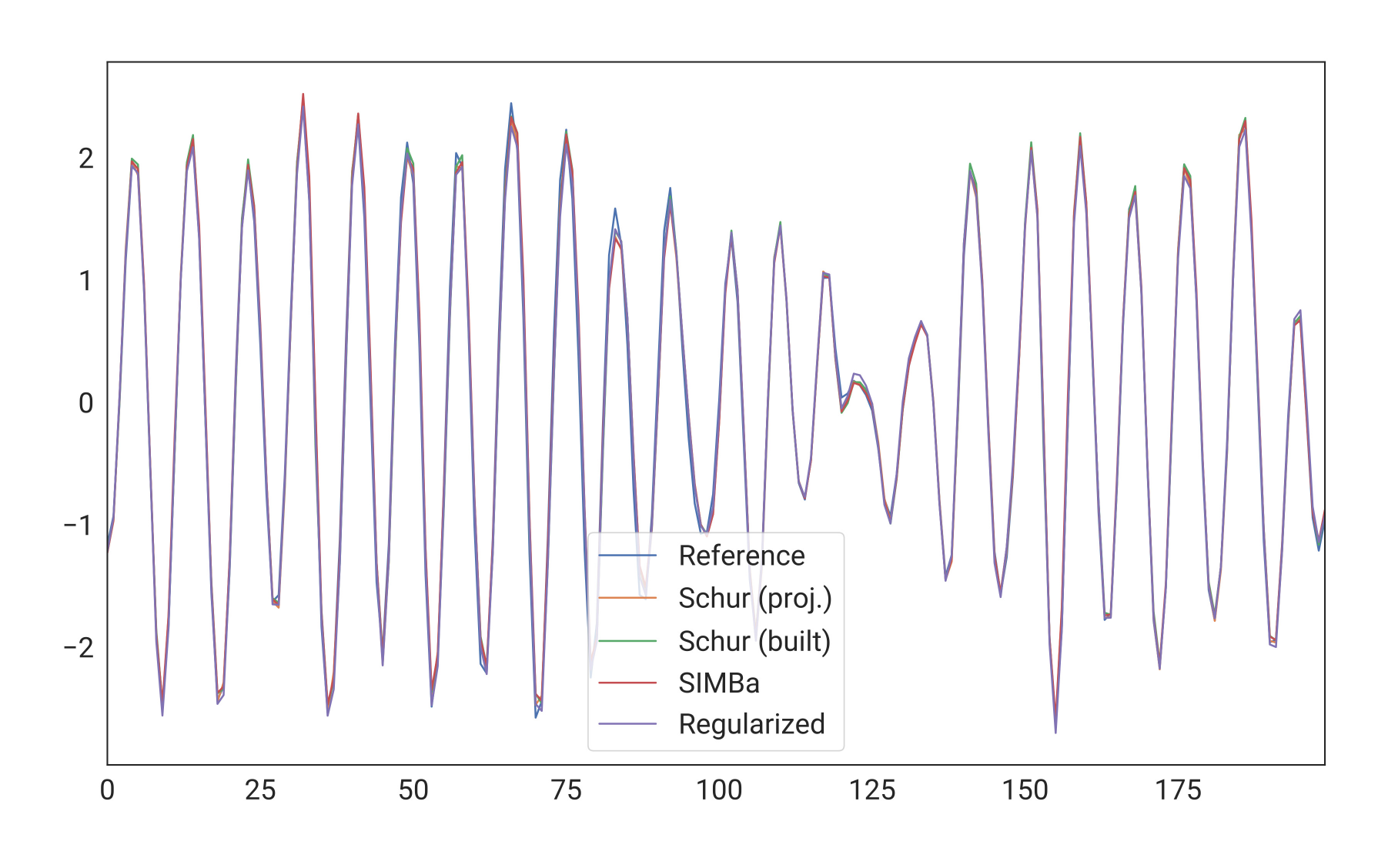}
    \caption{Silverbox test-partition output}
    \label{fig:silverbox_output}
\end{figure}

\begin{figure}[tb!]
    \centering
    \begin{subfigure}[m]{0.49\textwidth}
        \includegraphics[width=\textwidth]{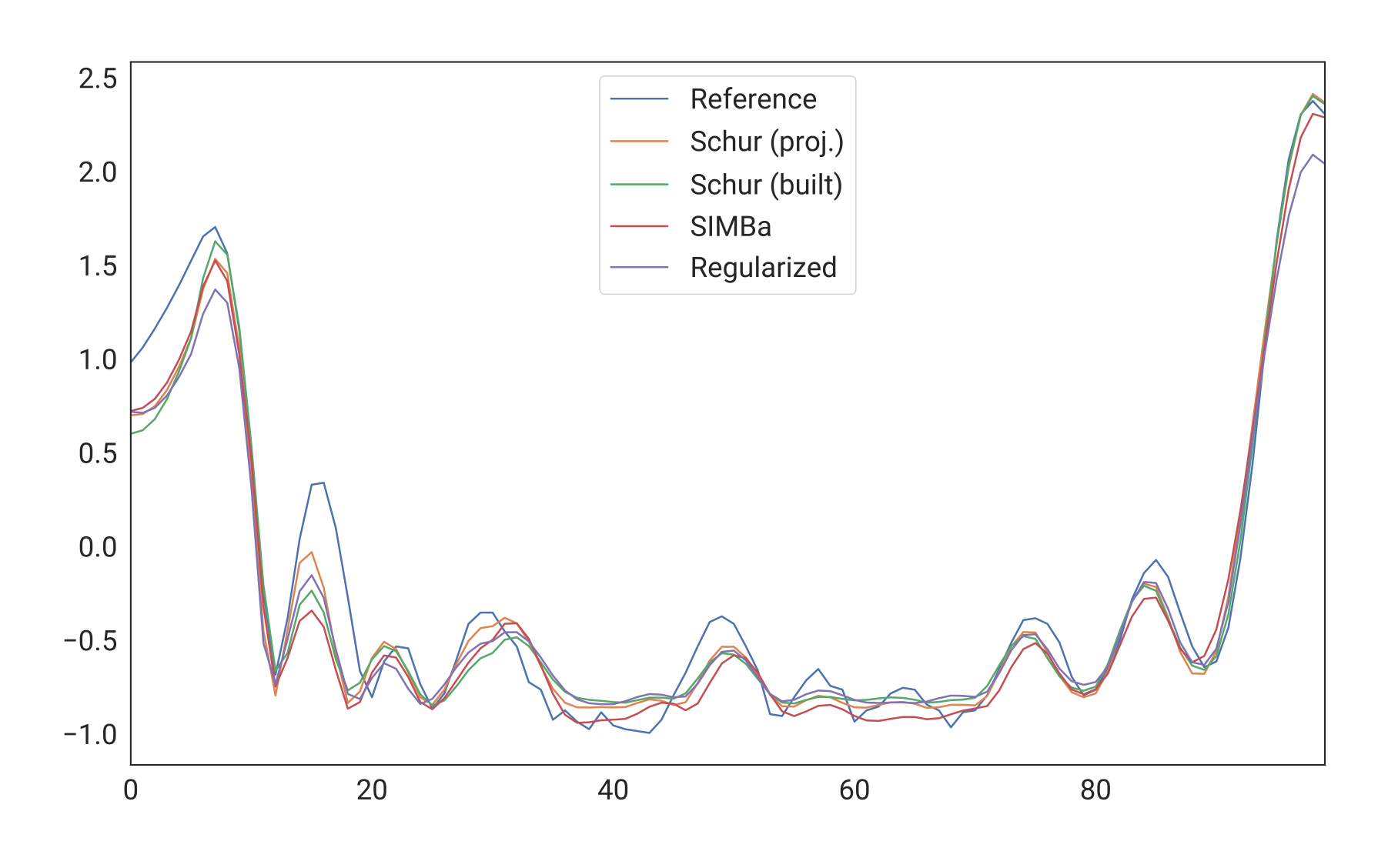}
        \caption{$y_1$}
    \end{subfigure}
    \hfill
    \begin{subfigure}[m]{0.49\textwidth}
        \includegraphics[width=\textwidth]{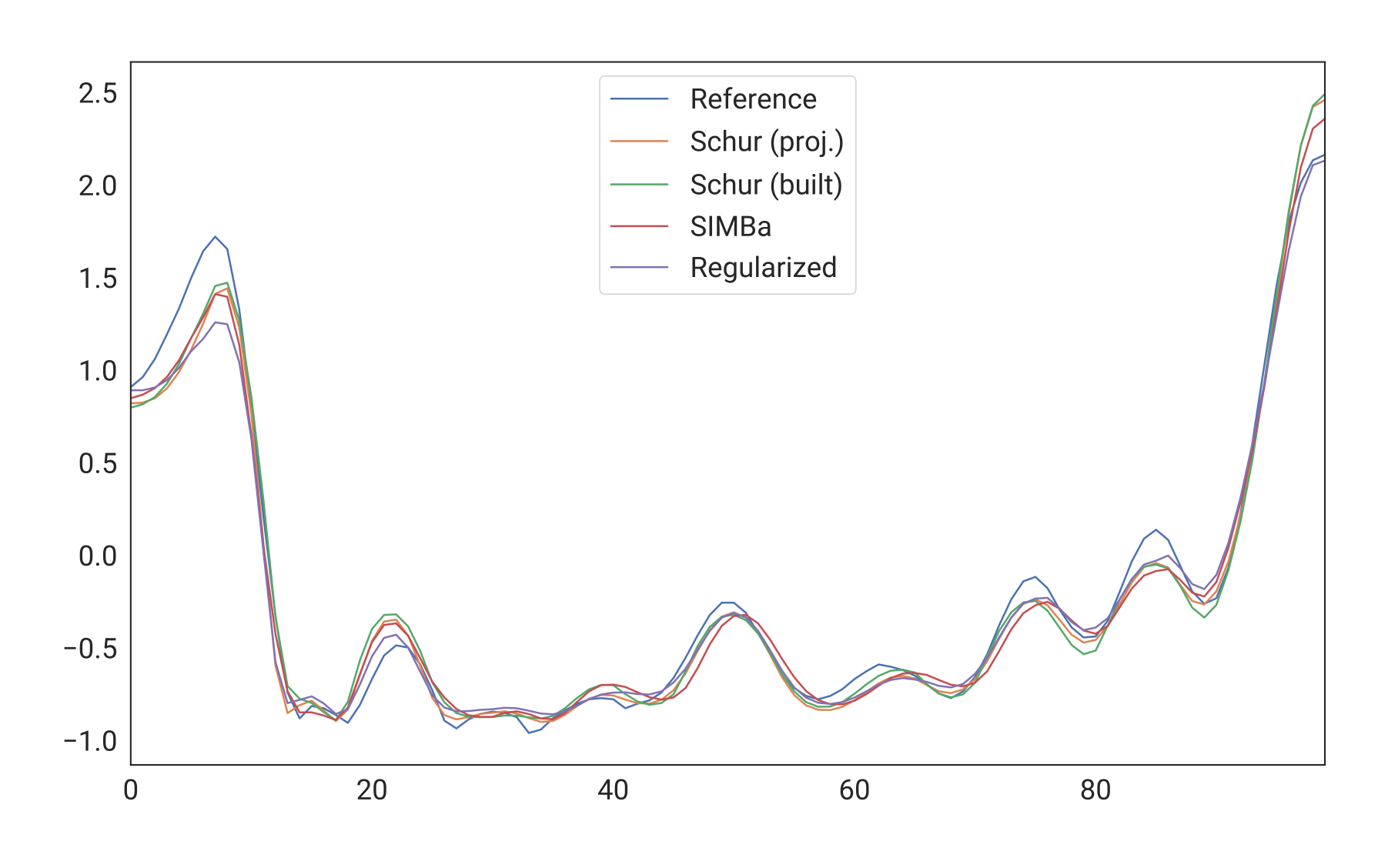}
        \caption{$y_2$}
    \end{subfigure}
    \caption{CED test-partition output}
    \label{fig:ced_output}
\end{figure}

\begin{figure}[tb!]
    \centering
    \includegraphics[width=0.75\textwidth]{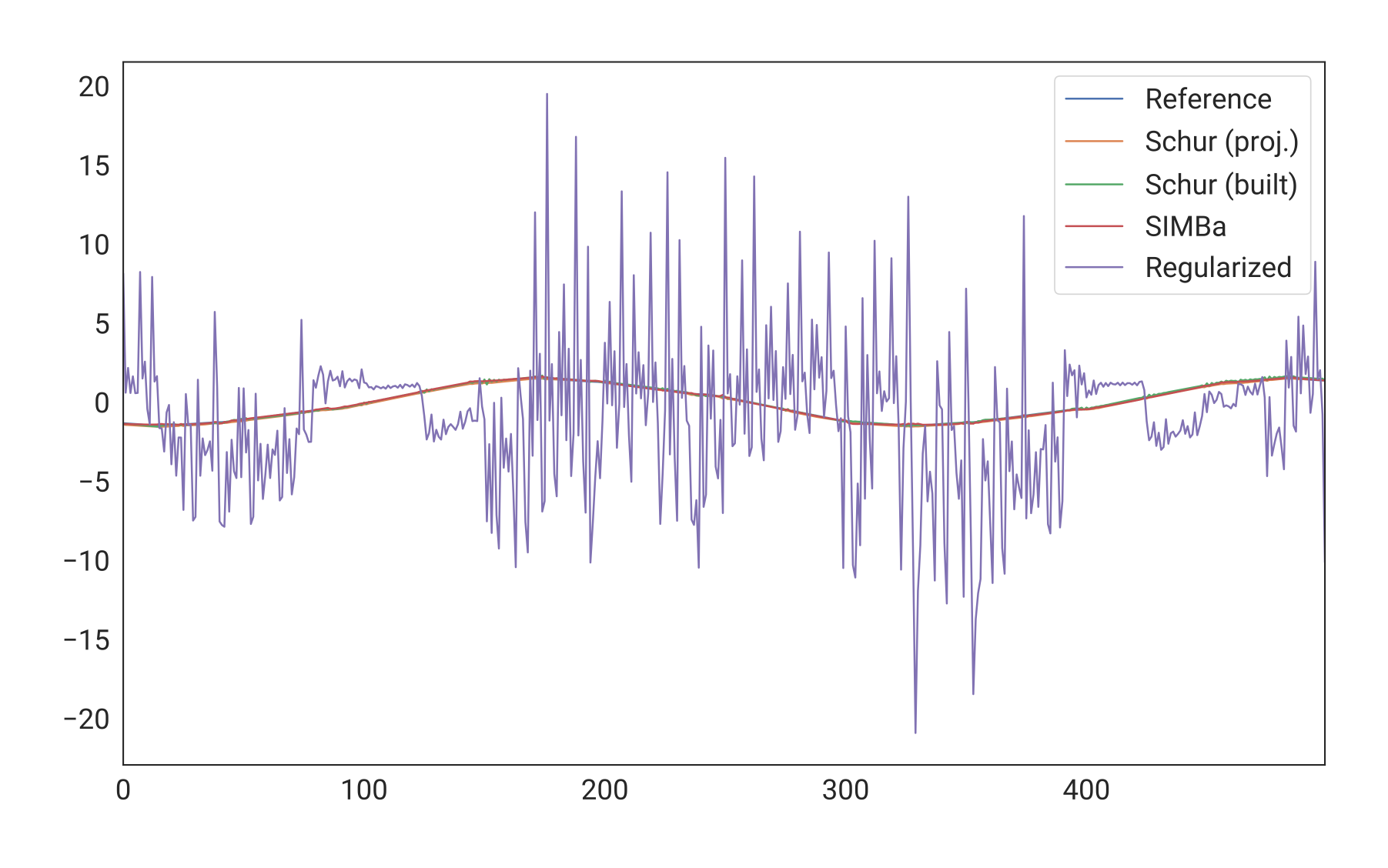}
    \caption{EMPS test-partition output}
    \label{fig:emps_output}
\end{figure}

\begin{figure}[tb!]
    \centering
    \begin{subfigure}[m]{0.49\textwidth}
        \includegraphics[width=\textwidth]{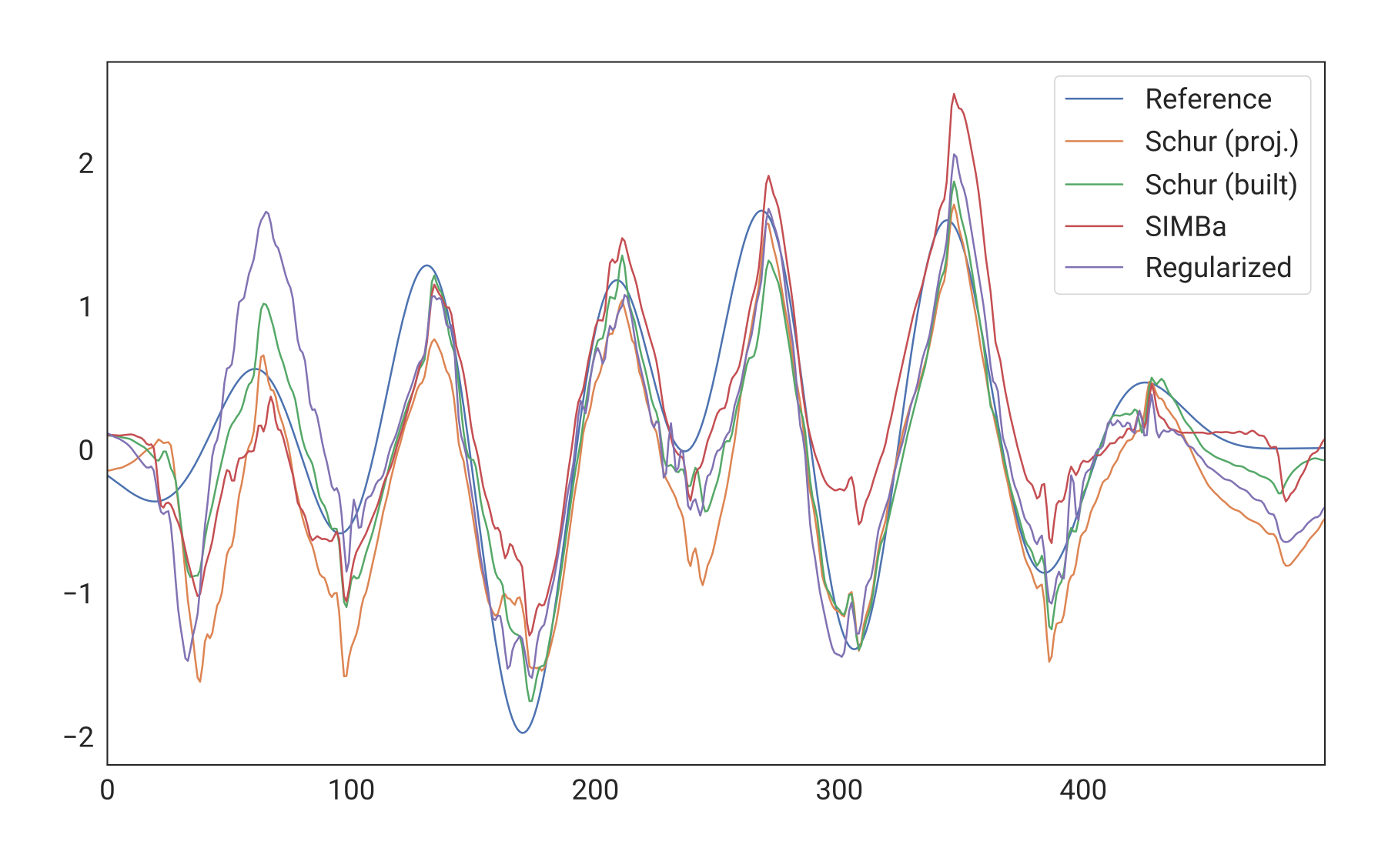}
        \caption{$y_1$}
    \end{subfigure}
    \hfill
    \begin{subfigure}[m]{0.49\textwidth}
        \includegraphics[width=\textwidth]{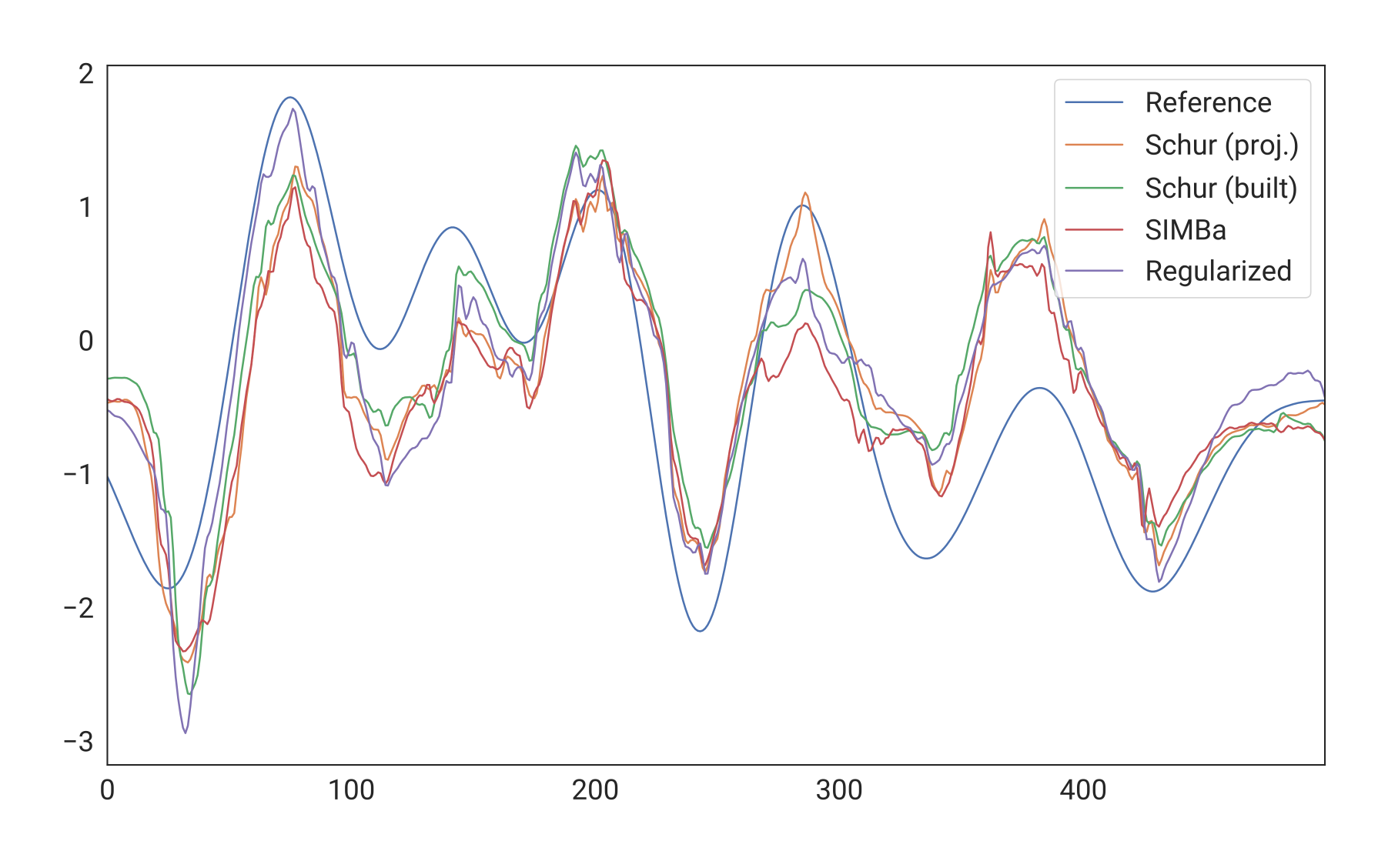}
        \caption{$y_2$}
    \end{subfigure}
    \\
    \begin{subfigure}[m]{0.49\textwidth}
        \includegraphics[width=\textwidth]{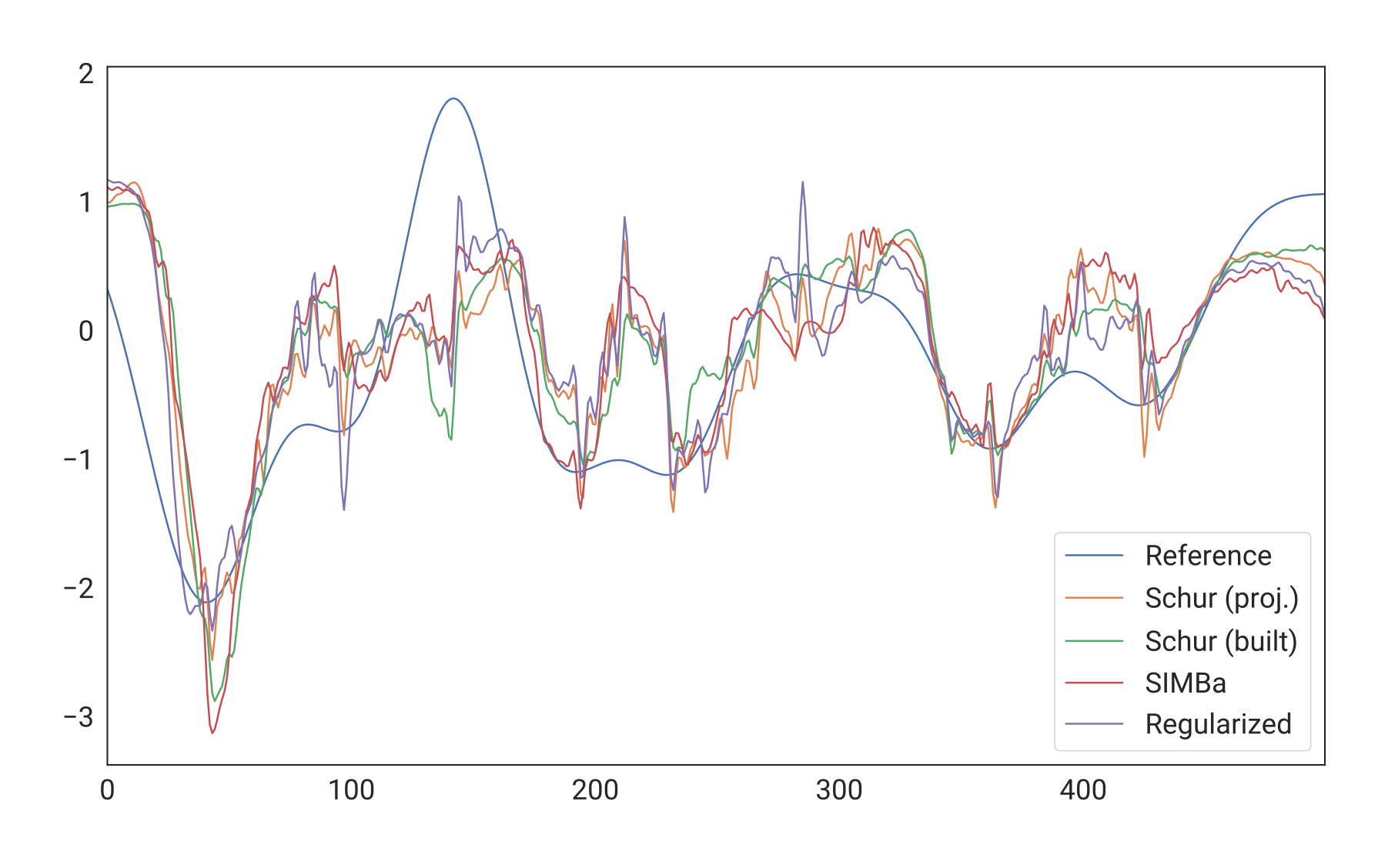}
        \caption{$y_3$}
    \end{subfigure}
    \hfill
    \begin{subfigure}[m]{0.49\textwidth}
        \includegraphics[width=\textwidth]{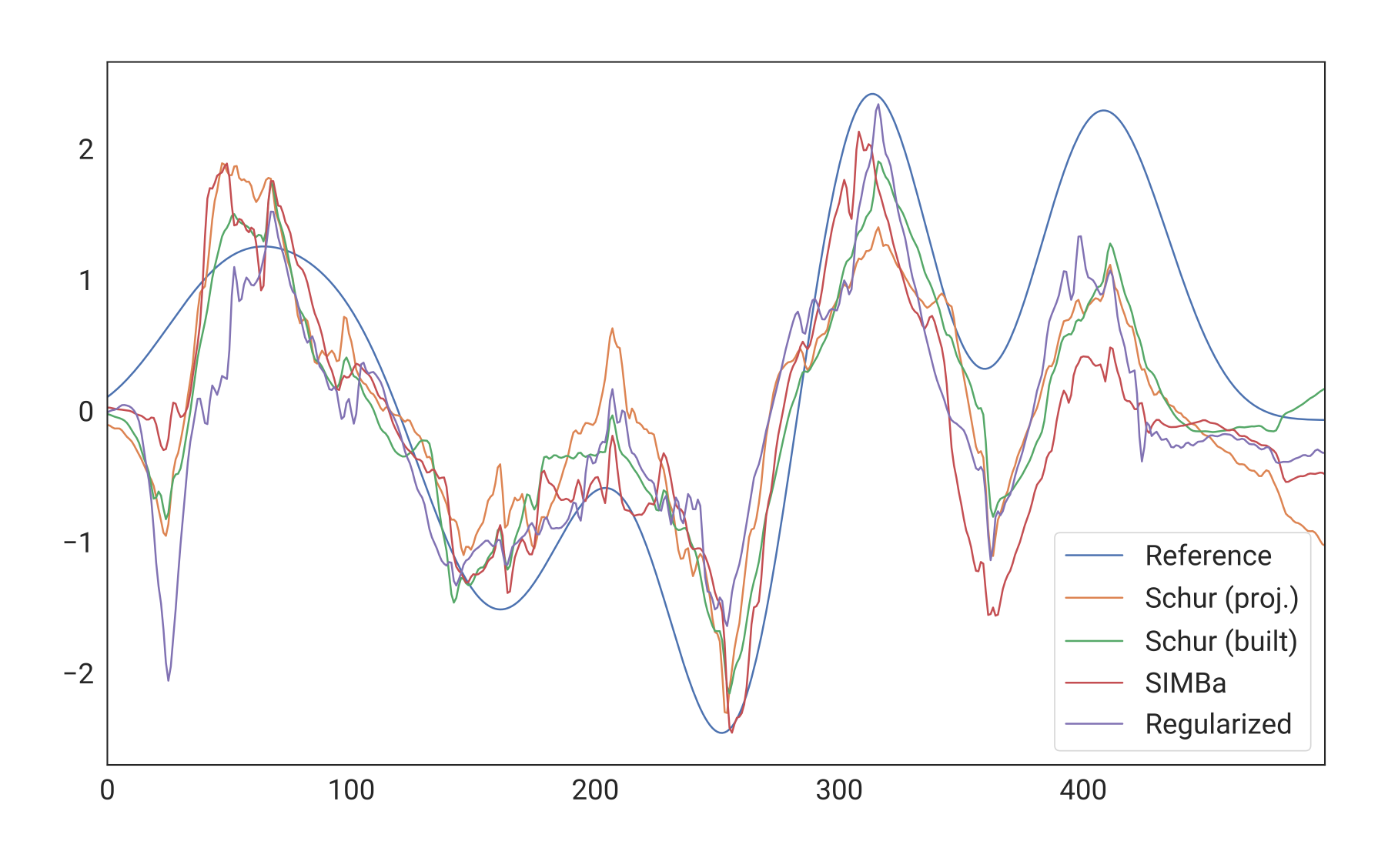}
        \caption{$y_4$}
    \end{subfigure}
    \\
    \begin{subfigure}[m]{0.49\textwidth}
        \includegraphics[width=\textwidth]{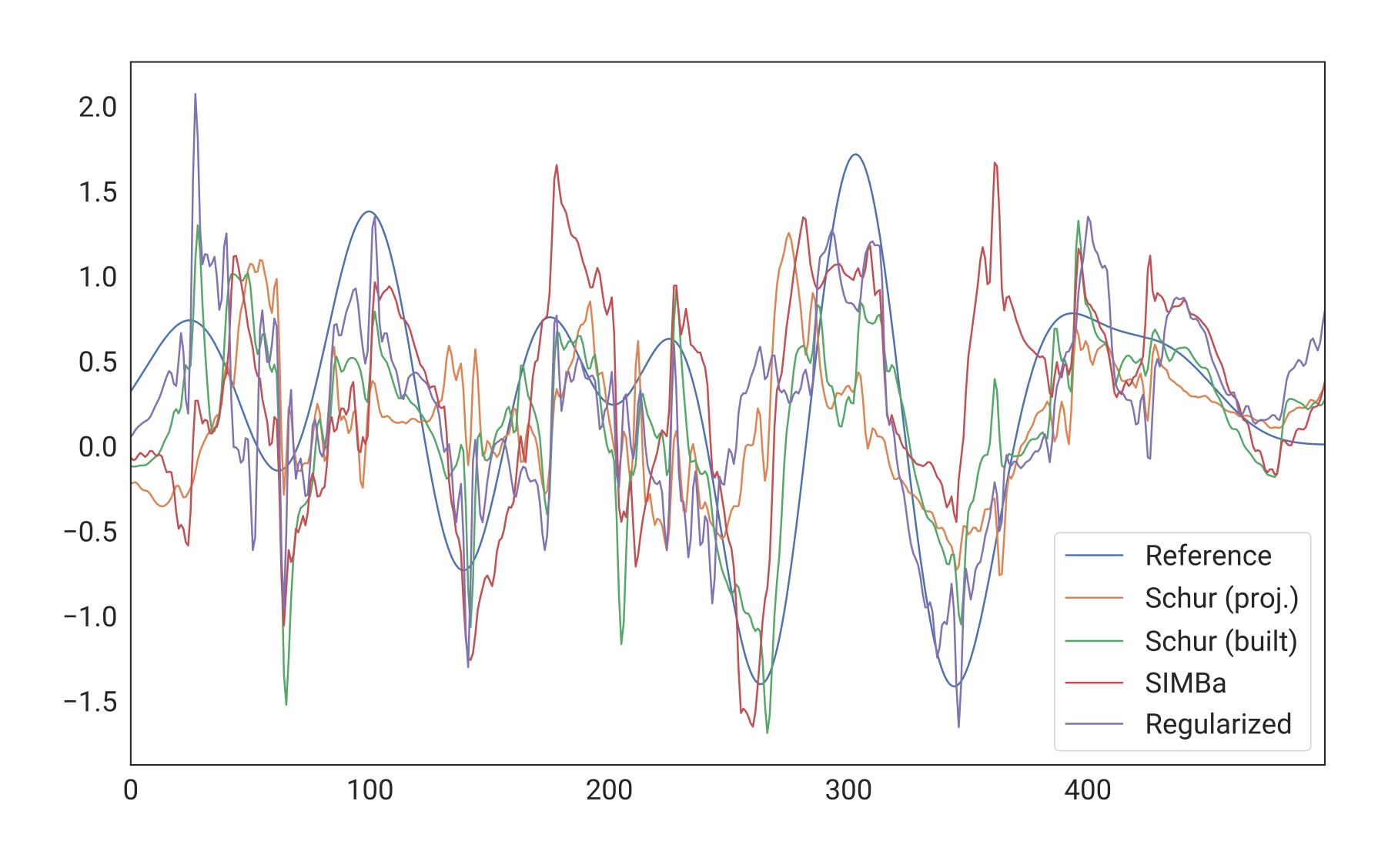}
        \caption{$y_5$}
    \end{subfigure}
    \hfill
    \begin{subfigure}[m]{0.49\textwidth}
        \includegraphics[width=\textwidth]{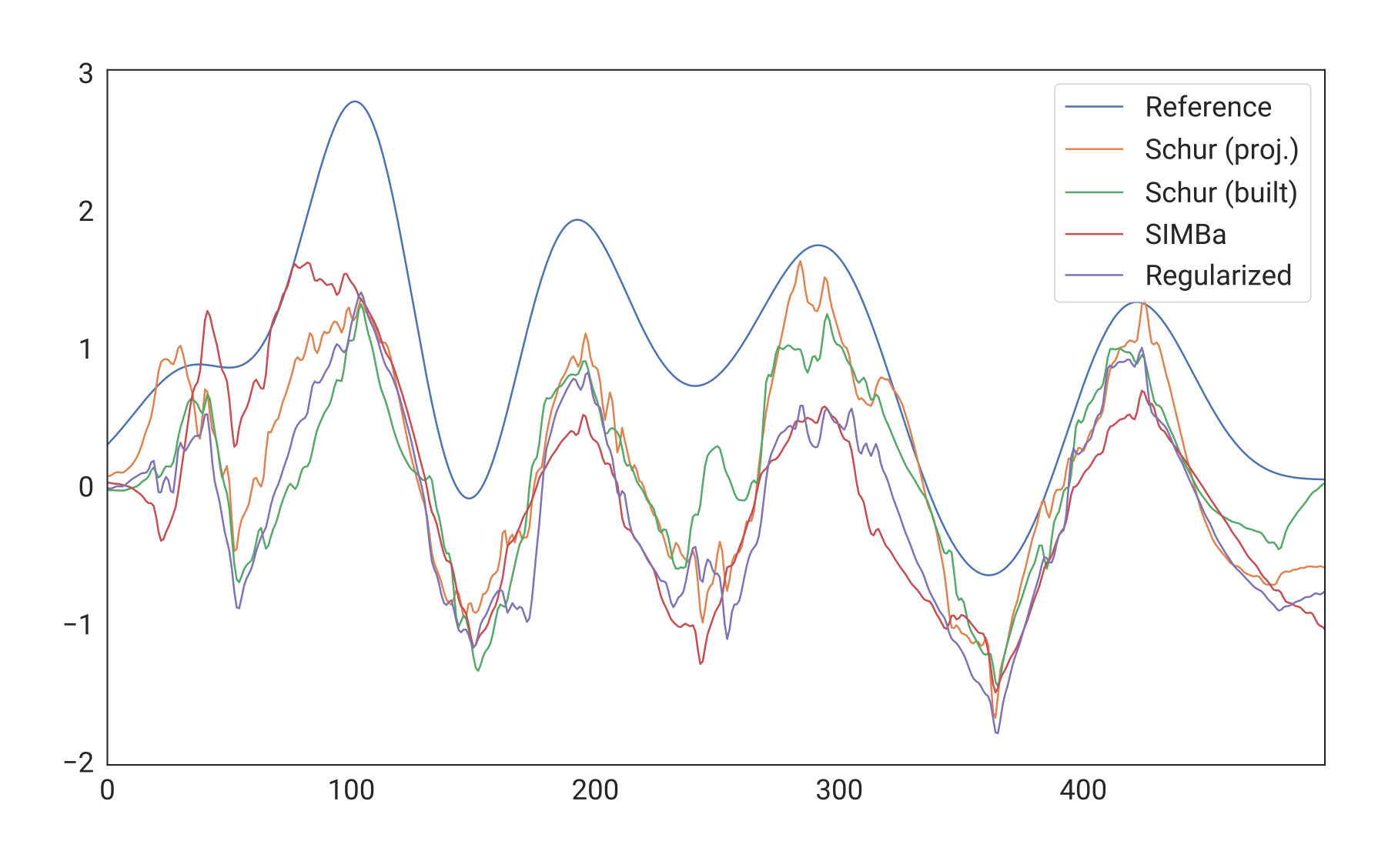}
        \caption{$y_6$}
    \end{subfigure}
    \caption{Industrial robot test-partition output}
    \label{fig:industrial_robot_output}
\end{figure}

\begin{figure}[tb!]
    \centering
    \begin{subfigure}[m]{0.75\textwidth}
        \includegraphics[width=\textwidth]{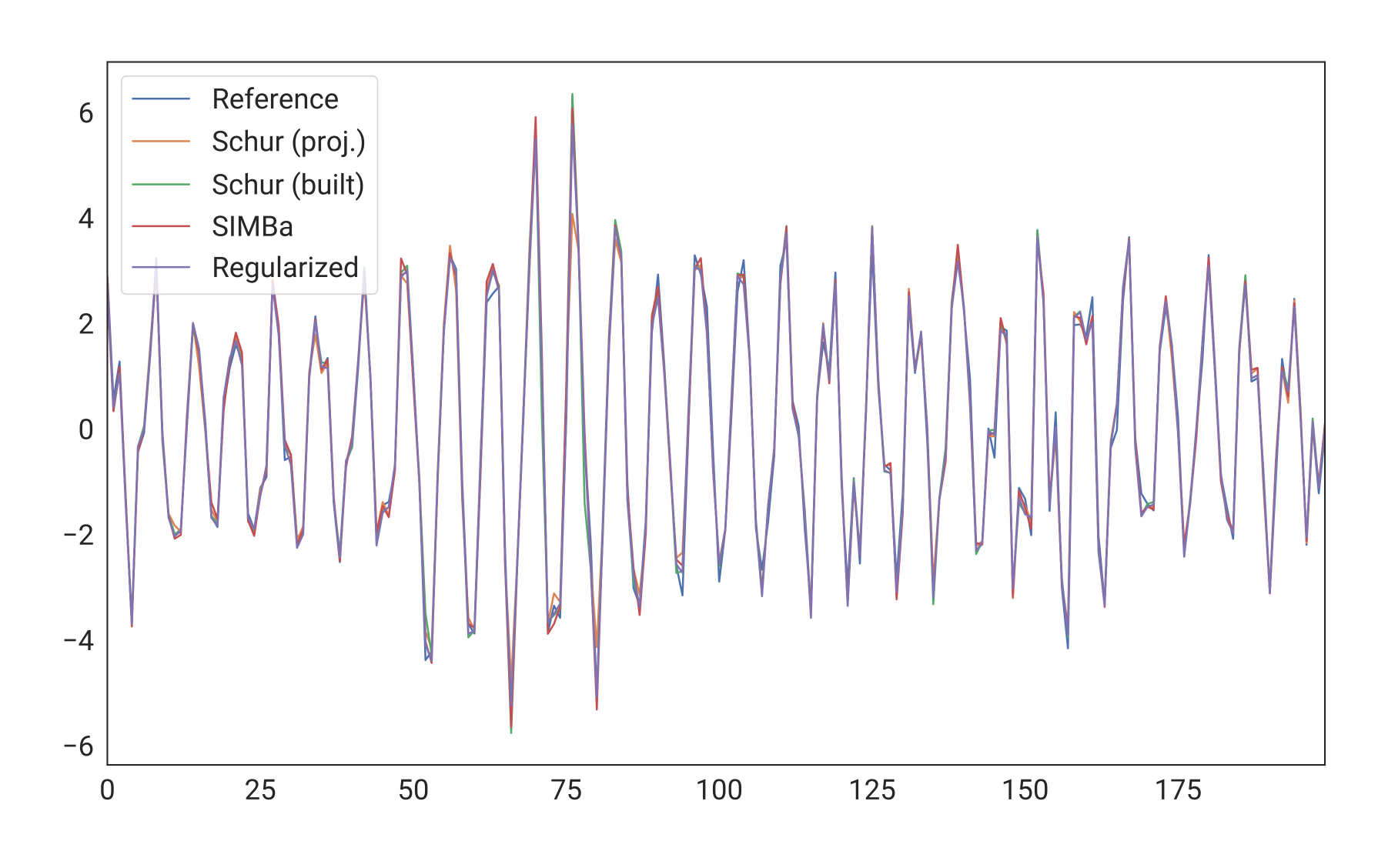}
        \caption{$y_1$}
    \end{subfigure}
    \\
    \begin{subfigure}[m]{0.75\textwidth}
        \includegraphics[width=\textwidth]{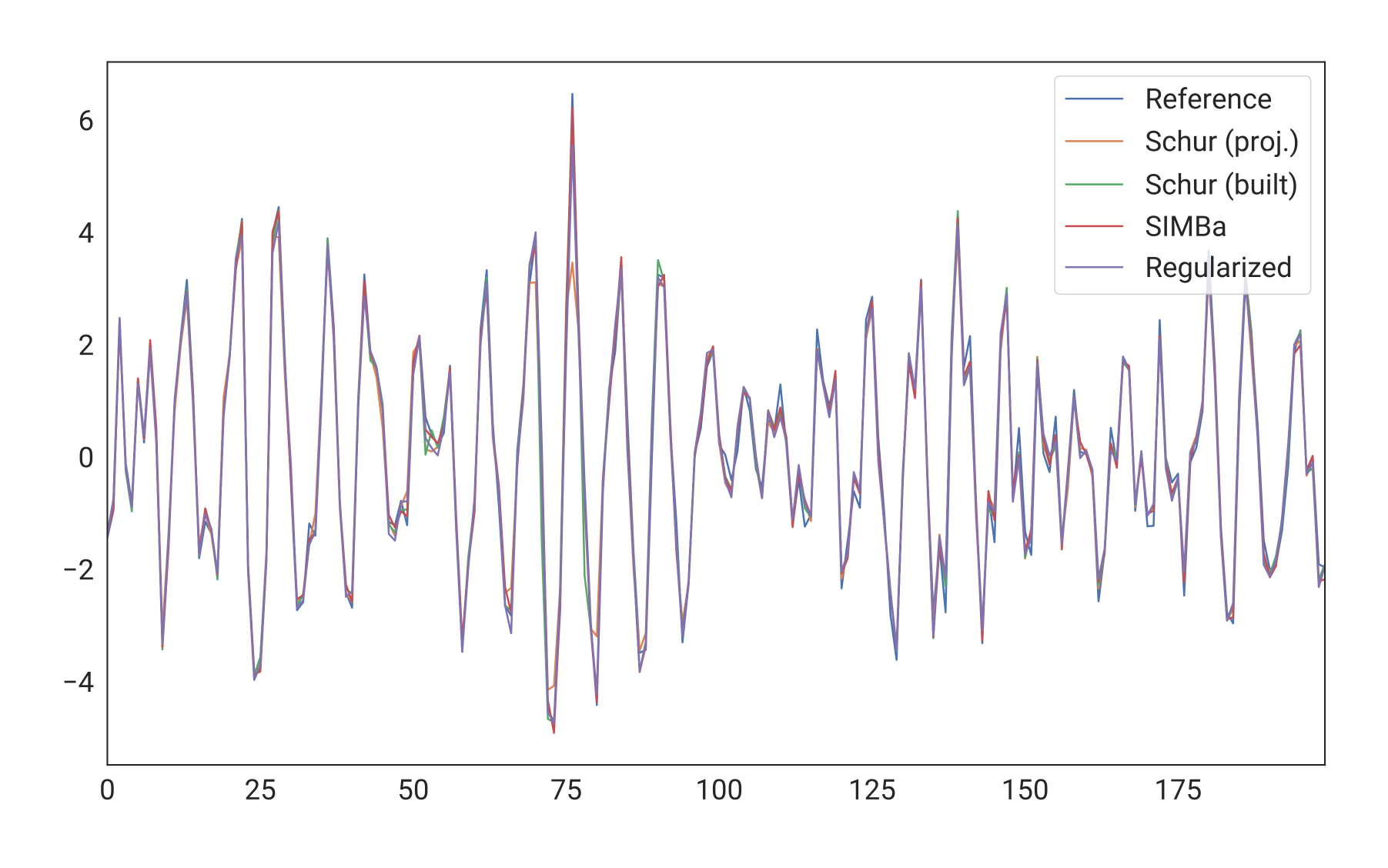}
        \caption{$y_2$}
    \end{subfigure}
    \\
    \begin{subfigure}[m]{0.75\textwidth}
        \includegraphics[width=\textwidth]{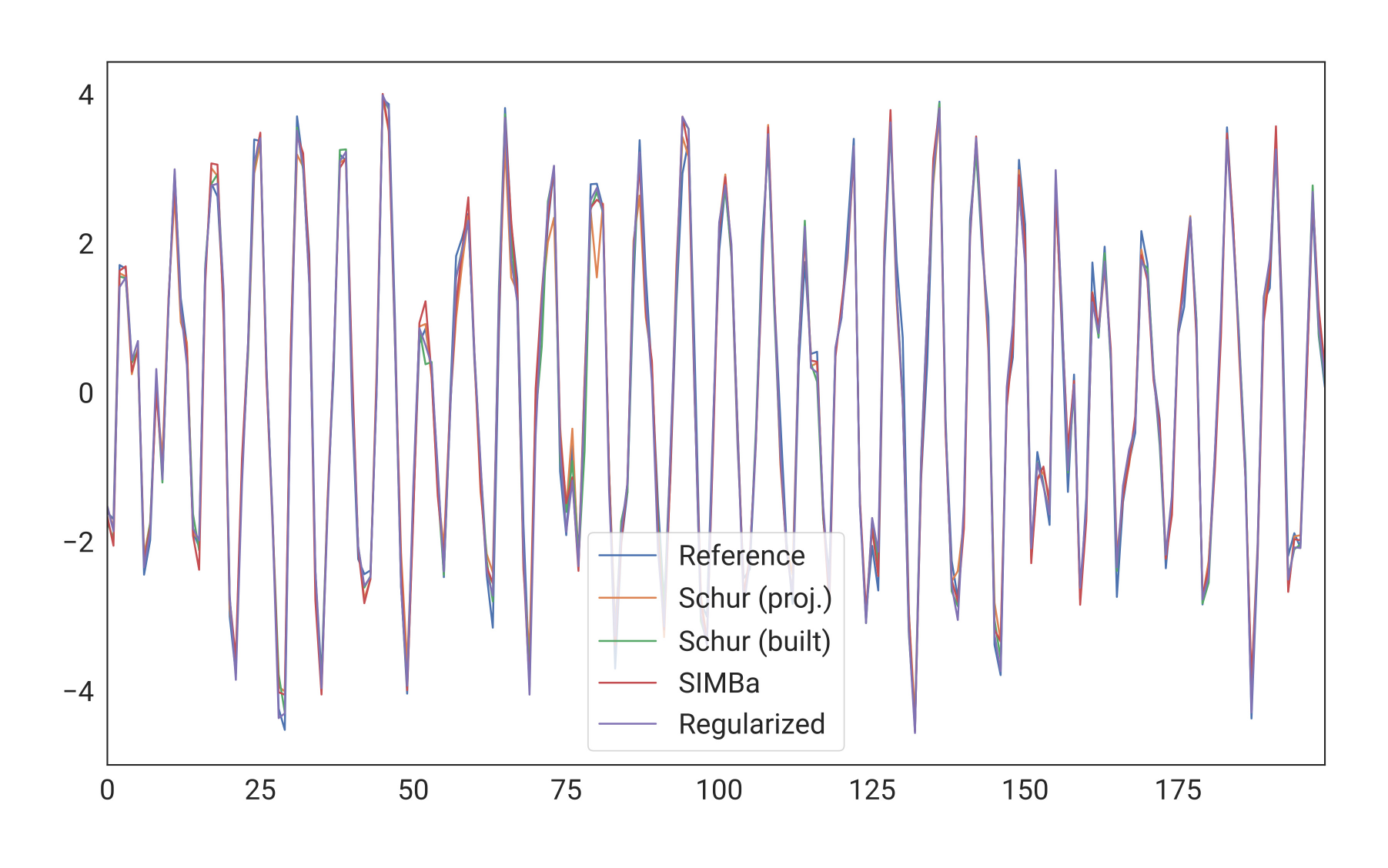}
        \caption{$y_3$}
    \end{subfigure}
    \caption{Fine-Steering Mirror test-partition output (first sequence)}
    \label{fig:fsr_output}
\end{figure}

The channel-wise error for the multiple-output datasets is reported in \Cref{tb:real_channel_error}. The validation \gls{nmse} history is reported only for the best-performing model; i.e., the one yielding the error metrics in \Cref{tb:real_benchmark,tb:real_channel_error}. Finally, all test-partition outputs are presented in \Cref{fig:silverbox_output,fig:ced_output,fig:emps_output,fig:industrial_robot_output,fig:fsr_output} (once again, only for the best-performing model).

The following analysis elements are highlighted:
\begin{itemize}
    \item Consistent with the numbers reported in \Cref{tb:real_benchmark}, the regularized method consistently exhibits slow convergence rates, with this trend only intensifying as $n_x$ grows larger. Conversely, the Schur (built) method always has the steepest (or at least second-steepest) convergence rate, which is consistent with the results obtained in \Cref{subsec:synthetic_exp}.
    \item While all other models manage to almost perfectly replicate the output signal of the EMPS dataset, the regularized method is incapable of capturing its dynamics, as illustrated by \Cref{fig:emps_output}. This dataset turned out to be more challenging than initially expected, as the input high-frequency components visible in \Cref{subfig:emps_visualization} led the models to oversensitive local minima, obstructing accurate identification.
\end{itemize}


\end{document}